\documentclass{article}
\usepackage{microtype}
\usepackage{hyperref}

\usepackage[preprint]{icml2026}

\icmltitlerunning{Data Reuse in Non-Stationary Learning}
\makeatletter
\renewcommand{\printAffiliationsAndNotice}[1]{%
  \global\icml@noticeprintedtrue
  \begin{NoHyper}%
  {\let\thefootnote\relax\footnotetext{\hspace*{-\footnotesep}%
    #1Correspondence to: \icmlcorrespondingauthor@text.\par
    \Notice@String}}%
  \end{NoHyper}%
}
\makeatother

\usepackage{amsmath,amsfonts,bm}

\def\eqref#1{equation~\ref{#1}}
\def\Eqref#1{Equation~\ref{#1}}
\def\plaineqref#1{\ref{#1}}
\def\1{\bm{1}}

\DeclareMathAlphabet{\mathsfit}{\encodingdefault}{\sfdefault}{m}{sl}
\SetMathAlphabet{\mathsfit}{bold}{\encodingdefault}{\sfdefault}{bx}{n}

\newcommand{\E}{\mathbb{E}}

\newcommand{\R}{\mathbb{R}}

\newcommand{\KL}{D_{\mathrm{KL}}}

\usepackage{url}

\usepackage{amsmath,amssymb,amsthm,mathtools}
\usepackage{enumitem}
\usepackage{xcolor}
\usepackage{booktabs}
\usepackage{graphicx}
\usepackage{tikz}
\usepackage{float}
\usepackage{algorithm}
\usepackage{algorithmic}
\usepackage[english]{babel}
\usepackage[autostyle, english = american]{csquotes}
\MakeOuterQuote{"}
\usetikzlibrary{arrows.meta}

\hypersetup{
  colorlinks=true, linkcolor=blue, citecolor=blue, urlcolor=blue,
  pdftitle={Data Reuse in Non-Stationary Learning},
  pdfauthor={Tomer Gafni, Garud Iyengar, Assaf Zeevi},
  pdfsubject={Preprint}
}

\theoremstyle{plain}
\newtheorem{theorem}{Theorem}[section]
\newtheorem{proposition}[theorem]{Proposition}
\newtheorem{lemma}[theorem]{Lemma}
\newtheorem{corollary}[theorem]{Corollary}
\theoremstyle{definition}

\newtheorem{remark}[theorem]{Remark}
\newcommand{\Pbb}{\mathbb P}
\newcommand{\cS}{\mathcal S}
\newcommand{\cB}{\mathcal B}

\newcommand{\cC}{\mathcal C}

\newcommand{\id}{\mathrm{id}}
\newcommand{\Dstr}{\mathsf D}
\newcommand{\Estr}{\mathsf E}
\newcommand{\anc}{\mathrm{anc}}
\newcommand{\un}{\mathrm{un}}
\newcommand{\unsafe}{\mathrm{unsafe}}
\newcommand{\bad}{\mathrm{bad}}
\newcommand{\FA}{\mathrm{FA}}

\usepackage{soul}

\begin{document}

\onecolumn
\icmltitle{Data Reuse in Non-Stationary Learning}
\begin{center}
\begin{tabular*}{0.85\textwidth}{@{\extracolsep{\fill}}ccc}
\textbf{Tomer Gafni} & \textbf{Garud Iyengar} & \textbf{Assaf Zeevi}\\
Columbia University & Columbia University & Columbia University\\
\end{tabular*}
\end{center}
\icmlcorrespondingauthor{Tomer Gafni}{tbg2124@columbia.edu}
\icmlkeywords{Online learning, non-stationarity, data reuse}
\vskip 0.3in
\printAffiliationsAndNotice{}

\begin{abstract}
We consider online learning in non-stationary environments, where the goal is to track an unknown parameter that switches abruptly between a finite set of recurring values. Recurrence opens the possibility of judiciously reusing past observations to improve algorithm performance. However, the changing nature of the underlying signal and lack of information on these dynamics may limit the ability to "safely" reuse data. In this paper we quantify some of the fundamental tradeoffs in this class of problems, and show that they bear a certain resemblance to the classical bias-variance dilemma. Specifically, we propose a class of anytime algorithms, dubbed Exposure-Capped Reuse (ECR), that combine online change detection, compatibility testing, and "contamination" control. We characterize the regime in which ECR’s regret scales with the number of distinct values rather than the number of changes, and derive a novel information-theoretic lower bound that establishes the near-minimax optimality of ECR. This provides rigorous quantification of the statistical "value" of data reuse.
\end{abstract}

\section{Introduction}
\label{sec:introduction}

\subsection{Background and Overview}
\label{ssec:background}
Many learning systems operate in non-stationary environments in which the
data-generating distribution may change and later
recur~\citep{katakis2010tracking, alippi2013just}. For example,
computational demand in cloud systems may change abruptly as workload
groups connect or disconnect, with previously observed demand patterns
recurring when similar configurations return. Often, the cause of a change
is hidden from the learner and not known a
priori~\citep{tsymbal2004problem}. {This setting is closely related to
learning under distribution shift, including statistical transfer
learning~\citep{wang2025transfer} and concept drift
adaptation~\citep{gama2014survey}.} When a previously observed
concept reappears, past information may become useful again, circumventing  the
need to learn it from scratch~\citep{suarez2023survey}.  Exploiting this
recurrence, however, requires the learner to continuously monitor and detect regime changes and concurrently 
determine which historical data can be "safely" reused. To wit, the data needs to be classified in an online manner and attributed with high enough confidence to the correct underlying statistical source to ensure that it can be reused effectively. As will be articulated in what follows, addressing these salient ingredients in a synergistic fashion presents interesting new challenges.

\textbf{The online data reuse problem.} To flesh out the fundamental statistical trade-offs embedded in recurrent structured non-stationary environments, 
we consider a canonical online mean-tracking problem. The problem captures key salient features yet maintains a clean enough structure that lends itself to  elucidating the main learning theoretic insights. The setup we consider is one where the  learner sequentially observes noisy sub-Gaussian samples $\{X_t\}_{t=1}^T$, whose mean $\mu_t$ evolves in a piecewise-constant manner over a
horizon $T$,  with $S$ changes among $M$ distinct values ({Fig.~\ref{fig:recurrent-environment}}). We refer to these values as
\textit{classes}. Before observing $X_t$, the learner predicts $\widehat{\mu}_t$, and performance is measured by the expected cumulative $L_1$ error 
\(
\mathbb{E}\!\left[\sum_{t=2}^{T}
\left|\widehat{\mu}_t-\mu_t\right|\right]
\),
{corresponding to dynamic regret with respect to the time-varying mean sequence.} We study it under a \textit{minimax} criterion, namely, when nature is allowed to deploy the worst case patterns of the changing mean. 

The primary regime of interest is when  $M\ll S$, i.e., the environment changes frequently but repeatedly revisits a small set of classes. The change points, 
the means for each class, the active class at the current decision epoch, as well as $T$, $S$, and $M$, are all unknown to the learner.  
Our formulation differs from existing recurrent-learning settings in that it is fully online and anytime, requires no prior knowledge of the parameters except for the noise proxy \(\sigma\), and imposes no assumptions on a minimum spacing between changes.  

\textbf{The value of recurrence.}
{Consider a policy that knows the change points and predicts \(\widehat{\mu}_t\) by estimating the mean after each change from scratch, without reusing observations from previous segments. Its regret, measured by the criterion above, scales as \(O(\sigma\sqrt{ST})\).} (To see this, imagine the change points are equi-spaced, so there are $S$ segments each of length $T/S$.) 
{We refer to this policy as the \textit{no-reuse benchmark}.} 
{If the recurrent class identities are also known, the policy can form each prediction \(\widehat{\mu}_t\) by pooling past observations from different occurrences of the current class. This increases the effective sample size and reduces estimation variance, resulting in \(O(\sigma\sqrt{MT})\) regret.} (To see this, simply imagine allocating $T/M$ data points to each class.) {We refer to this policy as the \textit{perfect-reuse benchmark}.}   
Thus, recurrence can be leveraged to replace the dependence on the number of changes $S$ with the number of classes $M \ll S$. 
In long-horizon settings where the number of changes grows with $T$ while the number of classes remains finite, as is natural in many applications, the "value" of data reuse, {reflected in the ratio of the two benchmark regret rates}, grows without bound  
(see Fig.~\ref{fig:oracle-value-recurrence}). 

\begin{figure}[t]
  \centering
  \begin{minipage}[t]{0.48\linewidth}
    \vspace{0pt}
\centering
 \begin{minipage}[t][0.67\linewidth][t]{\linewidth}
    \vspace{0pt}
    \includegraphics[width=\linewidth]{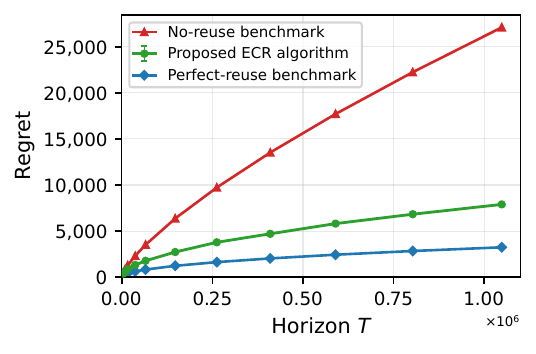}
  \end{minipage}\par
  \caption{
\textbf{The value of recurrence.} The setup is one where  $S\asymp{\sqrt{T}}$ and $M=4$. 
{The widening regret gap between the no-reuse and perfect-reuse benchmarks illustrates the growing benefit ("value") of ideal data reuse. Our proposed ECR algorithm is able to exploit this benefit without knowing the change points or class identities (Th.~\ref{thm:ecr-upper-bound}).}} 
% consistent with matching the perfect-reuse benchmark rate up to polylogarithmic factors .
  \label{fig:oracle-value-recurrence}
  \end{minipage}\hfill
  \begin{minipage}[t]{0.48\linewidth}
    \vspace{0pt}
\centering 
  \begin{minipage}[t][0.67\linewidth][t]{\linewidth}
    \vspace{0pt}
    \includegraphics[width=\linewidth]{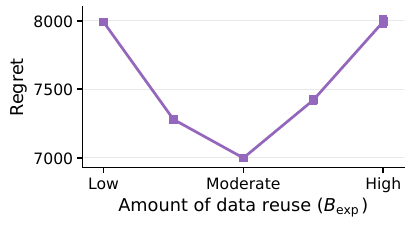}
  \end{minipage}\par
  \caption{
  {\textbf{More reuse is not necessarily better.} 
  A plot exhibiting regret as a function of the amount of data reuse (controlled by the exposure budget $B_{\exp}$). 
  Increasing reuse means more data can be used to reduce estimation variance. But this can result in "contamination" due to misclassification of past data, which increases bias. 
  Regret is, therefore, minimized at an intermediate level of data reuse.}
  }
  
  \label{fig:reuse-tradeoff}
  \end{minipage}
\end{figure}

\textbf{The price of non-stationarity.}
The central challenge in designing an anytime online algorithm with regret that scales with the number of classes, $O(\sigma\sqrt{MT})$, is that the learner must detect changes online and determine whether the current statistical regime is identified (with high confidence) with a previously observed class (if any).  
Each change necessarily incurs a detection delay, and sufficiently
short segments may be missed
altogether~\citep{gafni2026the}. Consequently, observations collected
before a change is detected, or when it is missed, \textit{contaminate} the stored
historical data. In other words, reusing past data that is incorrectly "labeled" can
introduce bias into the estimator. 
This creates a \textit{fundamental tradeoff}: reusing more historical data
reduces estimation variance, but also increases the 
bias created by 
contaminated observations. 

\subsection{Main contributions}
\label{ssec:main-contributions}

% Our main contributions are summarized below.

{\textbf{Exploiting recurrence despite non-stationarity.}
We propose \emph{Exposure-Capped Reuse} (ECR)
(Alg.~\ref{alg:ECR}), an anytime algorithm that realizes the statistical benefit of recurrence through judicious data reuse. We prove a non-asymptotic regret upper bound accounting for both estimation variance and contamination bias. In particular, when \(S=\widetilde O(\sqrt T)\), ECR attains \(\widetilde O(\sigma\sqrt{MT})\) regret, matching the perfect-reuse benchmark up to polylogarithmic factors (Th.~\ref{thm:ecr-upper-bound}).} 

ECR operates without knowing the horizon, the change
points, or the recurrent parameters. It combines: (i) \textit{Online
  segmentation:} a change detector that partitions the observation history
into reusable blocks despite potentially unavoidable missed changes
(Prop.~\ref{prop:anchored-segmentation}); (ii) \textit{Compatibility
  testing:} a statistical test that guarantees reuse of sufficiently "clean"
blocks from the same recurrent class while rejecting blocks associated
with different classes (Prop.~\ref{prop:compatibility-summary}); and (iii)
\textit{Exposure control:} a mechanism that limits the cumulative influence of each historical block (Sec.~\ref{ssec:exposure-capped-reuse}), as elaborated below.

% an exposure parameter that tracks each block's cumulative influence and a budget $B_{\exp}$ that limits its reuse, thereby balancing the bias-variance tradeoff (Sec.~\ref{ssec:exposure-capped-reuse}).

\textbf{Controlling bias from repeated reuse.}
We show that the main difficulty in exploiting recurrence
 in non-stationary settings  is \textit{bias accumulation under repeated
  reuse}.
% \textcolor{red}{Since there is no way to cleanly partition the data stream at the true underlying change points, some data from classes adjacent to the change point may "leak" and contaminate the present block.} 
Since the problem proceeds in an online fashion, the contamination effects may build up and compound downstream errors, accumulating beyond the ideal scale of $O(\sigma\sqrt{MT})$,
% \textcolor{red}{, the best performance achievable by fully exploiting recurrence via ideal data reuse}
signaling the possible pitfalls and subtleties of exploiting recurrence. (See Prop.~\ref{prop:uncapped-ecr}.) 
% \textcolor{red}{More generally, unrestricted reuse could potentially incur higher regret than an algorithm that does not reuse historical data.}

To control this accumulation, we introduce an exposure budget $B_{\exp}$
that limits each historical block's cumulative bias influence while preserving
sufficient reuse to reduce variance, thereby balancing the bias-variance tradeoff by effectively identifying the "sweet spot" of data reuse. Fig.~\ref{fig:reuse-tradeoff} illustrates this point graphically: a larger budget initially
improves performance through greater reuse, but eventually increases regret
by repeatedly propagating residual contamination. 

\textbf{Lower bound and minimax optimality.}
We establish a novel information-theoretic lower bound (Th.~\ref{thm:ecr-lower-bound}) of order
\(
\Omega(
\max\{
\frac{\sigma^2}{\Delta_{\min}}S\log\!\left(\frac{T}{S}\right),
\ \sigma\sqrt{MT}
\}
),
\)
where $\Delta_{\min}$ denotes the minimum separation between distinct parameters.
The lower bound shows that the condition $S=\widetilde{O}(\sqrt{T})$ required to attain the ideal recurrent rate is \textit{necessary} up to logarithmic factors (Cor.~\ref{cor:necessary-change-frequency}). ECR attains this rate under the same condition and essentially matches the lower bound, thereby crisply identifying the minimax complexity of the data reuse problem.
% \textcolor{red}{In a regime where the number of switches is of order $S=\widetilde{O}(\sqrt T)$ the value of recurrence can be  realized; ECR does so and essentially achieves the lower bound, and the minimax complexity of the data reuse problem is crisply identified.} 

We validate the theory with experiments on both synthetic and real data (Sec.~\ref{sec:numerical-results}).
Our formulation arises in a wide range of applications, including load
prediction in cloud computing systems~\citep{maghakian2019online}, temporal changes
in user preferences~\citep{tsymbal2004problem}, and online source
coding~\citep{milan2016forget}. We provide a concrete illustrative example
in Section~\ref{sec:problem-formulation}. 

\subsection{Related literature}
\label{ssec:related-literature}

We discuss related work from two research areas; an extended discussion appears in
App.~\ref{app:related-work}.

\paragraph{Online learning with memory.}

In adversarial prediction with expert advice, long-term-memory methods
exploit comparator sequences that may switch frequently while repeatedly
returning to a much smaller set of
experts~\citep{bousquet2002tracking,zheng2019equipping,herbster2020online,robinson2021improved}. 
In our formulation, both the recurring numerical parameters and their identities must be inferred from noisy observations, and the reused observations may themselves be contaminated.
In the stochastic bandit setting, \citet{re2021exploiting} combine change
detection, recurrence testing, and data reuse in a recurrent
piecewise-stationary formulation.
However, their theoretical guarantee remains change-dependent and does not quantify the resulting gains from recurrence.  
\citet{li2021unifying} study contextual bandits with a finite collection
of recurring parameters, using homogeneity tests for change detection and
historical-model aggregation. However, their analysis relies on
detectability assumptions and, more importantly, does not account for
unavoidable bias in stored blocks. While our model and learning objective
differ, we explicitly account for this contamination and show that even a
small residual bias can accumulate into substantial error under repeated
reuse. 

\paragraph{{Learning under distribution shift.}}
{Existing approaches differ in how they identify and reuse relevant
information. In recurrent concept drift~\citep{suarez2023survey}, some
recognize returning concepts and retrieve past observations or
classifiers~\citep{katakis2010tracking,alippi2013just,anderson2019recurring,chiu2018diversity},
while others weight historical predictors by their performance on
current data~\citep{sun2018concept,zhao2020handling,yang2021concept}.
Most are evaluated empirically; \citet{zhao2020handling} establish a
local benefit when a past model matches the current concept, but their
global guarantee uses a known variation budget. In statistical transfer
learning~\citep{hanneke2019value,cai2021transfer}, \citet{wang2025transfer}
combine informative-source selection with penalized estimation for
offline piecewise-constant means, obtaining minimax rates that depend on
source frequencies and source-target discrepancy. Our analysis accounts for contamination from imperfect online segmentation and establishes provable gains from recurrence under an anytime algorithm.}

\section{Problem formulation}
\label{sec:problem-formulation}
We consider an online learning problem in a recurrent piecewise-stationary environment.  
The learner observes a sequence of real-valued random variables
$\{X_1,\ldots,X_T\}$ whose means satisfy
\[
  \mu_t:=\E[X_t] \in
  \Theta=\{\theta^{(1)},\ldots,\theta^{(M)}\}\subset\R.
\]
We refer to the $M$ values in $\Theta$ as \textit{classes}.
The mean $\mu_t$ is  piecewise constant in time, and switches abruptly at an unknown deterministic sequence of $S$ \textit{change points}:
\[
1=\tau_0<\tau_1<\cdots<\tau_S<\tau_{S+1}=T+1.
\]
The mean is constant on each $\mathcal I_j\triangleq [\tau_j,\tau_{j+1})$, and the observations satisfy
\begin{equation}
\label{eq:observation-model}
X_t=\mu_t+\varepsilon_t,
\end{equation}
where $\{\varepsilon_t\}_{t\ge1}$ are i.i.d., zero-mean, and sub-Gaussian with known 
variance proxy $\sigma^2$.
We require that 
$\mu_{\tau_j}\ne\mu_{\tau_j-1}$ for $j=1,\ldots,S$.
We impose no minimum spacing between change points.
We assume that 
\(
\max_{m,m'}|\theta^{(m)}-\theta^{(m')}|
\le \Delta_{\max}\), 
\(\min_{m\ne m'}|\theta^{(m)}-\theta^{(m')}|
\ge \Delta_{\min},
\)
where
$0 < (M-1)\Delta_{\min}<\Delta_{\max}<\infty$.
Figure~\ref{fig:recurrent-environment} illustrates an example environment.

Let $\mathcal E \coloneqq \mathcal
E_{T,S,M}(\Delta_{\min},\Delta_{\max},\sigma)$ denote the class of all
data-generating processes satisfying the conditions above. 
The learner knows only the noise proxy $\sigma$; it does not know
$T,S,M,\Theta,\Delta_{\min},\Delta_{\max}$, the change points, or the ordered segment
labels. Let $\mathcal F_t=\sigma(X_1,\ldots,X_t)$ be the natural filtration,
and let $\Pi$ be the class of (possibly randomized) adapted policies. In every round $t\ge2$, a policy
$\pi\in\Pi$ outputs a prediction $\widehat\mu_t^\pi$
measurable with respect to $\mathcal F_{t-1}$ and its internal randomness.
We evaluate $\pi$ in environment $\nu\in\mathcal E$ by its expected
cumulative $L_1$ error (regret)
\begin{equation}
\label{eq:regret-definition}
R_T(\pi;\nu)
:=
\E_{\pi,\nu}
\left[\sum_{t=2}^T|\widehat\mu_t^\pi-\mu_t|\right].
\end{equation}
The expectation is over the observation noise and the policy's internal
randomness. The \textit{minimax} regret over $\mathcal E$ is
\begin{equation}
R_T^\star(\mathcal E)
:=
\inf_{\pi\in\Pi}\sup_{\nu\in\mathcal E}R_T(\pi;\nu).
\end{equation}
The supremum ranges over the change-point locations, recurrent parameter
values, ordered segment labels, and admissible noise laws.

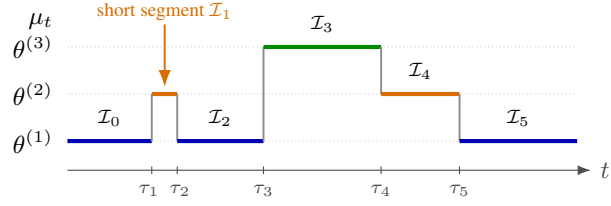
\begin{figure}[t]
  \centering
   \resizebox{0.48\linewidth}{!}{\begin{tikzpicture}[
  x=0.50cm,
  y=0.60cm,
  >=Latex,
  every node/.style={font=\small},
  stateone/.style={draw=blue!70!black,line width=1.5pt},
  statetwo/.style={draw=orange!85!black,line width=1.5pt},
  statethree/.style={draw=green!55!black,line width=1.5pt}
]
  % Class-mean levels and time axis.
  \draw[gray!35,densely dotted] (0,2.8) -- (13,2.8);
  \draw[gray!35,densely dotted] (0,1.8) -- (13,1.8);
  \draw[gray!35,densely dotted] (0,0.8) -- (13,0.8);
  \node[anchor=east] at (-0.12,2.8) {$\theta^{(3)}$};
  \node[anchor=east] at (-0.12,1.8) {$\theta^{(2)}$};
  \node[anchor=east] at (-0.12,0.8) {$\theta^{(1)}$};
  \node[anchor=east] at (-0.12,3.35) {$\mu_t$};
  \draw[->,black!70] (0,0.18) -- (13.35,0.18) node[right] {$t$};

  % The same population-mean trajectory as in the algorithm illustration.
  \draw[stateone] (0,0.8) -- (2.15,0.8);
  \draw[black!45,line width=0.7pt] (2.15,0.8) -- (2.15,1.8);
  \draw[statetwo] (2.15,1.8) -- (2.8,1.8);
  \draw[black!45,line width=0.7pt] (2.8,1.8) -- (2.8,0.8);
  \draw[stateone] (2.8,0.8) -- (5,0.8);
  \draw[black!45,line width=0.7pt] (5,0.8) -- (5,2.8);
  \draw[statethree] (5,2.8) -- (8,2.8);
  \draw[black!45,line width=0.7pt] (8,2.8) -- (8,1.8);
  \draw[statetwo] (8,1.8) -- (10,1.8);
  \draw[black!45,line width=0.7pt] (10,1.8) -- (10,0.8);
  \draw[stateone] (10,0.8) -- (13,0.8);

  % Segment labels; the short segment is labeled in the annotation below.
  \foreach \x/\y/\idx in {1.075/0.8/0,3.9/0.8/2,6.5/2.8/3,9/1.8/4,11.5/0.8/5} {
    \node[above=2pt,font=\scriptsize] at (\x,\y) {$\mathcal I_{\idx}$};
  }

  % True change points; the first two labels are offset slightly for legibility.
  \foreach \x in {2.15,2.8,5,8,10} {
    \draw[black!70] (\x,0.10) -- (\x,0.26);
  }
  \node[black!70,below,xshift=-1pt,font=\scriptsize] at (2.15,0.07) {$\tau_1$};
  \node[black!70,below,xshift=1pt,font=\scriptsize] at (2.8,0.07) {$\tau_2$};
  \foreach \x/\idx in {5/3,8/4,10/5} {
    \node[black!70,below,font=\scriptsize] at (\x,0.07) {$\tau_{\idx}$};
  }

  % Highlight the short segment.
  \node[orange!85!black,align=center,font=\scriptsize] at (2.48,3.55)
    {short segment $\mathcal I_1$};
  \draw[->,orange!85!black,line width=0.7pt]
    (2.48,3.28) -- (2.48,1.95);
\end{tikzpicture}}
  \caption{{\textbf{A recurrent environment.}
  The mean undergoes \(S=5\) changes across \(M=3\) classes. Returns to previously visited classes create opportunities for data reuse. However, if the short segment \(\mathcal I_1\) goes undetected, its observations can become mixed with those from the surrounding \(\theta^{(1)}\) segments. Reusing this mixed history when \(\theta^{(1)}\) returns can introduce bias.}}
  \label{fig:recurrent-environment}
\end{figure}

\paragraph{Illustrative example (recurrent cloud demand).}
Cloud platforms forecast incoming computational demand before provisioning resources~\citep{maghakian2019online}. Let $D_t$ denote demand and $m(z_t)$ a baseline forecast based on available information $z_t$. Suppose the system serves several workload groups, with the subset currently active defining a workload configuration $c_t$ that corresponds to the classes in our model. As groups connect or disconnect, the configuration can change abruptly and may later recur. We model demand as
\[
D_t=m(z_t)+\theta^{(c_t)}+\varepsilon_t,
\qquad c_t\in[M],
\]
where $\theta^{(c_t)}$ is the additive forecast offset associated with the current configuration and $\varepsilon_t$ is noise. When the baseline predictor $m(z_t)$ has already been learned, the remaining task is to estimate the recurrent offset $\theta^{(c_t)}$. Forming the forecast residual,
\(
X_t:=D_t-m(z_t)=\theta^{(c_t)}+\varepsilon_t,
\)
reduces the problem to the recurrence formulation considered in this paper.

\section{The Exposure-Capped Reuse (ECR) Algorithm}
\label{sec:algorithm}
In this section we present ECR, an anytime online algorithm that requires
only knowledge of the noise proxy $\sigma$. We first introduce a
decision/estimation split (Sec.~\ref{ssec:decision-estimation-split}) and
a change detector that partitions the horizon into blocks online
(Sec.~\ref{ssec:change-detection-test}). We then develop a compatibility
test for identifying reusable historical blocks
(Sec.~\ref{ssec:compatibility-test}), followed by an exposure mechanism (Sec.~\ref{ssec:exposure-capped-reuse}). 

\subsection{Decision/estimation split}
\label{ssec:decision-estimation-split}
We split the horizon into decision and estimation streams, $\cS^{\Dstr}$ and $\cS^{\Estr}$, respectively:
\[
\begin{aligned}
\cS^{\Dstr}&:=\{t\in[T]:t\text{ is odd}\}, \,
\cS^{\Estr}:=\{t\in[T]:t\text{ is even}\}.
\end{aligned}
\]
The decision stream is used for change detection and compatibility
decisions, while the estimation stream is used for prediction. This
separation makes the selected estimation samples conditionally independent
of the selection decisions, at the cost of only a constant factor in the
regret. For an interval 
$I\subseteq[T]$ and
$a\in\{\Dstr,\Estr\}$, let $N_I^a:=|I\cap\cS^a|$ and
\(
\bar X_I^a:=\frac{1}{N_I^a}\sum_{u\in I\cap\cS^a}X_u,\quad
\bar\mu_I^a:=\frac{1}{N_I^a}\sum_{u\in I\cap\cS^a}\mu_u.
\)
To lighten notation, we suppress stream superscripts whenever
the relevant stream is clear from context. The appendix uses explicit stream notation.

\subsection{Change detection test}
\label{ssec:change-detection-test}

\textbf{Detection statistic and blocks.}
At time $t$, the learner has access to observations up to time $t-1$ and
keeps track of  the most recent restart time $r<t$. The immediate goal is to detect
whether there was a change point $\tau \in (r,t)$.  
To test for a change, we split the observations at each candidate change point
$k$ with $r<k<t$ and define\footnote{Throughout, an empirical average
  is used only when its sample count is positive.} 
\begin{equation}
  \label{eq:two-samples-statistic}
\widehat D_{k,t}^r
:=
\frac1\sigma
\sqrt{\frac{N_{[r,k)}N_{[k,t)}}{N_{[r,t)}}}
\left|\bar X_{[r,k)}-\bar X_{[k,t)}\right|.
\end{equation}

The scan statistic $C_t^r$ and the alarm time (stopping time) $N_G^r$ are
\begin{equation}
\label{eq:detection-statistic}
\begin{aligned}
C_t^r:=\max_{r<k<t}\widehat D_{k,t}^r, \quad
N_G^r:=\inf\{t>r:C_t^r\ge\gamma_t^r\},
\end{aligned}
\end{equation}
with the time-varying threshold
\begin{equation}
\label{eq:det-threshold}
\begin{aligned}
\hspace{-0.2cm}\gamma_t^r
&:=\sqrt{6\log(t-r)+2\log(1/\alpha_r)
  +2\log(\pi^2/3)},
\end{aligned}
\end{equation}
where $\alpha_r:=\frac6{\pi^2}\frac{\alpha}{r^2}$, $\alpha\in(0,1)$ is a fixed tuning parameter, and
\(C_t^r=0\) when no admissible split exists. Thus,
$C_t^r$ searches over all candidate splits since the last restart for the
strongest evidence of a change, and an alarm is raised when it exceeds the
threshold. 

Whenever the detector raises an alarm, the interval $[r,N_G^r-1)$ is declared a completed
\emph{block}, and the restart time is updated to
$r\leftarrow N_G^r-1$. Let $K$ denote the total number of blocks,
and let
$r_0=1<r_1<\cdots<r_{K-1}<r_K=T$ be the restart times induced by the detector. The resulting blocks are
\[
\cB_b:=[r_b,r_{b+1}),
\qquad b=0,\ldots,K-1.
\]
For block-level statistics, we use the shorthand
$N_b^a:=N_{\cB_b}^a$, $\bar X_b^a:=\bar X_{\cB_b}^a$, and
$\bar\mu_b^a:=\bar\mu_{\cB_b}^a$.

\textbf{Endogenous confounding.}
Detecting a mean shift of size $\Delta$ with failure probability at most $\delta$
requires on the order of
$(\sigma^2/\Delta^2)\log(1/\delta)$
observations~\citep{lai1998information}. Since we impose no minimum spacing
between change points, some changes may 
be missed, and when that happens, the reference samples on the
left side of the split  in Eq.~(\ref{eq:two-samples-statistic}) may contain a
mixture of classes, which confounds the detection of subsequent
changes~\citep{gafni2026the}.
In order to state our detection guarantees in the presence of such
confounding, we define
the \emph{identification scale}
\begin{equation}
\label{eq:ell-id-main}
\ell_{\id}
:=
\left\lceil
C_{\id}\frac{\sigma^2}{\Delta_{\min}^2}
\log\frac{T}{\alpha}
\right\rceil,
\end{equation}
where $C_{\id}$ is a sufficiently large constant to be determined later. 
Note that $\ell_{\id}$ is only used for analysis and is not used  by the ECR algorithm. 
We call an interval $I$
\emph{long} if $N_I\ge\ell_{\id}$, and \emph{short} otherwise.

\begin{proposition}[Anchor segmentation]
\label{prop:anchored-segmentation}
If block $b$ contains a long class-$m$ interval, then, with probability at least $1-\alpha$
\begin{equation}
\label{eq:anchored-segmentation}
|\bar\mu_b-\theta^{(m)}|\le \Delta_{\min}/4.
\end{equation}
\end{proposition}

Eq.~(\ref{eq:anchored-segmentation}) uniquely associates block \(b\) with class \(m\); we
call such a block \emph{anchored} at class \(m\).
Prop.~\ref{prop:anchored-segmentation} shows that the detector partitions
the 
horizon into blocks with well-separated anchors (see
Fig.~\ref{fig:anchored-segmentation}).

\subsection{Compatibility test}

\label{ssec:compatibility-test}

\textbf{Compatibility rule.}
After the detector update at round $t$, let $b_t$ denote the index of the
current block, namely the unique index satisfying
$r_{b_t}<t\le r_{b_t+1}$, and define the \emph{active prefix} $\mathcal A_t:=[r_{b_t},t)$.
We compare each completed block $b<b_t$ with $\mathcal A_t$ using the compatibility statistic
\begin{equation}
\label{eq:compat-stat}
\widehat Q_b(t)
:=
\frac1\sigma
\sqrt{\frac{N_{\mathcal A_t}N_b}
{N_{\mathcal A_t}+N_b}}
\left|\bar X_{\mathcal A_t}-\bar X_b\right|.
\end{equation}

Smaller values of $\widehat Q_b(t)$ indicate greater compatibility. Let
\(
\begin{aligned}
\lambda_t^b
&:=
\sqrt{12\log t+2\log(1/\alpha_b)+2\log(\pi^2/3)},
\end{aligned}
\)
with $\alpha_b:=\frac6{\pi^2}\frac{\alpha}{(b+1)^2}$. The \textit{compatibility set} of blocks declared compatible for reuse at round $t$ is
\begin{equation}
\label{eq:compatible-set}
\cC_t
:=
\left\{b<b_t:
\widehat Q_b(t)<4\lambda_t^b\right\}.
\end{equation}

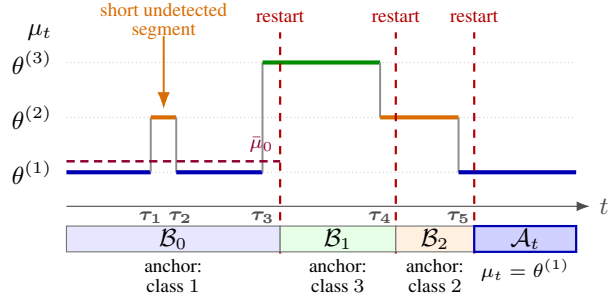
\begin{figure}[t]
  \centering
  \resizebox{0.48\linewidth}{!}{\begin{tikzpicture}[
  x=0.50cm,
  y=0.70cm,
  >=Latex,
  every node/.style={font=\small},
  stateone/.style={draw=blue!70!black,line width=1.5pt},
  statetwo/.style={draw=orange!85!black,line width=1.5pt},
  statethree/.style={draw=green!55!black,line width=1.5pt}
]
  % Class-mean levels and time axis.
  \draw[gray!35,densely dotted] (0,2.8) -- (13,2.8);
  \draw[gray!35,densely dotted] (0,1.8) -- (13,1.8);
  \draw[gray!35,densely dotted] (0,0.8) -- (13,0.8);
  \node[anchor=east] at (-0.12,2.8) {$\theta^{(3)}$};
  \node[anchor=east] at (-0.12,1.8) {$\theta^{(2)}$};
  \node[anchor=east] at (-0.12,0.8) {$\theta^{(1)}$};
  \node[anchor=east] at (-0.12,3.35) {$\mu_t$};
  \draw[->,black!70] (0,0.18) -- (13.35,0.18) node[right] {$t$};

  % Piecewise-stationary population mean.  The short orange segment is undetected.
  \draw[stateone] (0,0.8) -- (2.15,0.8);
  \draw[black!45,line width=0.7pt] (2.15,0.8) -- (2.15,1.8);
  \draw[statetwo] (2.15,1.8) -- (2.8,1.8);
  \draw[black!45,line width=0.7pt] (2.8,1.8) -- (2.8,0.8);
  \draw[stateone] (2.8,0.8) -- (5,0.8);
  \draw[black!45,line width=0.7pt] (5,0.8) -- (5,2.8);
  \draw[statethree] (5,2.8) -- (8,2.8);
  \draw[black!45,line width=0.7pt] (8,2.8) -- (8,1.8);
  \draw[statetwo] (8,1.8) -- (10,1.8);
  \draw[black!45,line width=0.7pt] (10,1.8) -- (10,0.8);
  \draw[stateone] (10,0.8) -- (13,0.8);

  % True change points.
  \foreach \x/\idx in {2.15/1,2.9/2,5/3,8/4,10/5} {
    \node[black!70,above] at (\x,-0.3) {\scriptsize $\boldsymbol{\tau_{\idx}}$};
  }

  % Population-average level of the first detector block.
  \draw[purple!75!black,dash pattern=on 3pt off 2pt,line width=0.9pt]
    (0,1.00) -- (5.45,1.00);
  \node[purple!75!black,anchor=south west,font=\scriptsize]
    at (4.42,1.00) {$\bar\mu_0$};

  % Restart boundaries induced by detector alarms.
  \foreach \x/\idx in {5.45/1,8.4/2,10.4/3} {
    \draw[red!70!black,dashed,line width=0.8pt] (\x,-0.68) -- (\x,3.35);
    \node[red!70!black,above] at (\x,3.35) {\scriptsize restart};
  }

  \filldraw[fill=blue!10,draw=black!55] (0,-0.65) rectangle (5.45,-0.18);
  \filldraw[fill=green!10,draw=black!55] (5.45,-0.65) rectangle (8.4,-0.18);
  \filldraw[fill=orange!12,draw=black!55] (8.4,-0.65) rectangle (10.4,-0.18);
  \filldraw[fill=blue!18,draw=blue!70!black,line width=0.8pt]
    (10.4,-0.65) rectangle (13,-0.18);
  \node at (2.725,-0.415) {$\mathcal B_0$};
  \node at (6.925,-0.415) {$\mathcal B_1$};
  \node at (9.4,-0.415) {$\mathcal B_2$};
  \node at (11.7,-0.415) {$\mathcal A_t$};
  \node[font=\scriptsize,align=center] at (2.725,-1.08) {anchor:\\[-1pt]class $1$};
  \node[font=\scriptsize,align=center] at (6.925,-1.08) {anchor:\\[-1pt]class $3$};
  \node[font=\scriptsize,align=center] at (9.4,-1.08) {anchor:\\[-1pt]class $2$};
  \node[font=\scriptsize] at (11.7,-0.98) {$\mu_t=\theta^{(1)}$};

  % Highlight the short segment whose two changes are missed.
  \node[orange!85!black,align=center,font=\scriptsize] at (2.48,3.55)
    {short undetected\\[-1pt]segment};
  \draw[->,orange!85!black,line width=0.7pt]
    (2.48,3.28) -- (2.48,1.95);
\end{tikzpicture}}
  \caption{{\textbf{Schematic ECR segmentation.} Red dashed lines delimit blocks
  created by detector alarms. Missed changes and detection delays introduce
  contamination, yet block $\cB_0$ remains anchored at class $1$.}}
  \label{fig:anchored-segmentation}
\end{figure}

\textbf{Statistical safety.}
The compatibility test may reject a same-anchor block if its contamination is large. To characterize which blocks are guaranteed to be reused, we introduce the following definition. We call a completed block $b$ anchored at class $m$ \emph{safe} if
\begin{equation}
\big|\bar\mu_b-\theta^{(m)}\big|
\le
\sqrt{\sigma^2/N_b}.
\end{equation}
Equivalently, the block mean lies within a confidence-scale neighborhood of its anchor and is therefore statistically indistinguishable from it. Safety is an analysis-only notion: the learner does not know which completed blocks are safe, since it does not know $\theta^{(m)}$.

\begin{proposition}[Compatibility guarantees]
\label{prop:compatibility-summary}
With probability at least $1-2\alpha$, the following statements
hold for every round $t$ and class $m$:
\begin{enumerate}[label=(\roman*), nosep]

\item \emph{Acceptance:} If $\mathcal{A}_t$ is safe for class $m$, then
every safe block $b<b_t$ anchored at $m$ is compatible.

\item \emph{Rejection:} If $\mathcal{A}_t$ contains a long class-$m$ interval,
then no block anchored at a different class is compatible.
\end{enumerate}
\end{proposition}
 
The acceptance guarantee enables variance reduction. However, an accepted block $b$ may still be contaminated, i.e., $\bar \mu_b \ne \theta^{(m)}$. Repeated reuse of such blocks
can accumulate residual bias, motivating the exposure mechanism introduced next.

\subsection{Exposure-capped reuse}

\label{ssec:exposure-capped-reuse}

For each completed block $b$, define its \textit{exposure parameter} at the beginning of round $t$ as
\begin{equation}
\label{eq:exposure-def}
H_b(t)
:=
\sum_{n<t:\,b\in\mathcal U_n}
\frac1{W_n},
\end{equation}
and define the \textit{reuse set}
\begin{equation}
\label{eq:reused-set}
\mathcal U_t
:=
\{b\in\cC_t:\ H_b(t)<B_{\exp}\},
\end{equation}
where $B_{\exp}>2$ is the exposure budget.
The resulting effective sample size, i.e., the number of observations used to form the prediction, is
\begin{equation}
\label{eq:W-def}
W_t
:=
N_{\mathcal A_t}+\sum_{b\in\mathcal U_t}N_b.
\end{equation}

These quantities are defined recursively:
$H_b(t)$ depends only on reuse decisions and effective sample sizes from
rounds $n<t$; it determines $\mathcal U_t$, which in turn determines $W_t$.
The exposure control in Eq.~(\ref{eq:reused-set}) is \textit{necessary}
for ECR: Prop.~\ref{prop:uncapped-ecr} gives a construction in which uncapped ECR incurs \(\Omega\!\left(T^{5/8}/(\log
  T)^{1/4}\right)\) regret, exceeding the perfect-reuse benchmark scale of $O(\sqrt{T})$.

\textbf{Output estimator.}
ECR outputs the average of the active prefix together with all
currently reusable completed blocks:
\begin{equation}
\label{eq:CR-estimator}
\widehat\mu_t
:=
\frac{
\displaystyle\sum_{n\in\mathcal A_t}X_n
+\displaystyle\sum_{b\in\mathcal U_t}
\sum_{n\in\cB_b}X_n}
{W_t}.
\end{equation}

We next explain how the exposure mechanism balances the bias-variance tradeoff.

\textbf{\textit{Limiting bias.}}
When the anchor of block $b$ agrees with the active class, say class $m$, its
population bias is
\(
\left|\bar\mu_b-\theta^{(m)}\right|.
\)
Its cumulative bias is therefore at most
\[
N_b
\left|\bar\mu_b-\theta^{(m)}\right|
\sum_{\substack{t:b\in\mathcal U_t}}
\frac{1}{W_t}
=
N_b
\left|\bar\mu_b-\theta^{(m)}\right|
H_b(T+1).
\]
Since \(H_b(T+1)\le B_{\exp}+1\), a fixed block’s cumulative bias contribution remains bounded.

\textbf{\textit{Reducing variance.}}
Retiring blocks limits the accumulation of contamination bias, but
may also reduce the effective sample size available for reuse. The key design
principle is that retired blocks are gradually replaced by newly formed blocks from recurrent occurrences of the same class. This replenishes the sample mass available for reuse. Our analysis shows that this mechanism preserves variance reduction.
Figure~\ref{fig:reuse-tradeoff} illustrates this tradeoff.

\begin{algorithm}[tb]
\caption{Exposure-Capped Reuse (ECR)}
\label{alg:ECR}
\begin{algorithmic}[1]
\STATE \textbf{Input:} known proxy $\sigma$; fix $\alpha\in(0,1),$
$B_{\exp}>2$
\STATE Observe $X_1$; set $r\gets1$ and $b\gets0$
\FOR{$t=2,3,\ldots$}
  \STATE Compute $C_t^r$ and $\gamma_t^r$ (Eqs.~\plaineqref{eq:detection-statistic}--\plaineqref{eq:det-threshold})
  \IF{$C_t^r\ge\gamma_t^r$}
    \STATE Store $\cB_b\gets[r,t-1)$; set $H_b\gets0$
    \STATE $b\gets b+1$; $r\gets t-1$
  \ENDIF
  \STATE Set $\mathcal A_t\gets[r,t)$; form $\cC_t$, $\mathcal U_t$, and $W_t$ using
  Eqs.~\plaineqref{eq:compatible-set}, \plaineqref{eq:reused-set}, and
  \plaineqref{eq:W-def}
  \STATE \textbf{predict} using Eq.~\plaineqref{eq:CR-estimator} if $W_t>0$;
  otherwise use the most recent past estimation observation 
  \STATE If $W_t>0$, set $H_j\gets H_j+1/W_t$, \quad $\forall$
  $j\in\mathcal U_t$
  \STATE Observe $X_t$ and update its parity-defined stream
\ENDFOR
\end{algorithmic}
\end{algorithm} 

\textbf{Computational complexity.}
Let $K_t$ be the number of completed blocks and $d_t=t-r$ the length of the
current block. Each split and each compatibility check costs $O(1)$. ECR therefore uses
$O(d_t+K_t)=O(t)$ time per round and $O(d_t+K_t)=O(t)$ memory. Restricting the
detector to a geometric multiscale grid~\citep{lai2010sequential} reduces the
number of splits to $O(\log d_t)$, giving
$O(\log d_t+K_t)$ time per round (App.~\ref{app:efficient-ecr-simulation}).

\section{Regret analysis}
\label{sec:regret-analysis}

\subsection{Upper bound}
\label{ssec:upper-bound}
\begin{theorem}[Regret of ECR]
\label{thm:ecr-upper-bound}

Fix $\alpha\in(0,1), B_{\exp}=4$. For every $\nu\in\mathcal E$,
\begin{align}
R_T(\mathrm{ECR};\nu)
&\le
C\sigma\sqrt{MT}
\label{eq:ecr-upper-bound-main}\\
&\hspace{-0.3cm}
+C\sigma(1+\sqrt M)
\left((S+1)\sqrt{\ell_{\id}}+\alpha L_T\right)
\nonumber\\
&\hspace{-0.3cm}
+C\Delta_{\max}(L_T+\sqrt M)
\left((S+1)\ell_{\id}+\alpha L_T\right),
\nonumber
\end{align}
where $C$ is a universal constant and $L_T:=1+\log(\frac{T}{\alpha})$.
\end{theorem}

The first term in Eq.~(\ref{eq:ecr-upper-bound-main}) is the recurrent
estimation rate. The remaining terms account for the costs of identifying
the active class and controlling contamination in the reused
data. For fixed $\sigma,\Delta_{\min},\Delta_{\max}$, the bound simplifies to
$R_T(\mathrm{ECR};\nu)=\widetilde O(\sqrt{MT}+S\sqrt M)$, yielding the following corollary.

\begin{corollary}[Exploiting recurrence]
\label{cor:ecr-recurrent-rate}

Fix $\sigma,\Delta_{\min},\Delta_{\max}$. If
\begin{equation}
\label{eq:sparse-change-regime}
S=\widetilde O(\sqrt T),
\end{equation}
then, for every $\nu\in\mathcal E$,
\(
R_T(\mathrm{ECR};\nu)
=\widetilde O(\sigma\sqrt{MT}).
\)
\end{corollary}

Eq.~(\ref{eq:sparse-change-regime}) characterizes the regime in which the
regret of ECR scales with the number of classes $M$, rather than the
number of changes $S$. Intuitively, this condition captures the "price of
non-stationarity": the perfect-reuse rate can be attained only when changes are not
too frequent, since every change induces an unavoidable detection
cost. 
Note that at the highest admissible switching frequency, \(S\asymp\sqrt T\), and for fixed
\(M\), the no-reuse benchmark regret rate is $O(\sqrt{ST}) \asymp  T^{3/4}$, while that of ECR is 
\(\widetilde O(T^{1/2})\). 
The lower bound below shows that the condition in Eq.~(\ref{eq:sparse-change-regime}) is necessary up to
logarithmic factors, establishing that it is an inherent limitation characteristic of online
recurrent learning. 

\paragraph*{Proof sketch.}

The proof is given in Appendix~\ref{app:upper-bound}; we summarize the main
ideas.

\noindent\textbf{Step 1: Regret decomposition.}
Conditioning on the decision stream freezes the adaptive sample selection and
centers the selected estimation noise. Lemma~\ref{lem:regret-decomposition}
decomposes the regret into stochastic-estimation and population-bias terms.

\noindent\textbf{Step 2: Segmentation and compatibility.}
On the detector concentration event (App.~\ref{app:segmentation}), Prop.~\ref{prop:anchored-segmentation} shows that a block containing a sufficiently long regime has a unique anchor. Lemma~\ref{lemma:global-contamination-control} shows that the total off-anchor and unanchored mass is \(O(S\ell_{\id})\). Prop.~\ref{prop:compatibility-summary} then ensures that only compatible blocks are reused.

\noindent\textbf{Step 3: Bounding the population bias.}
For a completed block \(b\) anchored at class $m$, its cumulative bias contribution on same-class reuse rounds is at most
\(
N_b\left|\bar\mu_b-\theta^{(m)}\right|H_b(T+1)
\) (see Sec.~\ref{ssec:exposure-capped-reuse}).
The exposure cap prevents the same residual
contamination from being charged indefinitely. The global contamination bound
from Step~2 then controls the total contribution from same-anchor blocks.

\noindent\textbf{Step 4: Recurrent variance reduction.}
For each class \(m\), Lemma~\ref{lem:variance-upper-bound} constructs a capacity
counter \(x_m(t)\) that combines observations in the active block with the
remaining exposure capacity of completed same-class blocks, and satisfies
\(W_t\gtrsim x_m(t)\). The capacity gained from new observations dominates the capacity spent on reuse.  A square-root telescoping argument across all returns to class \(m\) then shows that the estimation error scales with $M$ rather than $S$.
Capacity spent during bad intervals appears as leakage and is controlled separately.

\noindent\textbf{Step 5: Detector deviations.}
The preceding arguments hold on the detector and compatibility concentration events. Lemma~\ref{lem:robust-detector-localization} removes this
conditioning: detector deviations are localized to block prefixes with expected total length at most \(\alpha L_T\). Compatibility deviations contribute only
\(O(\alpha)\) additional rounds, while Lemma~\ref{lem:false-alarms} gives
\(\E[K]\le S+1+\alpha\). Substituting the resulting bias and variance bounds
into Lemma~\ref{lem:regret-decomposition} completes the proof.

\subsection{Lower bound}

\label{ssec:lower-bound}

\begin{theorem}[Minimax lower bound]
\label{thm:ecr-lower-bound}

Write $\rho:=\Delta_{\max}-(M-1)\Delta_{\min}$, and assume $2\le M\le S+1$.  There exist universal constants
$c,C>0$ such that, if
\begin{equation}
\label{eq:ecr-lower-bound-conditions}
T\ge C\left(S+\frac{\sigma^2M^3}{\rho^2}\right),
\; \;
\sigma^2\log(T/S)\le c\Delta_{\min}^2\sqrt{T/S},
\end{equation}
then
\begin{equation}
\label{eq:ecr-lower-bound-main}
R_T^\star(\mathcal E)
\gtrsim
\max\left\{
\frac{\sigma^2}{\Delta_{\min}}\,
S\log(T/S),
\ \sigma\sqrt{MT}
\right\}.
\end{equation}
\end{theorem}

The first term in Eq.~(\ref{eq:ecr-lower-bound-main}) captures the unavoidable cost of adapting to unknown change points, while the second is the recurrent estimation cost.
For fixed $M,\sigma,\Delta_{\min}$, and $\Delta_{\max}$, the conditions in
Eq.~(\ref{eq:ecr-lower-bound-conditions}) hold for all sufficiently large $T$
whenever $T/S\to\infty$.

\begin{corollary}[Change-frequency condition]
\label{cor:necessary-change-frequency}

Fix $M,\sigma,\Delta_{\min}$, and $\Delta_{\max}$, and suppose that
$T/S\to\infty$. If
\(
S=\widetilde\omega(\sqrt T),
\)
then 
\(
{R_T^\star(\mathcal E)=\widetilde\omega(\sigma\sqrt{ M T}).}
\)
\end{corollary}

%\hl{AZ: isn't it better to give the lower bound also as function of $M$ and $\sigma$ so that the match to teh upper bound is more evident? After all we write it that way in (18) no?  }. 
% Agree, fixed. 
Corollary~\ref{cor:necessary-change-frequency} shows that the condition
$S=\widetilde O(\sqrt T)$ in Eq.~(\ref{eq:sparse-change-regime}) is necessary, up to logarithmic factors, for attaining the recurrent rate
$\widetilde O(\sigma\sqrt{MT})$.
The proof of Theorem~\ref{thm:ecr-lower-bound} is given in
App.~\ref{app:lower-bound}. 

\paragraph*{Proof sketch.}

We consider two oracle-aided Gaussian subclasses; revealing
information can only help the learner. For the change-detection term, we reveal
the class means, ordered class labels, and \(S\) disjoint macroblocks of length \(\asymp T/S\),
leaving one change location per macroblock unknown. Comparing early and late
change environments creates a tradeoff between false alarms and detection delay. Gaussian
change of measure over disjoint windows gives
\(\Omega((\sigma^2/\Delta_{\min})\log(T/S))\)
loss per macroblock. Summing over the \(S\) macroblocks gives the first term.

For the recurrent-mean term, we reveal all change points and class labels and allocate
\(n\asymp T/M\) observations to each class. We independently perturb each of
the \(M\) separated class means by \(\pm\xi\), where
\(\xi\asymp\sigma/\sqrt n\). The stated conditions ensure that this family
preserves the required separation and diameter. A coordinate-wise Le Cam
argument yields \(\Omega(n\xi)=\Omega(\sigma\sqrt n)\) loss per class,
and hence \(\Omega(\sigma\sqrt{MT})\) overall. Taking the maximum of the two
oracles proves the theorem.

\subsection{Discussion and extensions} \vspace{-0.1cm}
\label{ssec:discussion}
\vspace{-0.1cm}

Retrospective segmentation refinement cannot guarantee exact change-point recovery, leaving contamination near estimated boundaries (alarms); see App.~\ref{app:retrospective}. Other reuse strategies, such as downweighting older observations, face a related bias-variance tradeoff.

For $L_p$ loss with $1<p<2$, we derive the benefit of recurrent pooling over the no-reuse benchmark: the estimation scale improves from \(\sigma^p(S+1)^{p/2}T^{1-p/2}\) to \(\sigma^pM^{p/2}T^{1-p/2}\); see App.~\ref{app:extension-p}. 
%A heuristic comparison with change-time costs suggests that attaining the recurrent benchmark scale requires \(S=\widetilde O(T^{1-p/2})\), for fixed model parameters. 
Establishing a corresponding regret guarantee for ECR remains open.
For vector-valued observations under Euclidean loss, we outline an extension of ECR with dimension-dependent confidence thresholds and a recurrent estimation term \(O(\sigma\sqrt{dMT})\). For fixed dimension, the resulting bound retains the recurrent rate under the same sparsity regime \(S=\widetilde O(\sqrt T)\); see App.~\ref{app:vector-extension}.

\section{Numerical results}
\label{sec:numerical-results}

{We evaluate ECR on synthetic and real data; additional experiments and
implementation details appear in App.~\ref{app:extended-simulations}.}

\begin{figure}[!t]
  \centering
  \setlength{\abovecaptionskip}{3pt}
  \begin{minipage}[t]{0.49\linewidth}
    \vspace{0pt}
  \includegraphics[width=\linewidth]{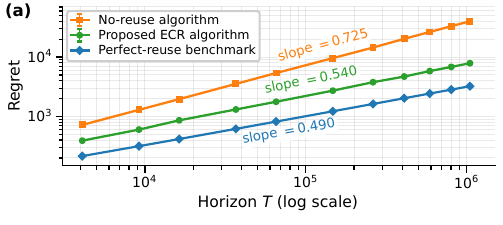}
  \end{minipage}\hfill
  \begin{minipage}[t]{0.49\linewidth}
    \vspace{0pt}
  \includegraphics[width=\linewidth]{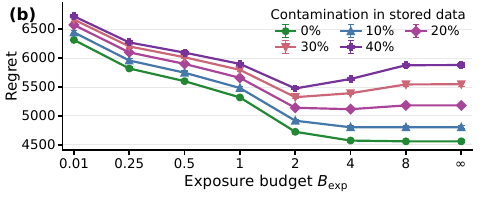}
  \end{minipage}
  \par\vspace{5pt}
  \begin{minipage}[t]{0.49\linewidth}
    \vspace{0pt}
  \begin{tikzpicture}
    % Restore the environment panel above the regret panel.
    \node[inner sep=0pt,anchor=north west] (realenv) at (0,0)
      {\includegraphics[width=\linewidth,trim={0 0.20in 3.5in 0.152in},clip]{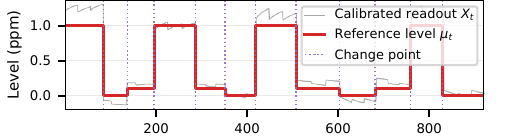}};
    \node[overlay,anchor=west,inner sep=0pt,
      font=\fontsize{8}{9}\selectfont,
      xshift={\dimexpr0.921537\linewidth+1.5pt\relax},
      yshift=0.019146\linewidth] at (realenv.south west) {$t$};
    \node[overlay,anchor=north west,inner sep=0pt,
      font=\sffamily\bfseries\fontsize{8}{9}\selectfont,
      xshift=0.018\linewidth,yshift=6pt] at (realenv.north west) {(c)};
  \end{tikzpicture}
  \end{minipage}\hfill
  \begin{minipage}[t]{0.49\linewidth}
    \vspace{0pt}
  \begin{tikzpicture}
    % Preserve the regret plot and update its panel marker.
    \node[inner sep=0pt,anchor=north west] (realpanel) at (0,0)
      {\includegraphics[width=\linewidth]{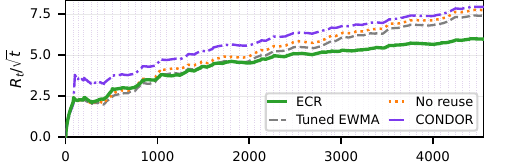}};
    \node[overlay,anchor=west,inner sep=0pt,
      font=\fontsize{8}{9}\selectfont,
      xshift={\dimexpr0.921537\linewidth+1.5pt\relax},
      yshift=0.019146\linewidth] at (realpanel.south west) {$t$};
    \node[overlay,anchor=north west,inner sep=0pt,
      font=\sffamily\bfseries\fontsize{8}{9}\selectfont,
      xshift=0.018\linewidth,yshift=6pt] at (realpanel.north west) {(d)};
  \end{tikzpicture}
  \end{minipage}
  \caption{{\textbf{Numerical results.}
  (a) Synthetic regret on log--log axes.
  (b) Exposure-budget tradeoff at different levels of source contamination.
  (c) Electronic-nose readouts and reference levels for the first $12$ segments.
  (d) Regret over the full replay, normalized by $\sqrt{t}$.
  Error bars in (a) and (b) show $95\%$ Monte Carlo intervals.}}
  \label{fig:synthetic-numerical-results}
  \label{fig:real-numerical-results}
  \label{fig:enose-environment}
\end{figure}
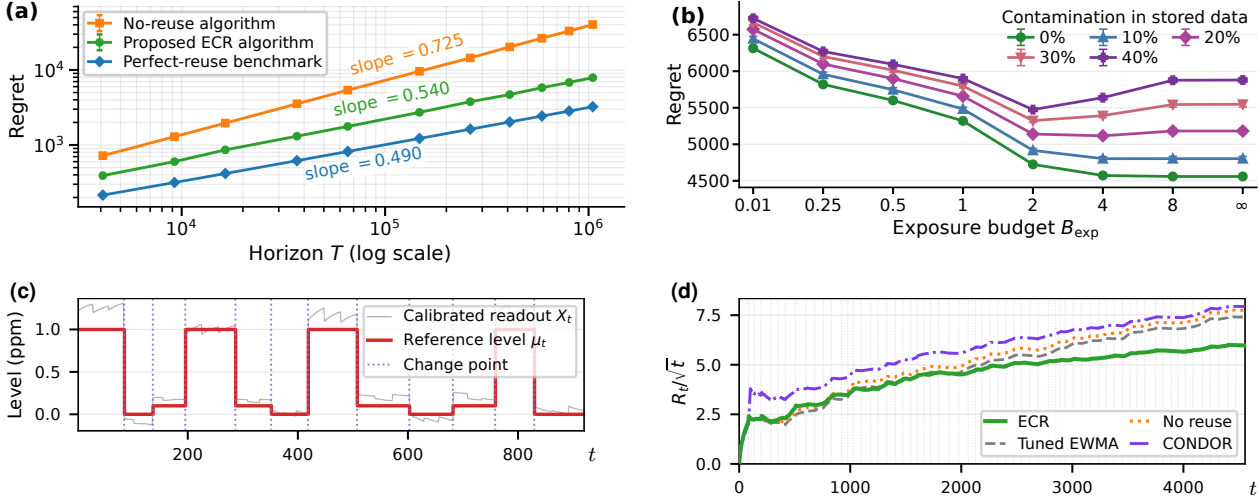

\textbf{Synthetic experiments.}
Figure~\ref{fig:synthetic-numerical-results}(a) provides empirical support for the regret scaling in Th.~\ref{thm:ecr-upper-bound} using $M=4$ Gaussian means $\{-9,-3,3,9\}$ across $S+1=\sqrt T/4$ segments.
The no-reuse baseline~\citep{gafni2026the} uses the same change detector as ECR but estimates the current mean afresh
after each alarm. 
For \(S\asymp\sqrt T\) and fixed \(M\), the theoretical rates of the no-reuse and perfect-reuse benchmarks scale as \(T^{3/4}\) and \(T^{1/2}\), respectively. The empirical slopes of the no-reuse algorithm (\(0.725\)) and perfect-reuse benchmark (\(0.490\)) are close to these exponents. ECR’s slope (\(0.540\)) is close to that of perfect reuse, consistent with the recurrence-dependent rate in Cor.~\ref{cor:ecr-recurrent-rate}.

{\textbf{Contamination and exposure control.}
Figure~\ref{fig:synthetic-numerical-results}(b) compares exposure budgets at
different contamination levels in stored data (controlled by the duration of
short, mostly undetected segments). With clean history (green curve), regret decreases
as more reuse is allowed. With more contaminated history,
the curves become U-shaped: additional reuse initially helps, but eventually
increases regret. The best-performing budgets therefore become smaller as
contamination increases, illustrating the balance between the statistical
benefit of reuse and repeated contamination bias. Implementation details appear in
App.~\ref{app:contamination-budget-experiment}.}

\textbf{Real-world data: electronic-nose assay.}
We use recorded sensor readings at three nominal diacetyl levels
$(0,0.1,1)$ ppm from the year-long dataset of \citet{worner2025long}.
Randomizing condition order and block durations within each of $21$
held-out days gives noncyclic replays with $63$ segments.
{We also include an adaptation of CONDOR~\citep{zhao2020handling} to scalar regression, using ECR's change detector.}
Calibration and
online parameters are selected on earlier days; class labels and change
points are unknown during evaluation.
Figure~\ref{fig:real-numerical-results}{(d)}
shows that ECR benefits from recurrence and outperforms all three baselines. ECR selectively reuses compatible history from earlier recurrences, whereas an exponentially weighted moving average (EWMA) discounts observations by age and no reuse starts afresh after each alarm. 
CONDOR periodically assigns equal weights to its stored models, including those from other classes, which can increase bias. ECR uses compatibility testing to limit such mixing.

\section{Conclusions}
\label{sec:conclusions}
We characterize the statistical complexity of online data reuse under recurrent non-stationarity. We develop ECR, a nearly minimax-optimal algorithm that realizes the statistical benefit of recurrence. We believe that the principles developed here, particularly those concerning repeated bias accumulation, may be useful beyond the specific formulation studied in this paper. An important direction for future work is to extend these ideas to sequential decision-making settings, where contaminated reuse interacts with exploration and action-dependent observations.

\bibliographystyle{icml2026}
\bibliography{data_reuse_non_stationary}

\clearpage
\appendix
\onecolumn

\section{Extended related literature}
\label{app:related-work}

\paragraph{Recurrence in sequential prediction.}
Long-term-memory prediction distinguishes the number of switches from the number of distinct predictors used by a comparator~\citep{bousquet2002tracking,robinson2021improved}. For bounded losses with full-information feedback, \citet{zheng2019equipping} obtain expected regret
\(O(\sqrt{T((S+1)\log T+M\log K)})\)
against a comparator that switches $S$ times among $M$ of $K$ available experts, with parameters tuned to $T,S,M$. They also provide a variant that adapts to $S,M$, replacing $M\log K$ by $M\log(K\log T)$. Recurrence reduces the expert-identification cost while retaining a switch-dependent term. \citet{herbster2020online} extend long-term memory to multitask prediction and RKHS hypothesis classes. In our formulation, the learner must estimate the recurrent means themselves from noisy observations and determine which potentially contaminated historical data to reuse.

For probabilistic sequence prediction, the Forget-me-not Process averages over binary temporal partitions and mixtures of fresh and stored model states~\citep{milan2016forget}. With a base-model redundancy of $O(\log(n+1))$ on $n$ stationary observations, its stated log-loss regret for $S$ changes is $O((S+1)\log^2 T)$. This bound does not improve when sources repeat, although recurrence yields empirical gains. \citet{shamir2004universal} establish asymptotically matching coding bounds that separate learning distinct source states, locating transitions, and assigning segments to states. These results concern comparator regret or coding redundancy; our objective instead measures cumulative $L_1$ error in an unobserved recurrent mean.

\paragraph{Recurrent bandits and data reuse.}
Routine Bandits pools data across recurring Gaussian bandit instances with known, equal-length period boundaries~\citep{saber2021routine}. For a fixed number of periods, its asymptotically optimal regret constant sums once over each distinct instance. Known boundaries avoid the segmentation-induced contamination arising in our setting from detection delays and missed changes. With unknown changes, \citet{re2021exploiting} combine detection, retrospective localization, and recurrence testing under jump and phase-length conditions. Their guarantee remains change-dependent, while recurrence gains are demonstrated empirically. For Gaussian linear bandits, DyClu uses homogeneity tests for change detection and historical-model aggregation~\citep{li2021unifying}. Its stated high-probability regret bound has the form
\(O(\sigma d\sqrt{T}\log T\sum_{k=1}^{M}\sqrt{p_k}+C\Gamma_T)\),
where $d$ is the feature dimension, $p_k$ is the fraction of interactions governed by recurrent parameter $k$, $M$ counts distinct parameters, $\Gamma_T$ counts stationary periods across users, and $C$ depends on problem and testing parameters. Since $\sum_k\sqrt{p_k}\le\sqrt M$, the leading term already benefits from recurrence. The analysis assumes detectability, parameter separation, and regular contexts, with additional period-length conditions. Although incoming observations are tested for homogeneity, false negatives can contaminate an archived history. The proof does not account for this residual bias or its amplification under repeated reuse.

\paragraph{Learning under recurring concept drift.}
Concept-drift methods recognize returning concepts or weight historical predictors by their performance on current data~\citep{suarez2023survey}. 
CONDOR reuses stored predictors through exponential weight updates based on their losses in the current epoch, rather than testing historical blocks for compatibility~\citep{zhao2020handling}. Their weighted combination also regularizes new models. For convex losses in $[0,1]$, suitable learning-rate tuning gives local cumulative regret
\(O(\sqrt{L_k^\star\log K}+\log K)\),
where $K$ bounds the model-pool size and $L_k^\star$ is the best stored predictor's cumulative loss in epoch $k$. Recurrence improves this guarantee when a retained model predicts the returning concept well. The global proof instead uses the worst-case local bound, yielding $O(V_T^{1/3}T^{2/3})$ dynamic regret against the best model at each round, with logarithmic dependence on $K$ suppressed. 
The update period is chosen using the function variation $V_T$ or a known upper bound. This global guarantee does not distinguish returning concepts from new ones. Our analysis quantifies a recurrence-dependent cumulative tracking rate without knowledge of the non-stationarity parameters.

\paragraph{Statistical borrowing and offline estimation.}
Statistical borrowing methods balance variance reduction against the bias introduced by combining populations that may differ. Examples include preliminary testing and shrinkage~\citep{shah2017merging}, penalties on differences between regression coefficients~\citep{chen2015data}, and regularization toward auxiliary hypotheses~\citep{kuzborskij2017fast}. \citet{wang2025transfer} derive minimax rates for offline piecewise-constant mean estimation that depend on source frequencies and source-target discrepancy, and develop informative-source selection. For online bandits, \citet{cheung2024leveraging} establish tight regret guarantees using a known bound on offline-online mismatch; without a nontrivial bound, offline data need not improve the minimax rate. Our additional difficulty is that historical blocks are generated through imperfect online segmentation and reused across many later predictions. Consequently, even exact parameter recurrence can leave residual contamination in stored data. Compatibility testing assesses whether reuse is supported by current observations, while exposure control separately bounds each block's cumulative influence on future predictions.
In offline mean estimation, \citet{mazhar2018bayesian} jointly estimate change points and shared levels, deriving an oracle inequality under squared loss for the global penalized estimator. They use the complete sequence, whereas our predictions are causal.

\section{Extended simulations}
\label{app:extended-simulations}

We first study how contamination accumulates under repeated reuse,
followed by experiments with CPU traces, vector-valued observations, and
different loss powers. We then examine the effect of reducing the number
of candidate splits in the detector. Implementation details for the
main-paper experiments in Section~\ref{sec:numerical-results} appear at
the end of this section.
Unless stated otherwise, synthetic observations have independent Gaussian
noise, methods are compared on the same observation paths, and error bars
denote pointwise $95\%$ Monte Carlo intervals.

Code and data are available at
\url{https://anonymous.4open.science/r/ecr-reproducibility-9F1B/}.

\subsection{Exposure capping under persistent contamination}
\label{app:hard-contamination}

The main-paper exposure experiment varies how much historical reuse is
allowed. Here, we fix the exposure budget and extend the subsequent recurrent
sequence to examine how the cost of contaminated reuse accumulates over time.
We construct a prefix with rapid, unresolved changes that enters the
repository as a mixed block. ECR and uncapped reuse initially coincide, but
once this block reaches its exposure limit, ECR retires it while uncapped
reuse continues to propagate its bias.
Figure~\ref{fig:hard-contamination-cumulative} shows the resulting gap in
cumulative regret as further recurrences are added. This experiment
illustrates the cost of repeatedly reusing contaminated observations;
Proposition~\ref{prop:uncapped-ecr} establishes the necessity of the cap
for ECR through an asymptotic counterexample. In the benign version of the
exposure experiment described in Appendix~\ref{app:exposure-experiment},
additional reuse remains beneficial because stored observations contain
little contamination.

\begin{figure*}[!t]
  \centering
  \includegraphics[width=0.88\textwidth]{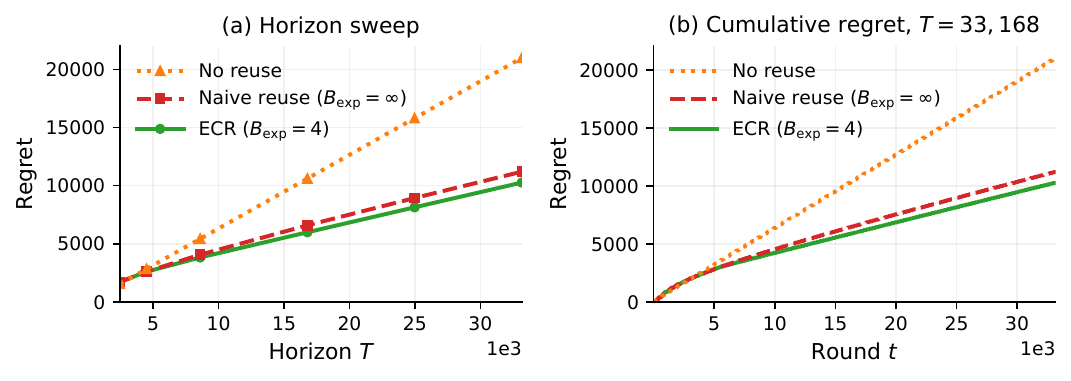}
  \caption{\textbf{Exposure capping under persistent contamination.}
  (a) Terminal cumulative $L_1$ regret as the subsequent recurrent sequence
  grows. (b) Cumulative regret at $T=33,168$. Results average $100$ paired
  paths.}
  \label{fig:hard-contamination-cumulative}
\end{figure*}

\paragraph{Implementation details.}
We use $\sigma=1$. During the first $400$ rounds, the mean equals $2$ for four rounds
and $0$ for one round, repeating this five-round pattern.
On all simulated paths, the detector misses all $159$ rapid changes, and the
prefix is stored as one mixed block with population mean $1.6$.
The subsequent sequence alternates
between mean-$64$ and mean-zero segments of length $256$. We vary the number
of recurrent pairs from $4$ to $64$, giving horizons from $2,448$ to
$33,168$, and average over $100$ paired paths. All methods use the same
detector and compatibility rule with their unscaled thresholds. ECR sets
$B_{\exp}=4$, naive reuse removes the cap, and no reuse sets $B_{\exp}=0$.

\subsection{Trace-driven NAB replay}
\label{app:nab-replay}

We next examine whether the benefit of reuse persists with noise from a
recorded, temporally dependent trace and irregular returns to previous
classes. We use an AWS CPU-utilization series from the Numenta Anomaly
Benchmark (NAB)~\citep{lavin2015evaluating}, with the learner predicting the
current mean utilization. Since NAB does not provide ground-truth recurrent
workload labels, we impose the recurrence schedule and use residuals from
approximately constant-mean intervals of the trace to generate the observations.
Figure~\ref{fig:nab-irregular-replay} shows one replay and the average regret
over $100$ evaluation paths. ECR's mean terminal regret is $16.9\%$ lower
than without reuse, and its terminal regret is lower than EWMA's on all
$100$ evaluation paths.

\begin{figure*}[!t]
  \centering
  \includegraphics[width=0.88\textwidth]{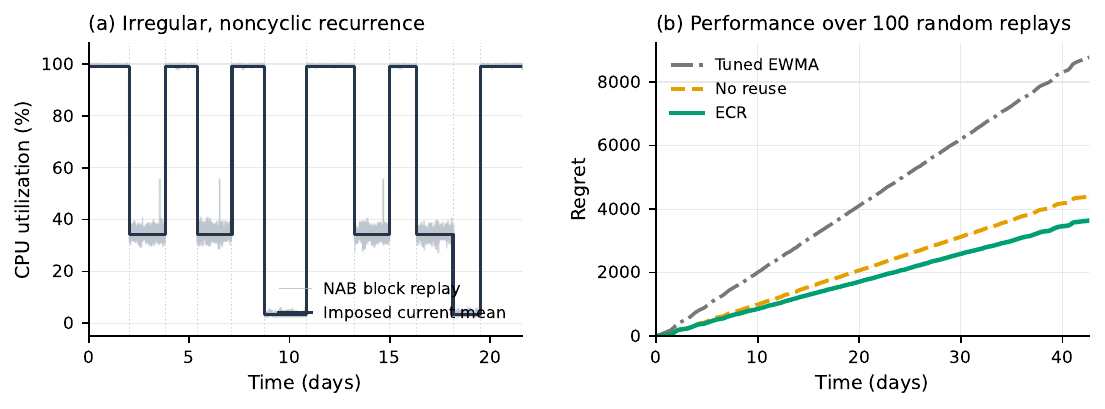}
  \caption{\textbf{Trace-driven NAB replay.}
  (a) The first $12$ segments of one irregular, noncyclic replay.
  (b) Mean cumulative $L_1$ regret over $100$ randomized replays.}
  \label{fig:nab-irregular-replay}
\end{figure*}

\paragraph{Implementation details.}
We use the intervals with low, middle, and high mean utilization identified by
\citet{gafni2026the}, excluding observations near the source boundaries.
Their empirical means are approximately $(3.3,34.2,99.1)$. We center the
residuals within each interval and
replay consecutive observations in circular blocks of length $12$, preserving
their local order. Each replay has $T=12,288$, $M=3$,
and $S=23$. The initial class is drawn uniformly, and each subsequent class
is chosen uniformly from the two inactive classes. Segment durations are
randomized around $512$ observations and rescaled to sum to $T$.
The detector uses the empirical noise proxy $\widehat\sigma=2.094$, the largest within-class
standard deviation. ECR uses $\alpha=0.05$, $B_{\exp}=4$, and compatibility
multiplier $1$; no reuse uses the same detector with $B_{\exp}=0$. The EWMA
update rate $\eta=0.4$ is selected on $100$ separate replay paths, and
evaluation uses $100$ new paired replays.

\subsection{Vector-valued recurrence}
\label{app:vector-simulation}

We compare ECR, no reuse, and the perfect-reuse benchmark in a synthetic
Gaussian environment with $M=4$ recurring mean vectors in $\mathbb R^5$,
using Euclidean loss. We vary $T$ from $1{,}024$ to $65{,}536$, with
$S+1=\sqrt T$ equal-length segments. In Figure~\ref{fig:vector-lp-experiments}(a),
ECR's regret remains close to a fixed multiple of the recurrent benchmark's
loss as the horizon grows, while the ratio without reuse increases more
quickly. The fitted log--log slopes of normalized regret are $0.040$ for
ECR and $0.204$ without reuse. Thus, the benefit observed in the scalar
experiment persists in this vector-valued setting.

\begin{figure*}[!t]
  \centering
  \begin{minipage}[t]{0.465\textwidth}
    \centering
    (a) Vector-valued recurrence\par
    \includegraphics[width=\linewidth]{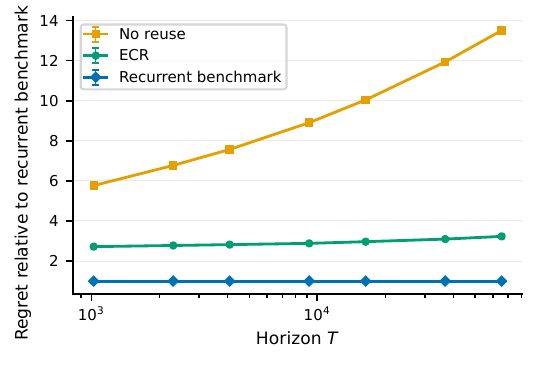}
  \end{minipage}\hspace{0.02\textwidth}%
  \begin{minipage}[t]{0.465\textwidth}
    \centering
    (b) Online benefit under $L_p$ loss\par
    \includegraphics[width=\linewidth]{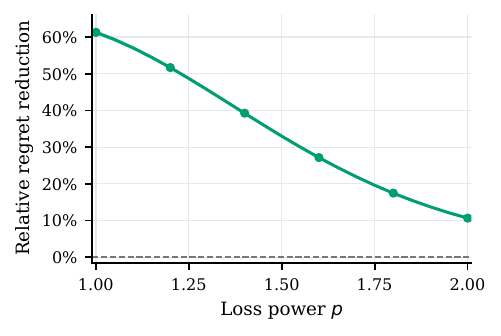}
  \end{minipage}
  \caption{(a) \textbf{Vector-valued recurrence.}
  Euclidean regret relative to the recurrent benchmark for $d=5$, $M=4$, and
  $S+1=\sqrt T$, averaged over $40$ paired paths.
  (b) \textbf{Online benefit under $L_p$ loss.}
  Relative regret reduction of ECR over no reuse at $T=262,144$, averaged
  over $30$ paired paths.}
  \label{fig:vector-lp-experiments}
\end{figure*}

\begin{figure*}[!t]
  \centering
  \includegraphics[width=0.93\textwidth]{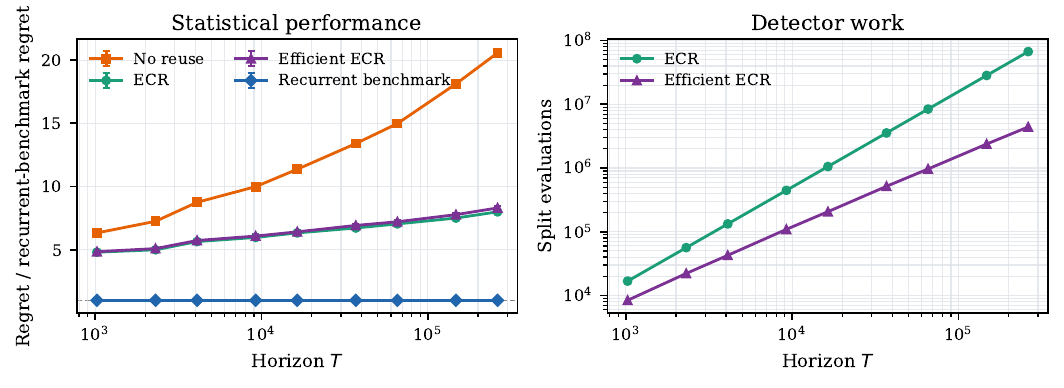}
  \caption{\textbf{Geometric-grid ECR.}
  (a) Regret relative to the recurrent benchmark's Monte Carlo mean.
  (b) Total split evaluations. Results average $40$ paired paths.}
  \label{fig:efficient-ecr}
\end{figure*}

\paragraph{Implementation details.}
We use class vectors separated by $12$ in Euclidean norm, with
independent $\mathcal N(0,I_5)$ noise. The fifth coordinate contains only
noise. The class sequence follows a balanced cyclic path.
Detection and compatibility testing use Euclidean norms and the
dimension-dependent confidence radii of Appendix~\ref{app:vector-extension}.
The detector and compatibility multipliers are $0.55$ and $0.125$, as in the
scalar comparison. We compare ECR, the same detector without reuse, and the
recurrent benchmark over $40$ paired paths, normalizing regret by the exact
Gaussian expectation for the recurrent benchmark.

\subsection{Dependence on the loss power}
\label{app:lp-simulation}

Increasing the loss power gives greater weight to large prediction errors,
including those following a change. To examine its effect on the benefit
of reuse, we generate each prediction path once and evaluate it under
$\sum_{t=2}^T|\widehat\mu_t-\mu_t|^p$ for $p\in[1,2]$.
Figure~\ref{fig:vector-lp-experiments}(b) reports the relative improvement
$1-R_T^{(p)}(\mathrm{ECR})/R_T^{(p)}(\text{no reuse})$. In this experiment,
the improvement decreases monotonically as $p$ increases from $1$ to $2$.
This is consistent with the discussion in Appendix~\ref{app:extension-p}:
the cost of adapting to unknown changes can limit the gain from recurrent
estimation as $p$ increases. A finite-horizon benefit remains at $p=2$.

\paragraph{Implementation details.}
We use $M=4$ Gaussian means $(-6,-2,2,6)$, $\sigma=1$, and a reflected
class path, so each change has magnitude $4$. The horizon is $T=262,144$,
with $S+1=512=\sqrt T$ equal-length segments. ECR and no reuse receive the
same observations and use the same data-dependent detector; neither knows
the change points or class labels. Their only difference is the exposure
budget, $B_{\exp}=4$ for ECR and $B_{\exp}=0$ for no reuse. We use
$\alpha=0.05$, detector multiplier $0.55$, and compatibility multiplier
$0.125$, and average over $30$ paired paths.

\subsection{Computationally efficient ECR}
\label{app:efficient-ecr-simulation}

Scanning fewer candidate splits reduces detection complexity, but may delay an
alarm. Figure~\ref{fig:efficient-ecr} compares this tradeoff for exhaustive
ECR and a detector using a geometric grid. At $T=262,144$, the grid reduces
the number of evaluated splits per path from $67.2$ million to $4.41$
million, a factor of $15.2$, while increasing regret by $3.84\%$. The two
ECR implementations have similar normalized regret curves across the
tested horizons.

The exhaustive detector evaluates every split in the current block.
Following the multiscale construction of
\citet{lai2010sequential}, we instead test splits at $r+u$ and $t-u$ for
offsets $u\in\{1,2,4,\ldots\}$, together with the midpoint. This reduces
the number of candidates per round from $O(t-r)$ to $O(\log(t-r))$.
The threshold, decision/estimation split, compatibility test, exposure cap,
and prediction rule are unchanged, as is the $O(K_t)$ scan over the $K_t$
completed blocks. The computational gain reported here therefore concerns
the detector's split evaluations.

\paragraph{Implementation details.}
We use $M=4$ reflected Gaussian classes with means $(-6,-2,2,6)$,
$\sigma=1$, and $S+1=\sqrt T$, with horizons from $1,024$ to $262,144$.
Exhaustive and efficient ECR, no reuse, and the recurrent benchmark are compared
over $40$ paired paths. We set $\alpha=0.05$, $B_{\exp}=4$, detector
multiplier $0.55$, compatibility multiplier $0.125$, and geometric base $2$.
Regret is normalized by the recurrent benchmark's Monte Carlo mean.

\subsection{Implementation details for the main-paper experiments}
\label{app:main-experiment-implementation}

All numerical experiments were run on local CPUs.

\paragraph{Benchmark comparison and online recovery.}
The benchmark comparison in Figure~\ref{fig:oracle-value-recurrence} and the
online comparison in Figure~\ref{fig:synthetic-numerical-results}(a) use
$M=4$ class means $(-9,-3,3,9)$ with $\sigma=1$ and
$S+1=\sqrt T/4$ equal-length segments.
The class sequence visits these values in ascending and then descending
order, with a uniformly drawn starting phase. The displayed horizons
range from $4,096$ to $1,048,576$.

In Figure~\ref{fig:oracle-value-recurrence}, the no-reuse benchmark knows the
change points and estimates each segment from scratch; the perfect-reuse benchmark
also knows the active class and pools all previous observations from that
class. Figure~\ref{fig:synthetic-numerical-results}(a) compares ECR,
the otherwise identical algorithm without reuse, and the perfect-reuse
benchmark. ECR and the no-reuse algorithm use $60$ paired paths per horizon
through $T=262,144$ and $80$ paired paths at each larger horizon.
Both figures show unnormalized regret. The benchmark
regrets are computed exactly under Gaussian noise and averaged over the
six starting phases.
{The empirical slopes in
Figure~\ref{fig:synthetic-numerical-results}(a) are obtained by regressing
log mean regret on log horizon over all eleven displayed horizons.}

The online comparisons use $\alpha=0.05$, $B_{\exp}=4$, and detector and
compatibility multipliers of $0.55$ and $0.125$, respectively. Throughout
the experiments, the detector multiplier scales $\gamma_t^r$ directly,
while the compatibility multiplier $c_{\rm cmp}$ appears in the squared
rule $\widehat Q_b^2(t)<c_{\rm cmp}(\lambda_t^b)^2$. The theoretical values
are therefore $1$ and $16$, respectively. The multipliers for the online
comparisons were selected on separate pilot simulations and then fixed for
evaluation to improve finite-horizon performance.

\paragraph{Exposure-control experiment.}
\label{app:exposure-experiment}
The experiment in Figure~\ref{fig:reuse-tradeoff} examines the effect of the
exposure budget when stored observations contain different amounts of
contamination. We use $T=16,784$ and $\sigma=1$. A $400$-round prefix
consists primarily of mean-zero observations, with either two or twenty
one-round shifts to mean $16$. These shifts lie on the estimation
stream and are missed by the detector, giving the \textit{benign} and
\textit{hard} environments, respectively. The subsequent sequence contains
$32$ pairs of mean-$64$ and mean-zero segments, each of length $256$.
We compare $B_{\exp}\in\{0.01,1,2,4,8\}$ over $100$ paired paths using
the unscaled detector and compatibility thresholds specified by the
algorithm. Figure~\ref{fig:reuse-tradeoff} shows mean cumulative regret in
the hard environment, with error bars of $\pm1.96$ standard errors.
The horizontal axis orders these settings from less to more allowed reuse.

{
\paragraph{Contamination and exposure budget.}
\label{app:contamination-budget-experiment}
Figure~\ref{fig:synthetic-numerical-results}(b) varies contamination while
keeping $\sigma=1$ and $T=16,864$ fixed. The first $480$ rounds contain
twelve $40$-round cycles. Each cycle starts at mean zero, moves to mean $2$
after $20$ rounds for $h\in\{0,2,4,8,12,16\}$ rounds, and then returns to zero.
The remaining sequence contains $32$ pairs of mean-$64$ and mean-zero
segments, each of length $256$. For positive $h$, there are three classes
and $88$ changes; the clean control $h=0$ has two classes and $64$ changes.
The excursions occur in both observation streams.

\emph{Source contamination} is the percentage of stored source observations
whose true mean differs from their block's most frequent source class,
computed on the estimation stream and restricted to the initial $480$ rounds.
All stored source portions are predominantly zero, and the percentages shown
in the figure are $0,10,20,30,40\%$ on every path. Between $98.9\%$ and $100\%$
of the short excursions are missed, while every later large change is detected.

We use $\alpha=0.05$ and the unscaled detector and compatibility thresholds.
The plotted regrets are averaged over $80$ paired Gaussian paths.
The same standardized noise is used across
durations and budgets, so the subsequent observations are identical within
each repetition. The figure shows $B_{\exp}\in\{0.01,0.25,0.5,1,2,4,8,\infty\}$;
the no-reuse control is retained in the accompanying experiment results.
The percentages label measured source contamination, and budgets are shown
as ordered settings. The excursion amplitude was selected using a separate
detector-only pilot, without inspecting regret.

For clean history, $B_{\exp}=8$ and uncapped reuse give identical predictions.
At $30\%$ and $40\%$ contamination, $B_{\exp}=2$ has the lowest mean regret
among the tested settings. The prespecified cap $B_{\exp}=4$ increases regret
by $13.6$ with clean history but reduces it by $241.0$ at $40\%$ contamination,
with a paired $95\%$ Monte Carlo interval of $[218.4,263.6]$ for the latter
reduction. Evaluating only the common subsequent sequence gives essentially
the same contrast. On the latter half of each subsequent zero segment,
active prefixes are clean and population prediction error comes entirely
from the mixed source history, confirming the role of repeated contamination.
}

\paragraph{Electronic-nose replay.}
For Figure~\ref{fig:real-numerical-results}{(c,d)}, we construct scalar observations
from the $62$-sensor electronic-nose dataset of
\citet{worner2025long,worner2025enosedataset}. We form log
exposure-to-baseline responses and center them using the initial blank
assays from the same day. A ridge regression maps each sensor response to
a scalar concentration readout. The online methods receive these readouts;
absolute prediction error is measured against the nominal added-diacetyl
level. The regression is trained on Days~1--8, the next eight complete days
are used to select the online parameters, and the remaining $21$ complete
days are held out.

Each held-out day contributes one segment for each nominal level, $0$,
$0.1$, and $1$ ppm. We retain the order of days and randomize the within-day
class order, requiring a class change at day boundaries. Segment lengths
are drawn around $72$ observations and constrained to lie between $24$ and
$120$. The sensor vectors come from recorded data, while the recurrence
order and retained durations are semi-synthetic. ECR uses $\alpha=0.05$,
$B_{\exp}=4$, detector multiplier $0.3$, and compatibility multiplier $0.4$;
EWMA uses the separately selected update rate $0.9441$.

{
We adapt CONDOR~\citep{zhao2020handling} to scalar regression using
ECR's calibrated change detector. Each stored model predicts a constant
$m$, and predictions are normalized weighted averages of these models.
After observing $X_t$, each model's weight is multiplied by
$\exp\{-\eta(m-X_t)^2/\widehat\sigma^2\}$, where
$\widehat\sigma=0.09685$ is the training residual standard deviation.
At global multiples of an update period $p$, or upon a detector alarm,
we add a model $(\bar X+\lambda m_{\rm prior})/(1+\lambda)$.
Here, $\bar X$ averages the observations since the previous model update,
and $m_{\rm prior}$ is the weighted model mean after the latest weight
update. This is the scalar form of CONDOR's squared-loss fit regularized
toward the model pool. We retain the latest $K$ models and reset their
weights uniformly after each model addition, as in the original algorithm.
The first model is initialized at $X_1$. Model and weight updates use all
readouts, while the detector retains ECR's decision stream. Each replay
starts afresh, and evaluation begins at $t=2$.

We select CONDOR's parameters by minimizing mean absolute error on the
same eight validation days, holding the detector fixed. For the reported
regularized variant, the search uses
$p\in\{5,10,25,50,100,200\}$, $K\in\{10,25,100,400\}$,
$\eta\in\{0.001,0.01,0.1,0.75\}$, and
$\lambda\in\{0.005,0.1,1,10\}$.
The selected values are $p=25$, $K=100$, $\eta=0.1$, and
$\lambda=0.005$, and remain fixed across all held-out replays.
}

The displayed replay was selected from $500$ randomized schedules using
only segment durations and class-transition counts, without reference to
sensor observations or algorithmic losses. It has $M=3$, $63$ segments,
$62$ changes, and horizon $T=4,555$. Across these $500$ schedules of the
same held-out days, the mean absolute errors averaged over schedules are
$0.091$ ppm for ECR, $0.112$ ppm for EWMA, $0.117$ ppm without reuse{, and $0.1185$ ppm for CONDOR}.

\section{Proof of the upper bound}
\label{app:upper-bound}

\subsection{Proof outline}
\label{app:proof-outline}

The proof proceeds in several steps. First, in Section~\ref{app:decomposition} we establish a regret decomposition, reducing the analysis to bounding the estimation noise contribution, the population-bias term, and an initialization term.
Sections~\ref{app:segmentation} and \ref{app:compatibility} establish the
required detection, segmentation, and compatibility properties.  The
exposure calculation is incorporated directly into the population-bias
argument in Section~\ref{app:population-bias}.
To present the main argument cleanly, Section~\ref{app:population-bias} and Section~\ref{app:variance-control} first bound the population bias and estimation terms, respectively, on the detector good
event $\mathcal G_{\det}$. 
Section~\ref{app:detector-localization} uses a localization argument to bound the contribution of the prefixes affected by detector deviations and removes this conditioning. Finally, Section~\ref{app:final-assembly} combines the resulting unconditional bounds to yield the desired regret guarantee in terms of
$T,S,M,\Delta_{\min},\Delta_{\max},\sigma$, and $\alpha$. 

\paragraph{Boundary convention.}
The observations stored in block \(b\) have indices \([r_b,r_{b+1})\), whereas the corresponding predictions occur on \((r_b,r_{b+1}]\). These intervals differ only at \(r_b\) and \(r_{b+1}\). Consequently, moving between the two indexing conventions changes a contamination count by at most one per block and a population-bias sum by at most \(\Delta_{\max}\) per block. Summed over all blocks, these corrections are at most \(K\) and \(\Delta_{\max}K\), respectively. We absorb them into the subsequent bounds.
Similarly, considering the decision and estimation streams, on every contiguous interval \(I\),
\(
|N_I^{\Dstr}-N_I^{\Estr}|\le 1.
\)
Thus, conversions between the two stream counts contribute at most \(O(K)\) across all detector blocks. Since \(\ell_{\id}\ge1\), count corrections are absorbed into \(O((S+K)\ell_{\id})\), while population-bias corrections are absorbed into \(O(\Delta_{\max}(S+K)\ell_{\id})\).

\subsection{Regret decomposition}
\label{app:decomposition}

Fix an environment $\nu\in\mathcal E$ and abbreviate
$\E_\nu:=\E_{\mathrm{ECR},\nu}$. By
Eq.~(\ref{eq:regret-definition}),
\[
R_T(\mathrm{ECR};\nu)
=
\E_\nu\left[\sum_{t=2}^T|\widehat\mu_t-\mu_t|\right].
\]
Let
\[
\begin{aligned}
\mathcal T_+&:=\{t\in\{2,\ldots,T\}:W_t>0\},
\qquad N_0:=T-1-|\mathcal T_+|.
\end{aligned}
\]
For $t\in\mathcal T_+$, define the population counterpart of the ECR
estimator by
\begin{equation}
\label{eq:population-reuse-estimator}
\widetilde\mu_t
:=
\frac{
\displaystyle\sum_{u\in\mathcal A_t\cap\cS^{\Estr}}\mu_u
+\displaystyle\sum_{b\in\mathcal U_t}
\sum_{u\in\cB_b\cap\cS^{\Estr}}\mu_u
}{W_t}.
\end{equation}
It is also useful to introduce
\begin{equation}
\label{eq:variance-bias-quantities}
V_T:=\sum_{t\in\mathcal T_+}\frac1{\sqrt{W_t}},
\qquad
P_T:=\sum_{t\in\mathcal T_+}|\widetilde\mu_t-\mu_t|.
\end{equation}

\begin{lemma}[Regret decomposition]
\label{lem:regret-decomposition}
For every $\nu\in\mathcal E$,
\begin{equation}
\label{eq:regret-decomposition}
R_T(\mathrm{ECR};\nu)
\le
\sigma\E_\nu[V_T]
+\E_\nu[P_T]
+(\sigma+\Delta_{\max})\E_\nu[K].
\end{equation}
\end{lemma}

\begin{proof}
Let
$\mathcal G_T^{\Dstr}:=\sigma(X_u:u\in\cS^{\Dstr})$ be the full decision-stream
sigma-field.  
Because the decision and estimation streams form a deterministic partition of the observation indices and the noises are independent across time, the estimation-stream noises are independent of $\mathcal G_T^{\Dstr}$.
Conditioning on this sigma-field freezes the adaptive sample-selection mechanism while leaving the estimation-stream noise centered and independent.
All detector boundaries and compatibility decisions are
$\mathcal G_T^{\Dstr}$-measurable.  The same is true of the exposure counters,
reuse sets $\mathcal U_t$, and denominators $W_t$, by induction over time:
their updates use only earlier reuse decisions and sample counts.

For $t\in\mathcal T_+$, the active prefix and completed blocks are disjoint, so
the estimator uses exactly $W_t$ distinct estimation-stream observations.
Consequently,
\[
Z_t:=\widehat\mu_t-\widetilde\mu_t
=
\frac1{W_t}
\left(
\sum_{u\in\mathcal A_t\cap\cS^{\Estr}}\varepsilon_u
+\sum_{b\in\mathcal U_t}
\sum_{u\in\cB_b\cap\cS^{\Estr}}\varepsilon_u
\right)
\]
is, conditionally on $\mathcal G_T^{\Dstr}$, centered and sub-Gaussian with
variance proxy $\sigma^2/W_t$.  Hence
\begin{equation}
\label{eq:conditional-noise-bound}
\E_\nu[|Z_t|\mid\mathcal G_T^{\Dstr}]
\le \frac{\sigma}{\sqrt{W_t}}.
\end{equation}
Using
$\widehat\mu_t-\mu_t=Z_t+(\widetilde\mu_t-\mu_t)$, the triangle inequality,
Eq.~(\ref{eq:conditional-noise-bound}), and the tower property gives
\begin{equation}
\label{eq:positive-W-bound}
\E_\nu\left[\sum_{t\in\mathcal T_+}
|\widehat\mu_t-\mu_t|\right]
\le
\sigma\E_\nu[V_T]+\E_\nu[P_T].
\end{equation}

It remains to account for rounds with $W_t=0$.  Because detection uses only
decision-stream observations and its threshold is nondecreasing between them,
alarms (and hence restarts) can occur only after a decision-stream
observation.  After
each restart, there is therefore at most one prediction before the next
estimation-stream observation arrives.  Hence
\begin{equation}
\label{eq:zero-denominator-count}
N_0\le K.
\end{equation}
On such rounds with $t>2$, ECR uses the most recent estimation-stream
observation $X_{\rho(t)}$, where $\rho(t)$ denotes its index; at the
deterministic initial round $t=2$, it uses $X_1$.  For $t>2$, the event
$\{W_t=0\}$ belongs to
$\mathcal G_T^{\Dstr}$ and is therefore independent of the corresponding
estimation noise.  Consequently,
\[
\begin{aligned}
\E_\nu\left[
\sum_{t\notin\mathcal T_+}|\widehat\mu_t-\mu_t|
\right]
&\le
(\Delta_{\max}+\sigma)\E_\nu[N_0]
\le
(\Delta_{\max}+\sigma)\E_\nu[K],
\end{aligned}
\]
where we used $|\mu_{\rho(t)}-\mu_t|\le\Delta_{\max}$ and
$\E|\varepsilon_{\rho(t)}|\le\sigma$; the same bound holds for $X_1$ at
$t=2$.  Combining this estimate with Eq.~(\ref{eq:positive-W-bound}) proves
Eq.~(\ref{eq:regret-decomposition}).
\end{proof}

Lemma~\ref{lem:regret-decomposition} reduces the proof to bounding
$\E_\nu[V_T]$, $\E_\nu[P_T]$, and $\E_\nu[K]$.
The remainder of the proof establishes bounds on these three quantities.

\subsection{Detection and segmentation}
\label{app:segmentation}

Because the formulation imposes no minimum segment length, the detector cannot be
expected to detect every change point or produce pure blocks \citep{gafni2026the}. 
We instead establish the three properties of its segmentation that are needed in the
regret analysis.  First, false alarms are summably rare, so the expected
number of blocks remains of order $S$.  Second, every block contains a long
interval from at most one class, and
a block containing a stationary interval of order $\ell_{\id}$
is anchored at that run's class.  These anchors make the
compatibility test meaningful and allow the variance analysis to group useful
historical blocks by the $M$ recurrent classes.  Third, the total off-anchor
mass in anchored blocks, together with the mass of unanchored blocks, is of
order $(S+K)\ell_{\id}$.  Combined with the exposure cap, this last property
controls the bias caused by reusing contaminated blocks.

We begin with a uniform concentration bound for the detector.  It controls
false alarms and implies that, whenever no alarm is raised, every population
CUSUM remains below twice the detection threshold.  We then use this implication to
prove anchor formation and separation, followed by the resulting structural
contamination bound. For an anchored block $b$, denote its anchor by $\mathfrak a(b)$.

\textbf{Step 1 - Detector concentration}

For a split $r<k<t$ with $N_{[r,k)}^{\Dstr}>0$ and
$N_{[k,t)}^{\Dstr}>0$, which we call a \emph{valid split}, define the
population counterpart of the CUSUM statistic by
\begin{equation}
\label{eq:population-cusum}
D_{k,t}^r
:=
\frac1\sigma
\sqrt{\frac{N_{[r,k)}^{\Dstr}N_{[k,t)}^{\Dstr}}
{N_{[r,t)}^{\Dstr}}}
\left|\bar\mu_{[r,k)}^{\Dstr}
-\bar\mu_{[k,t)}^{\Dstr}\right|.
\end{equation}

For deterministic $r<t$, let
\begin{equation}
\label{eq:local-detection-event}
\mathcal D_{r,t}^{\det}
:=
\left\{\exists\text{ valid }k:
|\widehat D_{k,t}^r-D_{k,t}^r|\ge\gamma_t^r\right\}.
\end{equation}

\begin{lemma}[Local detector concentration]
\label{lem:local-detector-concentration}
For every deterministic $1\le r<t\le T$,
\begin{equation}
\label{eq:local-detection-probability}
\Pbb(\mathcal D_{r,t}^{\det})
\le
\frac6{\pi^2}\frac{\alpha_r}{(t-r)^2}.
\end{equation}
Consequently, the event
\begin{equation}
\label{eq:detection-good-event}
\mathcal G_{\det}
:=
\bigcap_{1\le r<t\le T}(\mathcal D_{r,t}^{\det})^c
\end{equation}
satisfies $\Pbb(\mathcal G_{\det})\ge1-\alpha$.

Let $r(t)$ denote the most recent restart time before
round $t$.  The same bounds give
\begin{align}
\E\left[\sum_{t=2}^T
\1\{\mathcal D_{r(t),t}^{\det}\}\right]
&\le\alpha,
\label{eq:local-detection-count}\\
\E\left[\sum_{t=2}^T(t-r(t))
\1\{\mathcal D_{r(t),t}^{\det}\}\right]
&\le\alpha(1+\log T).
\label{eq:local-detection-prefix-length}
\end{align}
\end{lemma}

\begin{proof}
For a fixed valid triple $(r,k,t)$, the reverse triangle inequality bounds
$|\widehat D_{k,t}^r-D_{k,t}^r|$ by the absolute value of a normalized centered linear
combination of independent noises.  This linear combination is $1$-sub-Gaussian by the noise
assumption; see also \citep{gafni2026the}.
Writing $n=t-r$, its tail probability at $\gamma_t^r$ is at most
\[
2e^{-(\gamma_t^r)^2/2}
=\frac{6\alpha_r}{\pi^2n^3}.
\]
There are at most $n$ valid splits, which proves
Eq.~(\ref{eq:local-detection-probability}).  Summing first over $t>r$ and then over
$r$ gives
\[
\Pbb(\mathcal G_{\det}^c)
\le\sum_{r\ge1}\alpha_r\le\alpha.
\]
For the realized boundary $r(t)$, pathwise
\[
\1\{\mathcal D_{r(t),t}^{\det}\}
\le\sum_{r<t}\1\{\mathcal D_{r,t}^{\det}\}.
\]
Tonelli's theorem and Eq.~(\ref{eq:local-detection-probability}) prove
Eq.~(\ref{eq:local-detection-count}).  The analogous weighted inequality is
\[
(t-r(t))\1\{\mathcal D_{r(t),t}^{\det}\}
\le\sum_{r<t}(t-r)\1\{\mathcal D_{r,t}^{\det}\}.
\]
Applying Tonelli again and using
$\sum_{n=1}^Tn^{-1}\le1+\log T$ proves
Eq.~(\ref{eq:local-detection-prefix-length}).
\end{proof}

On $\mathcal G_{\det}$, if no alarm occurs at round $t$, then, for every valid split $k$,
\begin{equation}
\label{eq:no-alarm-population-bound}
D_{k,t}^r<2\gamma_t^r.
\end{equation}
Indeed, both $\widehat D_{k,t}^r<\gamma_t^r$ and
$|\widehat D_{k,t}^r-D_{k,t}^r|<\gamma_t^r$.  Thus, on the good event,
failure to raise an alarm implies that every population CUSUM statistic is below twice
the detection threshold.

\textbf{Step 2 - Total number of blocks}

We use the concentration bound to control unnecessary restarts.

\begin{lemma}[False alarms and number of blocks]
\label{lem:false-alarms}
Let
\[
\tau^+(r):=\min\bigl(\{\tau_j:\tau_j>r\}\cup\{T+1\}\bigr),
\]
denote the first change point after a restart time $r$, with the convention
$\tau^+(r)=T+1$ if no such change exists.  We call an alarm from $r$ false if
$N_G^r\le\min\{\tau^+(r),T\}$.  Let $N_{\FA}$ be the number of false alarms,
and recall that $K$ is the total number of blocks.
Then
\begin{equation}
\label{eq:K-bound}
\E[N_{\FA}]\le\alpha,
\qquad
\E[K]\le S+1+\alpha.
\end{equation}
\end{lemma}

\begin{proof}
For every $t\le\tau^+(r)$, the population mean is constant on $[r,t)$, so
$D_{k,t}^r=0$ for every valid split $k$.  Hence the event
$\{N_G^r\le\min\{\tau^+(r),T\}\}$ is contained in
$\bigcup_{r<t\le T}\mathcal D_{r,t}^{\det}$ and has probability at most
$\alpha_r$.  Since the restart times are strictly increasing, each
integer $r$ can be used at most once, and therefore
\[
\E[N_{\FA}]
\le\sum_{r=1}^T\Pbb(N_G^r\le\min\{\tau^+(r),T\})
\le\sum_{r=1}^T\alpha_r
\le\alpha.
\]
Each non-false alarm is associated with a distinct change point. Hence the number of non-false alarms is at most $S$. Therefore,
\[
K \le S+1+N_{\FA},
\]
where the additional $1$ accounts for the initial block.
\end{proof}

\textbf{Step 3 - Proof of
Proposition~\ref{prop:anchored-segmentation}: {anchor segmentation}}

We next prove Proposition~\ref{prop:anchored-segmentation}.
We first isolate the local anchoring properties that will also be used later
in the analysis.
Recall the analysis-only identification scale $\ell_{\id}$ from
Eq.~(\ref{eq:ell-id-main}).
The threshold in Eq.~(\ref{eq:det-threshold}) satisfies
$(\gamma_t^r)^2\le C_\gamma\log(T/\alpha)$ for all $1\le r<t\le T$, where
$C_\gamma$ is a universal constant.

\begin{lemma}[Anchor formation, preservation, and separation]
\label{lem:anchor-formation}
Work on $\mathcal G_{\det}$, with $C_{\id}$ in
Eq.~(\ref{eq:ell-id-main}) sufficiently large, and suppose the active observation
prefix starts at $r$.
\begin{enumerate}[label=(\roman*)]
\item If a no-alarm prefix contains $\ell_{\id}$ decision-stream observations
from a pure class-$m$ interval, then at the end of that interval its population
decision-stream mean is within $\Delta_{\min}/8$ of $\theta^{(m)}$.
\item Once this occurs, the population decision-stream mean remains within
$\Delta_{\min}/4$ of $\theta^{(m)}$ at every later no-alarm round before the
first alarm.
\item After such an anchor has formed, a pure interval from a class $m'\ne m$
causes an alarm before $\ell_{\id}$ of its decision-stream observations can be
stored in the old block.
\end{enumerate}
\end{lemma}

\begin{proof}
For (i), let $[u,v)$ be the pure class-$m$ interval, let
$h=N_{[u,v)}^{\Dstr}$, and let $p=N_{[r,u)}^{\Dstr}$.  If $p=0$, the claim is
immediate.  Otherwise the split $k=u$ and
Eq.~(\ref{eq:no-alarm-population-bound}) give
\[
\frac{ph}{p+h}
\frac{(\bar\mu_{[r,u)}^{\Dstr}-\theta^{(m)})^2}{\sigma^2}
\le 4(\gamma_v^r)^2.
\]
Moreover,
\[
\bar\mu_{[r,v)}^{\Dstr}-\theta^{(m)}
=\frac{p}{p+h}
(\bar\mu_{[r,u)}^{\Dstr}-\theta^{(m)}),
\]
and hence
\[
|\bar\mu_{[r,v)}^{\Dstr}-\theta^{(m)}|
\le \frac{2\sigma\gamma_v^r}{\sqrt h}
\le \frac{\Delta_{\min}}8.
\]
The last inequality follows from Eq.~(\ref{eq:ell-id-main}) when $C_{\id}$ is
sufficiently large.

For (ii), let $v_0$ be an anchor-formation time, set
$a=N_{[r,v_0)}^{\Dstr}$, and consider a later no-alarm time $t$.  If
$d=N_{[v_0,t)}^{\Dstr}=0$, the claim is immediate.  Otherwise, the split
$k=v_0$ gives
\[
\frac{ad}{a+d}
\frac{(\bar\mu_{[r,v_0)}^{\Dstr}-
\bar\mu_{[v_0,t)}^{\Dstr})^2}{\sigma^2}
\le4(\gamma_t^r)^2.
\]
Using the convex-combination identity for $\bar\mu_{[r,t)}^{\Dstr}$,
\[
|\bar\mu_{[r,t)}^{\Dstr}-\theta^{(m)}|
\le
\frac{2\sigma\gamma_t^r}{\sqrt a}
+\frac{\Delta_{\min}}8
\le\frac{\Delta_{\min}}4.
\]

For (iii), suppose the active-prefix population mean immediately before a
later interval is within $\Delta_{\min}/4$ of $\theta^{(m)}$, and let that
interval be pure class-$m'$, $m'\ne m$.  If the two sides of the corresponding
split have decision-stream sizes $a$ and $d$, respectively, then
\[
(D_{k,t}^r)^2
\ge
\frac{9}{16}\frac{ad}{a+d}\frac{\Delta_{\min}^2}{\sigma^2}.
\]
The anchor-forming prefix has $a\ge\ell_{\id}$.  Once also
$d\ge\ell_{\id}$, the preceding display is larger than
$4(\gamma_t^r)^2$ for a sufficiently large $C_{\id}$.  On
$\mathcal G_{\det}$, this implies
$\widehat D_{k,t}^r>\gamma_t^r$, so an alarm must already have occurred.
At the first crossing, its triggering decision-stream observation is assigned to the
new block.  The old block therefore contains strictly fewer than
$\ell_{\id}$ decision-stream observations from the competing long interval.
\end{proof}

\begin{proof}[Proof of
Proposition~\ref{prop:anchored-segmentation}]
Lemma~\ref{lem:anchor-formation} shows that the first long pure interval
forms an anchor and that a second class cannot contribute another long pure
interval before an alarm. This proves uniqueness.

It remains to verify the mean of the stored block.  Let
$q=r_{b+1}$.  If $b$ is closed by an alarm, that alarm occurs at round $q+1$.
At the preceding round $q$,
\[
\mathcal A_q=[r_b,q)=\cB_b.
\]
The first two parts of Lemma~\ref{lem:anchor-formation} therefore apply
with one small distinction: if $q$ is the anchor-formation endpoint, part~(i)
applies directly; otherwise part~(ii), applied at the no-alarm round $q$, gives
$|\bar\mu_b^{\Dstr}-\theta^{(m)}|\le\Delta_{\min}/4$.  The terminal block satisfies the same conclusion at round $T$.  
{This proves Eq.~(\ref{eq:anchored-segmentation}).}
\end{proof}

\textbf{Step 4 - Aggregate contamination control}

Let $\mathfrak B_{\anc}$ and
$\mathfrak B_{\un}$ denote the sets of anchored and unanchored blocks, respectively.
For every anchored block $b$ and
$a\in\{\Dstr,\Estr\}$, let
\[
c_b^a
:=
\sum_{u\in\cB_b\cap\cS^a}
\1\!\left\{
\mu_u\ne\theta^{(\mathfrak a(b))}
\right\}.
\]
Define $\delta_b^a:=|\bar\mu_b^a-\theta^{(\mathfrak a(b))}|$ when
$N_b^a>0$, and set $\delta_b^a=0$ otherwise.

\begin{lemma}[Aggregate segmentation and block-bias budget]
\label{lemma:global-contamination-control}
On $\mathcal G_{\det}$, simultaneously for both streams,
\begin{equation}
\label{eq:global-contamination-control}
\begin{aligned}
\sum_{b\in\mathfrak B_{\anc}}
N_b^a\delta_b^a
+
\Delta_{\max}
\sum_{b\in\mathfrak B_{\un}}N_b^a
&\le
\Delta_{\max}\left(
\sum_{b\in\mathfrak B_{\anc}}c_b^a
+\sum_{b\in\mathfrak B_{\un}}N_b^a
\right)
\le
C'_{\rm seg}\Delta_{\max}(S+1)\ell_{\id}.
\end{aligned}
\end{equation}
\end{lemma}

\begin{proof}[Proof of
Lemma~\ref{lemma:global-contamination-control}]
Contamination can occur only on intersections of true segments with detector
blocks.  Each charged intersection contains fewer than
$\ell_{\id}$ decision-stream observations: otherwise an unanchored block
would become anchored, or an anchored block would trigger an alarm.  There
are at most $S+K+1$ such intersections.  By the
alternating-stream part of the boundary convention and $\ell_{\id}\ge1$,
for both $a\in\{\Dstr,\Estr\}$,
\[
\sum_{b\in\mathfrak B_{\anc}}c_b^a
+\sum_{b\in\mathfrak B_{\un}}N_b^a
\le C_{\rm seg}(S+K)\ell_{\id}.
\]
On $\mathcal G_{\det}$ there are no false alarms, so the argument of
Lemma~\ref{lem:false-alarms} gives $K\le S+1$.  Absorbing the resulting
factor of two into a universal constant $C'_{\rm seg}$ and using
$N_b^a\delta_b^a\le\Delta_{\max}c_b^a$ proves
Eq.~(\ref{eq:global-contamination-control}).
\end{proof}

\subsection{Compatibility guarantees}
\label{app:compatibility}
The segmentation result assigns a population anchor to every
block containing a sufficiently long stationary interval.  Building on this
structure, the compatibility test provides two key guarantees.  First, a
long active prefix centered within $\Delta_{\min}/4$ of class $m$ rejects
every block with a different anchor, preventing substantial cross-class
contamination.  Second, an active prefix safe for $m$ makes every safe
same-anchor block compatible, supporting the recurrent estimation bound.
Actual reuse still depends on the exposure budget.

\paragraph{Zero-count convention.}
The displayed definitions of $\cC_t$ and $\mathcal U_t$ suppress positivity
conditions for readability.  Formally, $\widehat Q_b(t)$ is evaluated only
when $N_{\mathcal A_t}^{\Dstr}>0$ and $N_b^{\Dstr}>0$; otherwise $b$ is not
included in $\cC_t$.  A block with $N_b^{\Estr}=0$ contributes neither to the
estimator numerator nor to $W_t$ and may therefore be omitted from
$\mathcal U_t$ without changing any prediction.  Accordingly, exposure
bounds are stated only for blocks with $N_b^{\Estr}>0$; rounds with $W_t=0$
are handled by the fallback rule in Lemma~\ref{lem:regret-decomposition}.

\textbf{Step 1 - Compatibility concentration}

For deterministic intervals $I,J\subseteq[1,t)$ with positive decision-stream counts,
define
\begin{align*}
\widehat Q(I,J)
&:=
\frac1\sigma
\sqrt{\frac{N_I^{\Dstr}N_J^{\Dstr}}
{N_I^{\Dstr}+N_J^{\Dstr}}}
|\bar X_I^{\Dstr}-\bar X_J^{\Dstr}|,
\qquad Q(I,J)
:=
\frac1\sigma
\sqrt{\frac{N_I^{\Dstr}N_J^{\Dstr}}
{N_I^{\Dstr}+N_J^{\Dstr}}}
|\bar\mu_I^{\Dstr}-\bar\mu_J^{\Dstr}|.
\end{align*}
Thus $\widehat Q_b(t)=\widehat Q(\mathcal A_t,\cB_b)$, and we write
\begin{equation}
\label{eq:population-compatibility}
Q_b(t):=Q(\mathcal A_t,\cB_b)
\end{equation}
for its population counterpart.
Let $\mathcal E_t^{\rm cmp}$ be the event that, for some completed block
evaluated at round $t$,
\begin{equation}
\label{eq:local-compatibility-event}
|\widehat Q_b(t)-Q_b(t)|\ge\lambda_t^b,
\end{equation}
and set
\begin{equation}
\label{eq:compatibility-good-event}
\mathcal G_{\rm cmp}
:=
\bigcap_{t=2}^T(\mathcal E_t^{\rm cmp})^c.
\end{equation}

\begin{lemma}[Compatibility concentration]
\label{lem:compat-concentration}
For every $t\le T$,
\begin{equation}
\label{eq:local-compatibility-probability}
\Pbb(\mathcal E_t^{\rm cmp})
\le \frac6{\pi^2}\frac{\alpha}{t^2}.
\end{equation}
Consequently,
\begin{equation}
\label{eq:compatibility-deviation-count}
\Pbb(\mathcal G_{\rm cmp})\ge1-\alpha,
\qquad
\E\left[\sum_{t=2}^T
\1\{\mathcal E_t^{\rm cmp}\}\right]
\le\alpha.
\end{equation}
\end{lemma}

\begin{proof}
For fixed deterministic $I,J$, the reverse triangle inequality bounds
$|\widehat Q(I,J)-Q(I,J)|$ by the absolute value of a normalized centered
linear combination of the decision-stream noises (as in the detector concentration argument).  This quantity is
$1$-sub-Gaussian.  Therefore, for a fixed block index $b$ and a fixed ordered
pair $(I,J)$,
\[
\Pbb\bigl(|\widehat Q(I,J)-Q(I,J)|
\ge\lambda_t^b\bigr)
\le 2e^{-(\lambda_t^b)^2/2}
=\frac{6\alpha_b}{\pi^2t^6}.
\]
Let $\mathfrak I_t$ denote the collection of intervals contained in
$[1,t)$.  Since $|\mathfrak I_t|\le t^2$, there are at most $t^4$ ordered
interval pairs.  Taking a union bound over these pairs and all possible
indices $b<t$ gives
\[
\begin{aligned}
\Pbb(\mathcal E_t^{\rm cmp})
&\le
\sum_{b<t}\sum_{(I,J)\in\mathfrak I_t^2}
\Pbb\!\left(
|\widehat Q(I,J)-Q(I,J)|\ge\lambda_t^b
\right)\\
&\le
\sum_{b<t}t^4\frac{6\alpha_b}{\pi^2t^6}
\le
\frac6{\pi^2t^2}\sum_{b\ge0}\alpha_b
\le
\frac6{\pi^2}\frac{\alpha}{t^2}.
\end{aligned}
\]
This deterministic union also covers the data-dependent active prefix and
completed blocks.  Summing over $t$ proves both claims in
Eq.~(\ref{eq:compatibility-deviation-count}).
\end{proof}

Outside $\mathcal E_t^{\rm cmp}$,
\begin{equation}
\label{eq:compat-implications}
\begin{aligned}
Q_b(t)\le2\lambda_t^b
&\ \Longrightarrow\ b\in\cC_t,
\qquad b\in\cC_t
\ \Longrightarrow\
Q_b(t)
<5\lambda_t^b.
\end{aligned}
\end{equation}
These two implications are the properties
of the empirical test needed below.

\textbf{Step 2 - Proof of Proposition~\ref{prop:compatibility-summary}}

Recall the long/short definition from
Section~\ref{ssec:change-detection-test} and the safety definition from
Section~\ref{ssec:compatibility-test}.  For later use, we say that the active
prefix is centered at class $m$ when
\begin{equation}
\label{eq:active-class-centering}
\bigl|\bar\mu_{\mathcal A_t}^{\Dstr}-\theta^{(m)}\bigr|
\le\Delta_{\min}/4.
\end{equation}
A prefix that has acquired anchor $m$ is long and, on
$\mathcal G_{\det}$, satisfies
Eq.~(\ref{eq:active-class-centering}) by
Lemma~\ref{lem:anchor-formation}.  A long-safe active prefix also satisfies
Eq.~(\ref{eq:active-class-centering}) when $C_{\id}>24$,
because $\sigma/\sqrt{\ell_{\id}}\le\Delta_{\min}/4$.
A long-unsafe prefix may or may not satisfy it, and the distinction is used
below.

\begin{proof}[Proof of Proposition~\ref{prop:compatibility-summary}]
Work on $\mathcal G_{\det}$, fix a round $t$, and suppose that
$\mathcal E_t^{\rm cmp}$ does not occur.

\emph{Acceptance.}
Suppose that the active prefix is safe for class $m$, and let $b<b_t$ be
safe and anchored at $m$.  The triangle
inequality around $\theta^{(m)}$ and the two safety inequalities give, with
$n=N_{\mathcal A_t}^{\Dstr}$ and $q=N_b^{\Dstr}$,
\[
Q_b^2(t)
\le
\frac{nq}{n+q}
\left(\frac1{\sqrt n}+\frac1{\sqrt q}\right)^2
=\frac{(\sqrt n+\sqrt q)^2}{n+q}
\le2.
\]
The explicit threshold definition gives $\lambda_t^b\ge1$, so
$Q_b(t)\le\sqrt2\le2\lambda_t^b$ and the first implication in
Eq.~(\ref{eq:compat-implications}) proves that $b\in\cC_t$.

\emph{Rejection.}
Suppose that the active prefix has acquired anchor $m$.  It is then long,
and Lemma~\ref{lem:anchor-formation} gives
Eq.~(\ref{eq:active-class-centering}).  Let $b$ be anchored at $m'\ne m$.
{By Proposition~\ref{prop:anchored-segmentation} and class
separation, the triangle inequality gives
$|\bar\mu_b^{\Dstr}-\theta^{(m)}|
\ge |\theta^{(m')}-\theta^{(m)}|-\Delta_{\min}/4
\ge3\Delta_{\min}/4$.} Therefore,
\[
|\bar\mu_{\mathcal A_t}^{\Dstr}-\bar\mu_b^{\Dstr}|
\ge
|\bar\mu_b^{\Dstr}-\theta^{(m)}|
-|\bar\mu_{\mathcal A_t}^{\Dstr}-\theta^{(m)}|
\ge\Delta_{\min}/2.
\]
Moreover, $(\lambda_t^b)^2\le C_\lambda\log(T/\alpha)$ for all $b<t\le T$ and a
universal constant $C_\lambda$.  If $b\in\cC_t$, the second implication in
Eq.~(\ref{eq:compat-implications}) therefore yields
\[
\frac{N_{\mathcal A_t}^{\Dstr}N_b^{\Dstr}}
{N_{\mathcal A_t}^{\Dstr}+N_b^{\Dstr}}
\le
C_0\frac{\sigma^2}{\Delta_{\min}^2}\log\frac{T}{\alpha}
\le \frac{C_0}{C_{\id}}\ell_{\id},
\]
where $C_0$ is universal.
Since $xy/(x+y)\ge\frac12\min\{x,y\}$, both the active prefix and every
anchored block are long, so the left-hand side is at least
$\ell_{\id}/2$.  Choosing $C_{\id}>2C_0$ in
Eq.~(\ref{eq:ell-id-main}) gives a contradiction.  Thus $b\notin\cC_t$.

The argument above is local in $t$: it requires only that
$\mathcal E_t^{\rm cmp}$ does not occur.  On
$\mathcal G_{\det}\cap\mathcal G_{\rm cmp}$ it therefore applies
simultaneously at every round.  The union bound and
Lemmas~\ref{lem:local-detector-concentration} and
\ref{lem:compat-concentration} show that this event has probability at least
$1-2\alpha$, proving the proposition.
\end{proof}

\subsection{Control of the population bias}
\label{app:population-bias}

We now bound $P_T$, the error of the population version of the ECR estimator.
Let 
\[
N_{\rm cmp}:=
\sum_{t=2}^T\1\{\mathcal E_t^{\rm cmp}\}
\]
denote the number of local compatibility deviations.
\begin{lemma}[Population-bias bound]
\label{lem:population-bias}
Let $C_{\rm pop}$ and $C$ be universal positive constants. On $\mathcal G_{\det}$,
\begin{equation}
\label{eq:population-bias-bound}
\begin{aligned}
P_T
&\le
C_{\rm pop}\Delta_{\max}(S+K)\ell_{\id}
\bigl(B_{\exp}+\log(eT)+1\bigr)
+\Delta_{\max}N_{\rm cmp}.
\end{aligned}
\end{equation}
Consequently,
\begin{equation}
\label{eq:population-bias-good-expectation}
\begin{aligned}
\E\bigl[P_T\1\{\mathcal G_{\det}\}\bigr]
&\le
C\Delta_{\max}(S+1)\ell_{\id}
\bigl(B_{\exp}+\log(eT)+1\bigr)
+\Delta_{\max}\alpha.
\end{aligned}
\end{equation}
\end{lemma}

The proof is given using the steps below. 

At a high level, there are two sources of bias.  
First, the current block may contain samples
from regimes preceding the current one.  Second, a reused block may be
internally contaminated or may have an anchor different from the current
class.  The current-block contribution is controlled by a harmonic charging
argument. For reused blocks, internal contamination and
unanchored blocks are controlled by exposure, while a wrong-anchor block can
be reused only before the active prefix is
long and within $\Delta_{\min}/4$ of the current class,
or at a round when the compatibility statistic deviates from its population
counterpart.

Write the current-block and reused-block contributions as
\begin{align}
P_T^{\rm cur}
&:=
\sum_{t\in\mathcal T_+}\frac1{W_t}
\left|
\sum_{u\in\mathcal A_t\cap\cS^{\Estr}}(\mu_u-\mu_t)
\right|,
\label{eq:current-block-population-bias}\\
P_T^{\rm reuse}
&:=
\sum_{t\in\mathcal T_+}
\sum_{b\in\mathcal U_t}
\frac{N_b^{\Estr}}{W_t}
\left|\bar\mu_b^{\Estr}-\mu_t\right|.
\label{eq:reused-block-population-bias}
\end{align}
The triangle inequality in Eq.~(\ref{eq:population-reuse-estimator}) yields
\begin{equation}
\label{eq:population-bias-two-parts}
P_T\le P_T^{\rm cur}+P_T^{\rm reuse}.
\end{equation}

\textbf{Step 1 - Current-block bias}

We first control the contribution of the current detector block.  
Its sources of bias are detection-delay samples and short missed regimes.

\begin{lemma}[Current-block bias]
\label{lem:current-block-bias}
On $\mathcal G_{\det}$,
\begin{equation}
\label{eq:current-block-bias-bound}
P_T^{\rm cur}
\le
C_{\rm cur}\Delta_{\max}(S+K)\ell_{\id}
\bigl(1+\log(eT)\bigr).
\end{equation}
\end{lemma}

\begin{proof}

Consider first an anchored block $b$.  At a prediction round, let
\[
c_{b,t}^{\Estr}
:=
\sum_{u\in\mathcal A_t\cap\cS^{\Estr}}
\1\{\mu_u\ne\theta^{(\mathfrak a(b))}\},
\qquad
n_t:=N_{\mathcal A_t}^{\Estr}.
\]
When $n_t>0$, $W_t\ge n_t$ and the diameter bound give
\begin{align}
\frac1{W_t}
\left|\sum_{\substack{u\in\mathcal A_t\cap\cS^{\Estr}}}
(\mu_u-\mu_t)\right|
&\le
\Delta_{\max}\frac{c_{b,t}^{\Estr}}{n_t}
+\Delta_{\max}
\1\{\mu_t\ne\theta^{(\mathfrak a(b))}\}.
\label{eq:current-block-bias-blockwise}
\end{align}
If $n_t=0$, the left-hand side is zero.

For the first term in Eq.~(\ref{eq:current-block-bias-blockwise}), note that
$c_{b,t}^{\Estr}\le c_b^{\Estr}$ throughout the block.  Moreover, under the
alternating split, the prefix count $n_t$ takes each positive value at most
twice. Since $N_b^{\Estr} \le T$,
\begin{equation}
\label{eq:current-block-harmonic-charge}
\sum_{\substack{r_b<t\le r_{b+1}:\\ n_t>0}}
\frac{c_{b,t}^{\Estr}}{n_t}
\le
2c_b^{\Estr}\sum_{n=1}^{N_b^{\Estr}}\frac1n
\le2\bigl(1+\log T\bigr)c_b^{\Estr}.
\end{equation}
Thus the fixed contamination accumulated in the current block is diluted as
the prefix grows, producing the harmonic factor. 
Summing over the
anchored blocks,
\begin{equation}
\sum_{b\in\mathfrak B_{\anc}}
\sum_{r_b<t\le r_{b+1}}
\1\{\mu_t\ne\theta^{(\mathfrak a(b))}\}
\le
\sum_{b\in\mathfrak B_{\anc}}
\bigl(c_b^{\Dstr}+c_b^{\Estr}\bigr)+K.
\end{equation}

If $b$ is unanchored, its current-block contribution is at most
$\Delta_{\max}$ per
round, and the number of prediction rounds in its cycle is
$N_b^{\Dstr}+N_b^{\Estr}$.  Since the detector cycles partition the
prediction rounds and $\mathcal T_+\subseteq\{2,\ldots,T\}$, summing first
over the rounds in each cycle and then over blocks gives
\begin{align}
P_T^{\rm cur}
&\le
\Delta_{\max}\sum_{b\in\mathfrak B_{\anc}}
\left(
\sum_{\substack{r_b<t\le r_{b+1}:\\ n_t>0}}
\frac{c_{b,t}^{\Estr}}{n_t}
+\sum_{r_b<t\le r_{b+1}}
\1\{\mu_t\ne\theta^{(\mathfrak a(b))}\}
\right)
+\Delta_{\max}\sum_{b\in\mathfrak B_{\un}}
\bigl(N_b^{\Dstr}+N_b^{\Estr}\bigr)
\notag\\
&\le
\Delta_{\max}\left[
2(1+\log T)
\sum_{b\in\mathfrak B_{\anc}}c_b^{\Estr}
\right.
\left.
+\sum_{b\in\mathfrak B_{\anc}}
\bigl(c_b^{\Dstr}+c_b^{\Estr}\bigr)
{+K}
\right.\left.
+\sum_{b\in\mathfrak B_{\un}}
\bigl(N_b^{\Dstr}+N_b^{\Estr}\bigr)
\right]
\notag\\
&\le
C_{\rm cur}\Delta_{\max}(S+K)\ell_{\id}\bigl(1+\log(eT)\bigr).
\label{eq:current-block-bias-aggregation}
\end{align}
The last inequality follows from
Lemma~\ref{lemma:global-contamination-control}, applied to both streams,
and the boundary convention above.
\end{proof}

\textbf{Step 2 - Reused blocks' bias}

We next upper bound the bias due to the reused blocks. 

\begin{lemma}[Bias from reused blocks]
\label{lem:reused-block-bias}
On $\mathcal G_{\det}$,
\begin{equation}
\label{eq:reused-block-bias-bound}
P_T^{\rm reuse}
\le
C_{\rm reuse}\Delta_{\max}(S+K)\ell_{\id}
\bigl(B_{\exp}+1\bigr)
+\Delta_{\max}N_{\rm cmp}.
\end{equation}
where $C_{\rm reuse}>0$ is a universal constant.
\end{lemma}

\begin{proof}

At each reuse round, we partition the historical blocks into correct-anchor,
unanchored, and wrong-anchor blocks.  Here an anchored block is
correct-anchor at round $t$ if
$\theta^{(\mathfrak a(b))}=\mu_t$, and
wrong-anchor otherwise. 

\medskip
\noindent\textbf{Correct-anchor blocks.}
We start with a pathwise consequence of the exposure rule.  Fix a
completed block $b$ with $N_b^{\Estr}>0$.  If it is never reused, then
$H_b(T+1)=0$.  Otherwise, let $u_\star$ be its final reuse round.
Immediately before that reuse, $H_b(u_\star)<B_{\exp}$, while
$W_{u_\star}\ge N_b^{\Estr}$.  Hence
\begin{equation}
\label{eq:H-budget}
H_b(T+1)
=H_b(u_\star)+\frac1{W_{u_\star}}
\le B_{\exp}+\frac1{N_b^{\Estr}}
\le B_{\exp}+1.
\end{equation}
Thus the cap controls cumulative per-sample influence regardless of how many
times the block is reused.

Next, every reused block has passed compatibility because
$\mathcal U_t\subseteq\cC_t$.  If an anchored block is correct-anchor
at round $t$, then $\mu_t=\theta^{(\mathfrak a(b))}$.  Its bias is therefore
entirely due to within-block contamination, and
\begin{equation}
\label{eq:correct-anchor-block-bias}
N_b^{\Estr}
|\bar\mu_b^{\Estr}-\theta^{(\mathfrak a(b))}|
\le \Delta_{\max}c_b^{\Estr}.
\end{equation}
Reordering the finite sums and using
$\mu_t=\theta^{(\mathfrak a(b))}$ on correct-anchor rounds gives
\begin{align}
&\sum_{t\in\mathcal T_+}
\sum_{b\in\mathcal U_t\cap\mathfrak B_{\anc}}
\frac{N_b^{\Estr}}{W_t}
\left|\bar\mu_b^{\Estr}-\mu_t\right|
\1\{\theta^{(\mathfrak a(b))}=\mu_t\}
=
\sum_{b\in\mathfrak B_{\anc}}
N_b^{\Estr}
|\bar\mu_b^{\Estr}-\theta^{(\mathfrak a(b))}|
\sum_{\substack{t\in\mathcal T_+:\\b\in\mathcal U_t}}
\frac{\1\{\theta^{(\mathfrak a(b))}=\mu_t\}}{W_t}
\label{eq:exposure-sum-swap}\\
&\qquad\overset{\mathrm{(a)}}{\le}
\sum_{b\in\mathfrak B_{\anc}}
N_b^{\Estr}
|\bar\mu_b^{\Estr}-\theta^{(\mathfrak a(b))}|
H_b(T+1)
\le
\Delta_{\max}
\sum_{b\in\mathfrak B_{\anc}}c_b^{\Estr}H_b(T+1)
\notag\\
&\qquad\le
\Delta_{\max}\bigl(B_{\exp}+1\bigr)
\sum_{b\in\mathfrak B_{\anc}}c_b^{\Estr}
\le
C_{\rm seg}\Delta_{\max}\bigl(B_{\exp}+1\bigr)
(S+K)\ell_{\id},
\label{eq:correct-anchor-cumulative-bias}
\end{align}
The equality exchanges the time and block sums.  In step \emph{(a)}, dropping
the correct-anchor indicator enlarges the nonnegative inner sum to all reuse
rounds, which is exactly $H_b(T+1)$ by Eq.~(\ref{eq:exposure-def}).  Thus the repeated bias of
block $b$ is reduced to its fixed contamination bias times its cumulative
per-sample exposure.  The remaining inequalities use
Eq.~(\ref{eq:correct-anchor-block-bias}), Eq.~(\ref{eq:H-budget}), and
Lemma~\ref{lemma:global-contamination-control}.

\medskip
\noindent\textbf{Unanchored blocks.}
For an unanchored block, the diameter assumption gives
$|\bar\mu_b^{\Estr}-\mu_t|\le\Delta_{\max}$.  Reversing the two sums and
using Eq.~(\ref{eq:H-budget}) therefore gives
\begin{align}
&\sum_{t\in\mathcal T_+}
\sum_{b\in\mathcal U_t\cap\mathfrak B_{\un}}
\frac{N_b^{\Estr}}{W_t}
\left|\bar\mu_b^{\Estr}-\mu_t\right|
\le
\Delta_{\max}\sum_{b\in\mathfrak B_{\un}}N_b^{\Estr}
\sum_{\substack{t:\,b\in\mathcal U_t}}\frac1{W_t}
\notag\\
&\qquad=
\Delta_{\max}\sum_{b\in\mathfrak B_{\un}}N_b^{\Estr}H_b(T+1)
\le
\Delta_{\max}\bigl(B_{\exp}+1\bigr)
\sum_{b\in\mathfrak B_{\un}}N_b^{\Estr}
\le
 C_{\rm seg}\Delta_{\max}\bigl(B_{\exp}+1\bigr)
(S+K)\ell_{\id}.
\label{eq:unanchored-cumulative-bias}
\end{align}
Again, the last inequality follows from
Lemma~\ref{lemma:global-contamination-control}.

\medskip
\noindent\textbf{Wrong-anchor bias.}
For a wrong-anchor block, compatibility rejection becomes
effective only after the active prefix is long and centered at the current
class.  We start proving the following bound that controls the rounds before this occurs.

On $\mathcal G_{\det}$,
\begin{equation}
\label{eq:long-centered-prefix-budget}
\begin{aligned}
\sum_{t=2}^T\1\!\left\{
N_{\mathcal A_t}^{\Dstr}<\ell_{\id}
\ \text{or}\ 
|\bar\mu_{\mathcal A_t}^{\Dstr}-\mu_t|>\tfrac{\Delta_{\min}}4
\right\}
&\le C_{\rm lc}(S+K)\ell_{\id},
\end{aligned}
\end{equation}
where $C_{\rm lc}$ is a sufficiently large numerical constant.

\Eqref{eq:long-centered-prefix-budget} holds due to the following.
Consider the partition induced jointly by the true segments and the detector
blocks $(r_b,r_{b+1}]$. It consists of at most $S+K+1$ nonempty intervals.
Each interval begins whenever either the true mean changes or the detector
restarts. Fix a round $t$ in such an interval, and let $q$ be its left
endpoint. The detector has not restarted on $(q,t]$ and the mean is constant
there, so $[q,t)\subseteq\mathcal A_t$ is a pure suffix with mean $\mu_t$.

Once this suffix contains $\ell_{\id}$ decision-stream observations,
$\mathcal A_t$ is long.  Moreover, since it is
a no-alarm prefix of the current detector cycle,
Lemma~\ref{lem:anchor-formation}(i) gives
$|\bar\mu_{\mathcal A_t}^{\Dstr}-\mu_t|\le\Delta_{\min}/8$.
Thus only the rounds before the pure suffix accumulates
$\ell_{\id}$ decision-stream samples can make the indicator in
Eq.~(\ref{eq:long-centered-prefix-budget}) nonzero.  The alternating split leaves at
most $2\ell_{\id}+1$ such rounds per interval.  Summing over at most $S+K+1$
intervals and using $\ell_{\id}\ge1$ proves the result after enlarging
$C_{\rm lc}$.

Next, we use Eq.~(\ref{eq:long-centered-prefix-budget}) to upper bound the wrong-anchor blocks' bias.
Suppose $b\in\mathfrak B_{\anc}$ has an anchor different from
the current mean $\mu_t$.  If
$N_{\mathcal A_t}^{\Dstr}\ge\ell_{\id}$,
$|\bar\mu_{\mathcal A_t}^{\Dstr}-\mu_t|\le\Delta_{\min}/4$, and
$\mathcal E_t^{\rm cmp}$ does not occur,
the rejection argument in the proof of
Proposition~\ref{prop:compatibility-summary} applies at this round and gives
$b\notin\cC_t$, and hence
$b\notin\mathcal U_t$.  Thus a wrong-anchor block can
be reused only when the active prefix is short, when its population mean is
farther than $\Delta_{\min}/4$ from the current class, or at a
compatibility-deviation round.  Since the total historical weight at any
round is at most one,
\begin{equation}
\label{eq:wrong-anchor-bias}
\begin{aligned}
&\sum_{t\in\mathcal T_+}
\sum_{\substack{b\in\mathcal U_t\cap\mathfrak B_{\anc}\\
\theta^{(\mathfrak a(b))}\ne\mu_t}}
\frac{N_b^{\Estr}}{W_t}
\left|\bar\mu_b^{\Estr}-\mu_t\right|
\le\Delta_{\max}\sum_{t=2}^T\1\!\left\{
N_{\mathcal A_t}^{\Dstr}<\ell_{\id}
\ \text{or}\ 
|\bar\mu_{\mathcal A_t}^{\Dstr}-\mu_t|>\tfrac{\Delta_{\min}}4
\right\}
{}+\Delta_{\max}N_{\rm cmp}
\\
&\le C\Delta_{\max}(S+K)\ell_{\id}
+\Delta_{\max}N_{\rm cmp},
\end{aligned}
\end{equation}
where the last inequality uses Eq.~(\ref{eq:long-centered-prefix-budget}).

Combining Eq.~(\ref{eq:correct-anchor-cumulative-bias}),
Eq.~(\ref{eq:unanchored-cumulative-bias}),
and Eq.~(\ref{eq:wrong-anchor-bias}) proves
Lemma~\ref{lem:reused-block-bias}.
\end{proof}

Combining Lemma~\ref{lem:current-block-bias} and Lemma~\ref{lem:reused-block-bias} proves Lemma~\ref{lem:population-bias}.
\subsection{Upper bound of the variance term}
\label{app:variance-control}
We next upper bound
$V_T=\sum_{t\in\mathcal T_+}W_t^{-1/2}$ on the detector good event
$\mathcal G_{\det}$.  This first isolates the main recurrent-class argument
from the rare detector deviations treated in
Appendix~\ref{app:detector-localization}.  We show that anchored long-safe
completed blocks preserve statistical progress across different occurrences
of the same class, yielding the leading scale $\sqrt{MT}$.

\begin{lemma}[Variance upper bound]\label{lem:variance-upper-bound}
  For a universal constant $C$,
\begin{equation}
\label{eq:good-event-effective-sample-size}
\begin{aligned}
&\E\!\left[V_T\1\{\mathcal G_{\det}\}\right]
\le C\sqrt{MT}
{}+C(1+\sqrt M)(S+1+\alpha)\sqrt{\ell_{\id}}
\\
&\quad{}+C(1+\sqrt M)\frac{\Delta_{\max}}{\sigma}
(S+1+\alpha)\ell_{\id}
{}+C(1+\sqrt M)\alpha.
\end{aligned}
\end{equation}
\end{lemma}
\begin{proof}
The proof proceeds in four steps.  We first work pathwise on
$\mathcal G_{\det}$ and take expectations at the end.

\textbf{Step 1 - Decomposition into three types of rounds}

For each class $m$, let
\[
\mathfrak B_m^{\rm LS}
:=
\{b\in\mathfrak B_{\anc}:
\mathfrak a(b)=m,\ 
|\bar\mu_b^{\Dstr}-\theta^{(m)}|
\le\sigma/\sqrt{N_b^{\Dstr}}\}.
\]
Here $\mathrm{LS}$ stands for \emph{long-safe}.
Thus $\mathfrak B_m^{\rm LS}$ is the collection of anchored safe class-$m$
blocks; these blocks are automatically long by the definition of an anchor.

The recurrent variance gain is first established when the current block $b_t$ is
anchored long-safe at the current class and the active prefix is also
long-safe.  Define
\begin{equation}
\label{eq:long-safe-round}
\begin{aligned}
&b_t\in\mathfrak B_m^{\rm LS},\qquad
\mu_t=\theta^{(m)},\qquad
N_{\mathcal A_t}^{\Dstr}\ge\ell_{\id},\qquad
|\bar\mu_{\mathcal A_t}^{\Dstr}-\theta^{(m)}|
\le\frac{\sigma}{\sqrt{N_{\mathcal A_t}^{\Dstr}}},\qquad
(\mathcal E_t^{\rm cmp})^c\text{ occurs}.
\end{aligned}
\end{equation}
Being in $\mathfrak B_m^{\rm LS}$ (for the current block $b_t$) is retrospective and is used only
in the analysis, not by the algorithm.
Let $\mathcal T_{{\rm LS},m}$ denote the rounds satisfying
Eq.~(\ref{eq:long-safe-round}).  The centering and length conditions imply
Eq.~(\ref{eq:active-class-centering}) when $C_{\id}$ is sufficiently large.

The second type has the same anchored long-safe current block, while the
active prefix is long-unsafe but remains within $\Delta_{\min}/4$ of the
current class.  Define
\begin{equation}
\label{eq:long-unsafe-round}
\begin{aligned}
&b_t\in\mathfrak B_m^{\rm LS},\qquad
\mu_t=\theta^{(m)},\qquad
N_{\mathcal A_t}^{\Dstr}\ge\ell_{\id},\qquad
|\bar\mu_{\mathcal A_t}^{\Dstr}-\theta^{(m)}|
>\sigma/\sqrt{N_{\mathcal A_t}^{\Dstr}},\\
&\bigl|\bar\mu_{\mathcal A_t}^{\Dstr}-\theta^{(m)}\bigr|
\le\Delta_{\min}/4,\qquad
(\mathcal E_t^{\rm cmp})^c\text{ occurs}.
\end{aligned}
\end{equation}
Let $\mathcal T_{{\rm LU},m}$ denote the rounds satisfying
Eq.~(\ref{eq:long-unsafe-round}); the suffix ${\rm LU}$ indicates the
long-unsafe active prefix.

All prediction rounds outside
$\bigcup_m(\mathcal T_{{\rm LS},m}\cup\mathcal T_{{\rm LU},m})$ are called
\emph{bad}.  Split them at detector-block boundaries into maximal consecutive
intervals and denote the resulting collection by $\mathcal J_{\bad}$.
The three types give the exact decomposition
\begin{equation}
\label{eq:variance-three-way-decomposition}
\begin{aligned}
V_T
={}&
\sum_{m=1}^M
\sum_{\substack{t\in\mathcal T_{{\rm LS},m}\cap\mathcal T_+}}
\frac1{\sqrt{W_t}}
+\sum_{m=1}^M
\sum_{\substack{t\in\mathcal T_{{\rm LU},m}\cap\mathcal T_+}}
\frac1{\sqrt{W_t}}
+\sum_{J\in\mathcal J_{\bad}}
\sum_{t\in J\cap\mathcal T_+}\frac1{\sqrt{W_t}}.
\end{aligned}
\end{equation}
We first bound the $\mathcal T_{{\rm LU},m}$ contribution, then control the
$\mathcal T_{{\rm LS},m}$ contribution through the recurrent-capacity
argument, and finally bound the bad-interval contribution.

\textbf{Step 2 - Long-unsafe, $\Delta_{\min}/4$-centered active prefixes}

Define
\begin{equation}
\label{eq:long-unsafe-variance-cost}
\Gamma_{\rm LU}
:=
\sum_{m=1}^M
\sum_{t\in\mathcal T_{{\rm LU},m}\cap\mathcal T_+}
\frac1{\sqrt{W_t}}.
\end{equation}

\begin{lemma}[Variance with a long-unsafe active prefix]
\label{lem:long-unsafe-variance-cost}
On $\mathcal G_{\det}$,
\begin{equation}
\label{eq:long-unsafe-variance-cost-bound}
\Gamma_{\rm LU}
\le
CK
+C\frac{\Delta_{\max}}{\sigma}(S+K)\ell_{\id}.
\end{equation}
\end{lemma}

\begin{proof}
Fix an anchored long-safe block $b\in\mathfrak B_m^{\rm LS}$, and write
$n=N_{\mathcal A_t}^{\Dstr}$.  At any round
$t\in\mathcal T_{{\rm LU},m}$ with $b_t=b$,
\[
\left|
\sum_{u\in\mathcal A_t\cap\cS^{\Dstr}}
(\mu_u-\theta^{(m)})
\right|
>\sigma\sqrt n.
\]
Because $\mathcal A_t\subseteq\cB_b$ and $\mu_t=\theta^{(m)}$, its off-anchor decision-stream
mass is at most $c_b^{\Dstr}$; hence the absolute value of the inner sum is at
most $\Delta_{\max}c_b^{\Dstr}$.  Every such round in block $b$ therefore
satisfies
\begin{equation}
\label{eq:long-unsafe-active-count}
n
<
\frac{\Delta_{\max}^2(c_b^{\Dstr})^2}{\sigma^2}.
\end{equation}
Moreover, $W_t\ge N_{\mathcal A_t}^{\Estr}$, the two active-stream counts differ by at
most one, and each value of $n$ occurs on at most two prediction rounds.
Consequently, 
\[
\sum_{\substack{t\in\mathcal T_{{\rm LU},m}\cap\mathcal T_+:\\b_t=b}}
\frac1{\sqrt{W_t}}
\le
C+C\sum_{n=1}^{\lceil L_b\rceil}\frac1{\sqrt n}
\le
C+C\frac{\Delta_{\max}}{\sigma}c_b^{\Dstr},
\]
where $L_b$ is the right-hand side of
Eq.~(\ref{eq:long-unsafe-active-count}).  Summing over blocks and applying
Lemma~\ref{lemma:global-contamination-control}
proves the claim.
\end{proof}

\textbf{Step 3 - Long-safe active prefixes and recurrent capacity}

We now control the first term in
Eq.~(\ref{eq:variance-three-way-decomposition}).  Throughout this step, we fix
a class $m$ and consider the rounds $t$ in $\mathcal T_{{\rm LS},m}$.
At such rounds, the pointwise acceptance argument in Proposition~\ref{prop:compatibility-summary} declares every same-class
long-safe completed block compatible.

\medskip
\noindent\textbf{Step 3.1: Guaranteed class-relevant mass.}
The denominator $W_t$ may contain blocks of several types.  We consider only
the completed blocks guaranteed to be reused at a round in
$\mathcal T_{{\rm LS},m}$, as these are the blocks that can provide the monotonicity
argument below.  Their total mass gives a lower bound on $W_t$ and hence a
valid upper bound on the variance term.
Define the "live" repository
\begin{equation}
\label{eq:causal-live-repository}
\mathcal S_m(t)
:=
\{b\in\mathfrak B_m^{\rm LS}:b<b_t,\ H_b(t)<B_{\exp}\}
\end{equation}
and the active class-$m$ estimation mass
\begin{equation}
\label{eq:persistent-active-credit}
A_m(t)
:=
N_{\mathcal A_t}^{\Estr}\1\{b_t\in\mathfrak B_m^{\rm LS}\}.
\end{equation}

On $\mathcal G_{\det}$, at every round in $\mathcal T_{{\rm LS},m}$,
\begin{equation}
\label{eq:guaranteed-class-mass}
\mathcal S_m(t)\subseteq\mathcal U_t,
\qquad
W_t
\ge
A_m(t)+\sum_{b\in\mathcal S_m(t)}N_b^{\Estr}.
\end{equation}

Indeed, at a round in $\mathcal T_{{\rm LS},m}$,
$A_m(t)=N_{\mathcal A_t}^{\Estr}$.  Moreover, the active prefix is safe for $m$, while
every block in $\mathcal S_m(t)$ is safe and anchored at $m$.
By Proposition~\ref{prop:compatibility-summary} each such block is
in $\cC_t$.  Since $H_b(t)<B_{\exp}$, it is also in
$\mathcal U_t$.  Keeping only these blocks in the definition of $W_t$ proves
the bound.

\medskip
\noindent\textbf{Step 3.2: Persistent capacity.}
The live sample mass in Eq.~(\ref{eq:guaranteed-class-mass}) may decrease when a
block $b \in \mathcal S_m(t)$ reaches the budget $B_{\exp}$.  We therefore track how much
reusable "capacity" remains.  Define
\begin{equation}
\label{eq:residual-capacity-potential}
\begin{aligned}
\Phi_m(t)
&:=
\sum_{\substack{b\in\mathfrak B_m^{\rm LS}:b<b_t}}
N_b^{\Estr}\bigl(B_{\exp}-H_b(t)\bigr)_+
=
\sum_{b\in\mathcal S_m(t)}
N_b^{\Estr}\bigl(B_{\exp}-H_b(t)\bigr)
\end{aligned}
\end{equation}
and
\begin{equation}
\label{eq:normalized-capacity-counter}
\begin{aligned}
Y_m(t)&:=B_{\exp}A_m(t)+\Phi_m(t),\qquad
x_m(t):=\frac{Y_m(t)}{B_{\exp}-2}.
\end{aligned}
\end{equation}
$\Phi_m(t)$ denotes the remaining reusable capacity. It measures how much future reuse budget remains available in the completed long-safe repository $\mathcal{S}_m(t)$.
$Y_m(t)$ denotes the total reusable capacity. It combines the capacity contributed by the active block with the remaining capacity of the stored repository.
$x_m(t)$ denotes the normalized capacity counter. 

At a round in $\mathcal T_{{\rm LS},m}$,
Eq.~(\ref{eq:guaranteed-class-mass}) and
$0<B_{\exp}-H_b(t)\le B_{\exp}$ give
\begin{equation}
\label{eq:capacity-denominator-lower-bound}
\begin{aligned}
W_t
&\ge
A_m(t)+\sum_{b\in\mathcal S_m(t)}N_b^{\Estr}
\ge
A_m(t)+\frac{\Phi_m(t)}{B_{\exp}} = \frac{Y_m(t)}{B_{\exp}}
=
c_Bx_m(t),
\end{aligned}
\end{equation}
where $c_B:=\frac{B_{\exp}-2}{B_{\exp}}$.
Since $c_Bx_m(t)$ is a
lower bound on $W_t$, growth of $x_m(t)$ yields an upper bound on the variance
term. Therefore, our main goal is to show the growth of $x_m(t)$. 

First, a class-$m$ ${\rm LS}/{\rm LU}$ interval is a maximal consecutive interval
contained in
$\mathcal T_{{\rm LS},m}\cup\mathcal T_{{\rm LU},m}$.  Between two
consecutive estimation-stream observations becoming available to the
predictor, the alternating split allows at most two predictions.  We
use $B_{\exp}>2$; factors depending only on this fixed numerical choice are
absorbed into universal constants.

When an alarm completes an anchored long-safe class-$m$
block $b$, the term $B_{\exp}N_b^{\Estr}$ disappears from
$B_{\exp}A_m$ and the same term enters $\Phi_m$, because the newly completed
block has exposure zero.
Thus block completion preserves $Y_m$ exactly\footnote{We use the observation
from the proof of Lemma~\ref{lem:regret-decomposition} that restarts occur
after decision-stream observations, so the active and stored estimation counts
agree at the transfer.}. This is summarized in the
following lemma.

For $W_t>0$, define the class-$m$ capacity usage
\begin{equation}
\label{eq:class-capacity-expenditure}
g_m(t)
:=
\sum_{b\in\mathcal S_m(t)\cap\mathcal U_t}
\frac{N_b^{\Estr}}{W_t}.
\end{equation}

\begin{lemma}[Capacity growth]
\label{lem:hard-capacity-dynamics}

Consider a complete pair of prediction rounds $\{2v+1,2v+2\}$ contained in
the same class-$m$ ${\rm LS}/{\rm LU}$ interval.  Let $\Delta$ denote the
change from immediately before $X_{2v}$ becomes available to immediately
after the exposure update at round $2v+2$, before $X_{2v+2}$ becomes
available, including any block-completion transfers.  Then
\begin{equation}
\label{eq:long-prefix-capacity-increments}
\Delta Y_m\ge B_{\exp}-2,
\qquad
\Delta x_m\ge1.
\end{equation}
\end{lemma}
\begin{proof}

When a block $b\in\mathcal S_m(t)\cap\mathcal U_t$ is reused, its residual
capacity decreases by at most
\[
\begin{aligned}
N_b^{\Estr}\bigl(B_{\exp}-H_b(t)\bigr)_+
&\mathrel{}-
N_b^{\Estr}
\left(B_{\exp}-H_b(t)-\frac1{W_t}\right)_+
\le
\frac{N_b^{\Estr}}{W_t}.
\end{aligned}
\]
This also covers the round on which the block crosses $B_{\exp}$.  Since the reused
long-safe repositories are disjoint across classes and their reused mass is
part of $W_t$,
\begin{equation}
\label{eq:total-capacity-expenditure}
\sum_{m=1}^M g_m(t)\le1.
\end{equation}
Observation $X_{2v}$ adds $B_{\exp}$ to $Y_m$.  Any block completion preserves
$Y_m$ by the exact transfer above, while
Eq.~(\ref{eq:total-capacity-expenditure}) bounds the total capacity consumed
by the two prediction updates by $2$.  Thus, after $X_{2v}$ becomes
available and through the final exposure update, $Y_m$ remains at least
$B_{\exp}-2$ above its initial value.  Dividing by $B_{\exp}-2$ proves the
stated growth of $x_m$, including at both prediction rounds.
\end{proof}

\medskip
\noindent\textbf{Step 3.3: The ideal recurrent telescope.}
We first derive the upper bound on the variance term in the ideal case with no "leakage", meaning that no capacity is spent on sample reuse during bad intervals.
Let $I_{m,1},\ldots,I_{m,E_m}$ be obtained from the maximal
consecutive intervals contained in
$\mathcal T_{{\rm LS},m}\cup\mathcal T_{{\rm LU},m}$ by omitting at most
four endpoint prediction rounds from each. The estimation observation
$X_{2v}$ becomes available for predictions at $2v+1$ and $2v+2$; trimming
the endpoints leaves only complete pairs of these prediction rounds.
Since $W_t\ge1$ on
$\mathcal T_+$, the omitted rounds contribute at most $4E_m$, which is restored
in Step~3.4. 

Let $q_{m,j}$ be the
number of estimation observations whose following prediction interval is
fully contained in $I_{m,j}$, and let $x_{m,j}$ be the counter immediately
before the first such observation becomes available.
If trimming leaves an empty interval, retain its index, set $q_{m,j}=0$,
and take both $x_{m,j}$ and its terminal counter to be the counter immediately
before the first prediction of the original interval.  When $E_m=0$, set
$x_{m,\mathrm{final}}=0$ and omit the class from the telescope.

First suppose that no class-$m$ capacity is spent between two consecutive
intervals $I_{m,j}$ or during the omitted endpoint predictions.  Iterating
Eq.~(\ref{eq:long-prefix-capacity-increments}) then gives the ideal
recurrence
\begin{equation}
\label{eq:ideal-episode-recurrence}
x_{m,j+1}\ge x_{m,j}+q_{m,j}.
\end{equation}
For the final interval, set
$x_{m,E_m+1}=x_{m,\mathrm{final}}$ equal to the counter at the end of its
trimmed interval, before the omitted endpoint predictions.
If the trimmed interval is empty, use the terminal counter specified above.
After the $i$th counted observation becomes available, the counter remains at
least $x_{m,j}+i$ until the next estimation observation.  Since at most two predictions occur between consecutive estimation-stream observations, Eq.~(\ref{eq:capacity-denominator-lower-bound}) gives
\begin{align}
\sum_{j=1}^{E_m}
\sum_{\substack{t\in I_{m,j}\cap\mathcal T_{{\rm LS},m}}}
\frac1{\sqrt{W_t}}
&\le
\frac{2}{\sqrt{c_B}}\sum_{j=1}^{E_m}\sum_{i=1}^{q_{m,j}}
\frac1{\sqrt{x_{m,j}+i}}
\label{eq:ideal-capacity-telescope}
\le
\frac{4}{\sqrt{c_B}}\sum_{j=1}^{E_m}
\left(\sqrt{x_{m,j}+q_{m,j}}-\sqrt{x_{m,j}}\right)
\\
&\overset{\text{Eq.~(\ref{eq:ideal-episode-recurrence})}}{\le}
\frac{4}{\sqrt{c_B}}\sum_{j=1}^{E_m}
\left(\sqrt{x_{m,j+1}}-\sqrt{x_{m,j}}\right)
=\frac{4}{\sqrt{c_B}}\left(\sqrt{x_{m,E_m+1}}-\sqrt{x_{m,1}}\right)
\le \frac{4}{\sqrt{c_B}}\sqrt{x_{m,\mathrm{final}}}.
\nonumber
\end{align}

The active and completed blocks represented in $Y_m(t)$ are disjoint, and
the retrospective set $\mathfrak B_m^{\rm LS}$ includes the eventual current
anchored long-safe block.  Since each such sample provides at most $B_{\exp}$
capacity,
\begin{equation}
\label{eq:class-capacity-terminal-bound}
x_m(t)
\le
\frac{B_{\exp}}{B_{\exp}-2}
\sum_{b\in\mathfrak B_m^{\rm LS}}N_b^{\Estr}
\qquad\text{for every }t.
\end{equation}
The long-safe repositories are disjoint across classes.  Consequently,
\[
\sum_{m=1}^M x_{m,\mathrm{final}}
\le
\frac{B_{\exp}}{B_{\exp}-2}
\sum_{m=1}^M\sum_{b\in\mathfrak B_m^{\rm LS}}N_b^{\Estr}
\le \frac{T}{c_B}.
\]
Thus, in the ideal comparison,
\begin{align}
\label{eq:ideal-recurrent-variance}
\sum_{m=1}^M\sum_{j=1}^{E_m}
\sum_{\substack{t\in I_{m,j}\cap\mathcal T_{{\rm LS},m}}}
\frac1{\sqrt{W_t}}
&\le
\frac{4}{\sqrt{c_B}}\sum_{m=1}^M\sqrt{x_{m,\mathrm{final}}}
\le \frac{4}{\sqrt{c_B}}
\sqrt{M\sum_{m=1}^Mx_{m,\mathrm{final}}}\le C\sqrt{MT},
\end{align}
where $C$ is universal.

Thus the dependence on the number of changes $S$ disappears: every occurrence of class
$m$ advances the same counter rather than restarting the denominator.
Multiplying Eq.~(\ref{eq:ideal-recurrent-variance}) by $\sigma$ gives the target
$O(\sigma\sqrt{MT})$ variance scale before the leakage correction.

\paragraph{Step 3.4: Adding leakage to the telescope.}
Capacity can leak across bad intervals, but not during an
${\rm LS}/{\rm LU}$ round of another class, due to the
rejection argument in the proof of
Proposition~\ref{prop:compatibility-summary}.  This changes the
recurrence between intervals; the growth within each $I_{m,j}$ remains as in
Step~3.3.

For a bad interval $J\in\mathcal J_{\bad}$, define its
normalized class-$m$ leakage by
\begin{equation}
\label{eq:normalized-bad-interval-leakage}
d_{m,J}
:=
\frac1{B_{\exp}-2}
\sum_{t\in J\cap\mathcal T_+}g_m(t).
\end{equation}

Let $J^-$ and $J^+$ denote the instants immediately before and
after, respectively, all updates associated with $J$.
On $\mathcal{G}_{\det}$, class-\(m\) capacity can decrease
across $J$ only through exposure spent on reused class-\(m\) blocks.  Hence
\begin{equation}
\label{eq:bad-interval-capacity-leakage}
x_m(J^+)\ge\bigl(x_m(J^-)-d_{m,J}\bigr)_+.
\end{equation}
Moreover, since \(\sum_m g_m(t)\le1\), summing over \(t\in J\) gives
\begin{equation}
\label{eq:bad-interval-capacity-leakage2}
\sum_m d_{m,J}\le |J|/(B_{\exp}-2).
\end{equation}

Return now to the intervals $I_{m,j}$ defined in the zero-leakage comparison.
For $j<E_m$, let $\mathcal J_{m,j}$ be the bad intervals between
$I_{m,j}$ and $I_{m,j+1}$, and set
\begin{equation}
\label{eq:inter-episode-leakage}
d_{m,j}:=\sum_{J\in\mathcal J_{m,j}}d_{m,J}.
\end{equation}
Combining the growth bound from Step~3.3 with the intervening leakage gives,
for $j<E_m$ and a universal constant $C$,
\begin{equation}
\label{eq:leaky-episode-recurrence}
x_{m,j+1}
\ge
\bigl(x_{m,j}+q_{m,j}-d_{m,j}-C\bigr)_+.
\end{equation}
Here $C$ accounts for capacity spent during the omitted endpoint predictions.

\begin{lemma}[Leaky square-root telescope]
\label{lem:classwise-leaky-telescope}
On $\mathcal G_{\det}$,
\begin{equation}
\label{eq:classwise-leaky-telescope}
\begin{aligned}
\sum_{t\in\mathcal T_{{\rm LS},m}\cap\mathcal T_+}
\frac1{\sqrt{W_t}}
&\le
C\sqrt{x_{m,\mathrm{final}}}
+C\sum_{J\in\mathcal J_{\bad}}\sqrt{d_{m,J}}
+CE_m.
\end{aligned}
\end{equation}
\end{lemma}

\begin{proof}
The first two inequalities in Eq.~(\ref{eq:ideal-capacity-telescope}) apply
unchanged within each trimmed interval.  Only the recurrence used to
telescope across intervals changes.  For $j<E_m$,
Eq.~(\ref{eq:leaky-episode-recurrence}) implies
$x_{m,j}+q_{m,j}\le x_{m,j+1}+d_{m,j}+C$.  Hence
\[
\sqrt{x_{m,j}+q_{m,j}}
\le
\sqrt{x_{m,j+1}}+\sqrt{d_{m,j}}+C.
\]
For the final interval,
$x_{m,E_m}+q_{m,E_m}\le x_{m,\mathrm{final}}$ as in Step~3.3.
The same telescope therefore gives the terminal-capacity term, together
with a correction of at most
$C\sum_{j<E_m}\sqrt{d_{m,j}}+CE_m$.
Each bad interval belongs to at most one gap between consecutive class-$m$
intervals, so subadditivity of the square root gives
\[
\sum_{j<E_m}\sqrt{d_{m,j}}
=
\sum_{j<E_m}
\sqrt{\sum_{J\in\mathcal J_{m,j}}d_{m,J}}
\le
\sum_{J\in\mathcal J_{\bad}}\sqrt{d_{m,J}}.
\]
Restoring the omitted endpoint predictions, whose contribution is at most
$CE_m$, proves Eq.~(\ref{eq:classwise-leaky-telescope}).
\end{proof}

It remains to aggregate the leakage across classes.  For every bad interval
$J$, Cauchy--Schwarz and
Eq.~(\ref{eq:bad-interval-capacity-leakage2}) yield
\[
\sum_{m=1}^M\sqrt{d_{m,J}}
\le
\sqrt{M\sum_{m=1}^Md_{m,J}}
\le
\sqrt{\frac{M|J|}{B_{\exp}-2}}.
\]
The terminal-capacity calculation in
Eq.~(\ref{eq:ideal-recurrent-variance}) gives
$\sum_m\sqrt{x_{m,\mathrm{final}}}\le C\sqrt{MT}$.
Summing Eq.~(\ref{eq:classwise-leaky-telescope}) over $m$ therefore gives
\begin{align}
\label{eq:long-safe-variance-with-leakage}
\sum_{m=1}^M
\sum_{t\in\mathcal T_{{\rm LS},m}\cap\mathcal T_+}
\frac1{\sqrt{W_t}}
&\le
C\sqrt{MT}
+C\sqrt{\frac{M}{B_{\exp}-2}}
\sum_{J\in\mathcal J_{\bad}}\sqrt{|J|}
+C\sum_{m=1}^ME_m.
\end{align}
The two remaining quantities depend only on the bad intervals and are
controlled in Step~4.

\textbf{Step 4 - Bad rounds}

We finally control the contribution of the bad intervals.  Their square-root
complexity governs both the third term in
Eq.~(\ref{eq:variance-three-way-decomposition}) and the leakage correction in
Eq.~(\ref{eq:long-safe-variance-with-leakage}); it will also control the number
of intervals $I_{m,j}$.  We first bound the direct variance on each bad
interval and then use the segmentation and contamination structure to bound
their total complexity.

A bad interval $J$ does not cross a detector-block boundary.
Hence $N_{\mathcal A_t}^{\Estr}$ increases once every two rounds, while each
positive count occurs on at most two prediction rounds.  Since
$W_t\ge N_{\mathcal A_t}^{\Estr}$, treating the initial zero-count
predictions separately and summing $n^{-1/2}$ gives
\begin{equation}
\label{eq:bad-round-variance}
\sum_{J\in\mathcal J_{\bad}}
\sum_{t\in J\cap\mathcal T_+}\frac1{\sqrt{W_t}}
\le
C\sum_{J\in\mathcal J_{\bad}}\sqrt{|J|} =:C \,\Xi_{\bad}.
\end{equation}

It remains to bound $\Xi_{\bad}$.  On $\mathcal G_{\det}$,
the bad rounds consist of rounds in unanchored or unsafe blocks, rounds in
anchored long-safe blocks outside the two displayed classes, and
compatibility-deviation rounds.

\begin{lemma}[Bad-interval upper bound]
\label{lem:bad-interval-complexity}
On $\mathcal G_{\det}$
\begin{equation}
\label{eq:bad-complexity-coarse}
\Xi_{\bad}
\le
C(S+K)\sqrt{\ell_{\id}}+CN_{\rm cmp}
+C\frac{\Delta_{\max}}{\sigma}
(S+K)\ell_{\id}.
\end{equation}
\end{lemma}

\begin{proof}
We bound each type of bad interval separately.
When intervals of different types overlap, assign each round to the first applicable type.  Splitting
an interval can only increase the resulting upper bound because
$\sqrt{x+y}\le\sqrt x+\sqrt y$.  We repeatedly use the elementary bound
\begin{equation}
\label{eq:bad-family-cauchy-schwarz}
\sum_{J\in\mathcal F}\sqrt{|J|}
\le
\sqrt{|\mathcal F|\sum_{J\in\mathcal F}|J|}
\end{equation}
for each type $\mathcal F$ of bad intervals.

\medskip
\noindent\textbf{Unanchored blocks.}
As in Appendix~\ref{app:population-bias}, the true segments and detector blocks produce $O(S+K)$ stationary
intervals.  Every interval in an unanchored block contains
$O(\ell_{\id})$ observations.  Hence these
intervals contribute $O((S+K)\sqrt{\ell_{\id}})$. 

\medskip
\noindent\textbf{Unsafe blocks.}
Let $\mathfrak B_{\unsafe}$ denote the anchored unsafe blocks.
For an anchored block,
\[
\left|\bar\mu_b^{\Dstr}
-\theta^{(\mathfrak a(b))}\right|
\le \Delta_{\max}\frac{c_b^{\Dstr}}{N_b^{\Dstr}}.
\]
If the block is unsafe, combining this inequality with
$|\bar\mu_b^{\Dstr}-\theta^{(\mathfrak a(b))}|
>\sigma/\sqrt{N_b^{\Dstr}}$ gives
\begin{equation}
\label{eq:unsafe-block-length}
N_b^{\Dstr}
<
\frac{\Delta_{\max}^2(c_b^{\Dstr})^2}{\sigma^2}.
\end{equation}
Using the alternating-stream part of the boundary convention,
taking square roots and summing shows that the contribution of unsafe blocks is at most
\begin{align}\label{eq:app:step4-1}
  C\frac{\Delta_{\max}}\sigma
\sum_{b\in\mathfrak B_{\unsafe}}
c_b^{\Dstr}.
\end{align}
The unsafe blocks are a subset of the anchored blocks.  Therefore
Lemma~\ref{lemma:global-contamination-control} together with Eq.~(\ref{eq:app:step4-1}) implies
the third term in Eq.~(\ref{eq:bad-complexity-coarse}).

\medskip
\noindent\textbf{Initial and mismatched intervals of anchored long-safe blocks.}
As in Appendix~\ref{app:population-bias}, the rounds before the current pure suffix accumulates
$\ell_{\id}$ decision-stream observations form an interval of length
$O(\ell_{\id})$; afterward, the no-alarm argument makes the active prefix
long and ensures
$|\bar\mu_{\mathcal A_t}^{\Dstr}-\mu_t|\le\Delta_{\min}/4$.
Within a same-anchor stationary interval in a fixed detector block, the
cumulative population deviation from the anchor stays fixed while the
prefix count grows.  The length and centering conditions therefore persist
once satisfied.  Apart from compatibility deviations, the bad rounds thus
form a single, possibly empty, initial interval.
A class-anchor mismatch
interval is an off-anchor stationary interval and also has length
$O(\ell_{\id})$ by anchor uniqueness (Proposition~\ref{prop:anchored-segmentation}).  There are $O(S+K)$ intervals in total, so
these contribute $O((S+K)\sqrt{\ell_{\id}})$.

\medskip
\noindent\textbf{Compatibility deviations.}
Finally, the total contribution of rounds in $\mathcal E_t^{\rm cmp}$ is at most $N_{\rm cmp}$.  This proves
Lemma~\ref{lem:bad-interval-complexity}.
\end{proof}

To complete the proof of Lemma~\ref{lem:variance-upper-bound},
order all intervals $I_{m,j}$ chronologically, over every
$m$ and $j$. Within a detector block its anchor is fixed, so consecutive
good intervals in that block are separated by a nonempty bad interval.
Allowing one initial good interval per block gives
\[
\sum_mE_m
\le
K+|\mathcal J_{\bad}|
\le
K+\Xi_{\bad}.
\]
Substituting this inequality and
$\sum_{J\in\mathcal J_{\bad}}\sqrt{|J|}=\Xi_{\bad}$ into
Eq.~(\ref{eq:long-safe-variance-with-leakage}) bounds the
$\mathcal T_{{\rm LS},m}$ contribution
by
\[
C\sqrt{MT}
+C\left(1+\sqrt{\frac{M}{B_{\exp}-2}}\right)\Xi_{\bad}
+CK.
\]
Adding the $\mathcal T_{{\rm LU},m}$ contribution $\Gamma_{\rm LU}$ and the
direct bad-round bound Eq.~(\ref{eq:bad-round-variance}) gives, on
$\mathcal G_{\det}$,
\[
V_T
\le
C\sqrt{MT}
+C\Gamma_{\rm LU}
+C(1+\sqrt M)\Xi_{\bad}+CK.
\]
Substitute Eq.~(\ref{eq:long-unsafe-variance-cost-bound}) and
Eq.~(\ref{eq:bad-complexity-coarse}). Since $\ell_{\id}\ge1$, the extra
$CK$ is absorbed by the $(S+K)\sqrt{\ell_{\id}}$ term. Multiplying by
$\1\{\mathcal G_{\det}\}$, taking expectations, and using
$K\le S+1$ on $\mathcal G_{\det}$,
$\E[N_{\rm cmp}]\le\alpha$, and $\ell_{\id}\ge1$ proves
Eq.~(\ref{eq:good-event-effective-sample-size}).
\end{proof}

\subsection{Localization of detector deviations}
\label{app:detector-localization}

The preceding bias and variance bounds use $\mathcal G_{\det}$ to identify
anchors, control contamination, and ensure that the active prefix is long and
within $\Delta_{\min}/4$ of the current mean outside a controlled number of
rounds.  We now allow detector deviations.  In each block, we charge the prefix
through the last deviation by its length, then recover the required properties
on the suffix that follows it.

The main proof idea is the following. 
If this suffix contains no sufficiently long stationary interval, we treat the
block as unanchored and charge the short stationary intervals in the suffix as
before.  Otherwise, a long class-$m$ interval either triggers an alarm or makes
the population mean of the full prefix close to $\theta^{(m)}$.  In the absence
of an alarm, this interval therefore establishes an anchor.  A later long
interval from another class would trigger an alarm, so the same suffix cannot
contain long intervals from two distinct classes.  A long interval that lies
mostly before the last deviation may leave too few observations in the suffix
to establish an anchor; these observations are covered by the same prefix and
short-interval charges.  We first formalize this structure and then return to
the bias and variance bounds.

\textbf{Step 1 - Post-deviation structure}

Let
\begin{equation}
\label{eq:detector-deviation-cost}
e_t^{\det}:=\1\{\mathcal D_{r(t),t}^{\det}\},
\qquad
\Xi_{\det}:=
\sum_{t=2}^T(t-r(t))e_t^{\det}.
\end{equation}
For a detector block $\cB_b=[r_b,q_b)$, where $q_b:=r_{b+1}$, let
$t_b^\star$ be the last round $t\in\{2,\ldots,T\}$ for which
$r(t)=r_b$ and $e_t^{\det}=1$, with $t_b^\star=r_b$ if no such round exists.
Truncate this time at the block's right boundary:
\begin{equation}
\label{eq:last-detector-deviation}
s_b:=\min\{q_b,t_b^\star\}.
\end{equation}
We call
\begin{equation}
\label{eq:deviation-free-suffix}
\mathcal R_b:=[s_b,q_b)
\end{equation}
the \emph{deviation-free suffix}.  We use this suffix only to identify an
anchor.  All means, counts, exposures, and capacity terms still refer to the
full stored block $\cB_b$; the affected prefix $[r_b,s_b)$ remains in the
estimator, and its cost is accounted for by $\Xi_{\det}$.

A block is \emph{locally anchored} if its deviation-free suffix contains a pure interval
with at least $\ell_{\id}$ decision-stream observations.  The corresponding
class, shown below to be unique, is denoted by
$\mathfrak a_{\rm loc}(b)$.  All other blocks are called locally unanchored;
write $\mathfrak B_{\anc}^{\rm loc}$ and
$\mathfrak B_{\un}^{\rm loc}$ for the two collections.  For a locally
anchored block and each stream $a\in\{\Dstr,\Estr\}$, define
\begin{equation}
\label{eq:localized-contamination}
c_{b,\rm loc}^a
:=
\sum_{u\in\cB_b\cap\cS^a}
\1\{\mu_u\ne\theta^{(\mathfrak a_{\rm loc}(b))}\}.
\end{equation}
Call a locally anchored block $b$ \emph{locally safe} if
\[
|\bar\mu_b^{\Dstr}-\theta^{(\mathfrak a_{\rm loc}(b))}|
\le\frac{\sigma}{\sqrt{N_b^{\Dstr}}},
\]
and locally unsafe otherwise.  Every locally anchored block is long, since
its suffix contains at least $\ell_{\id}$ decision-stream observations.
The next lemma extends the anchor and contamination guarantees of
Lemmas~\ref{lem:anchor-formation} and
\ref{lemma:global-contamination-control} to this local definition.

\begin{lemma}[Localized detector structure]
\label{lem:robust-detector-localization}
On every sample path,
\begin{equation}
\label{eq:localized-prefix-cost}
\sum_{b=0}^{K-1}(s_b-r_b)\le\Xi_{\det}.
\end{equation}
Each deviation-free suffix contains a long interval from at most one class.  If
$b$ is locally anchored at class $m$, its full stored block satisfies
\begin{equation}
\label{eq:localized-anchor-mean}
|\bar\mu_b^{\Dstr}-\theta^{(m)}|
\le\Delta_{\min}/4.
\end{equation}
Moreover, for $a\in\{\Dstr,\Estr\}$,
\begin{equation}
\label{eq:localized-structural-budget}
\sum_{b\in\mathfrak B_{\anc}^{\rm loc}}c_{b,\rm loc}^a
+\sum_{b\in\mathfrak B_{\un}^{\rm loc}}N_b^a
\le
C\bigl((S+K)\ell_{\id}+\Xi_{\det}\bigr).
\end{equation}
The number of rounds with a short active prefix or a population mean farther
than $\Delta_{\min}/4$ from $\mu_t$ satisfies
\begin{equation}
\label{eq:localized-long-centered-prefix-budget}
\begin{aligned}
&\sum_{t=2}^T
\1\!\left\{
N_{\mathcal A_t}^{\Dstr}<\ell_{\id}
\ \text{or}\ 
|\bar\mu_{\mathcal A_t}^{\Dstr}-\mu_t|>\Delta_{\min}/4
\right\}
\le C\bigl((S+K)\ell_{\id}+\Xi_{\det}\bigr).
\end{aligned}
\end{equation}
Finally, outside $\mathcal E_t^{\rm cmp}$ the local argument in the proof of
Proposition~\ref{prop:compatibility-summary} applies with local anchors: a locally
safe block is accepted when the active prefix is safe for the same class,
while no block with a different local anchor is accepted when the active
prefix is long and its population mean lies within $\Delta_{\min}/4$ of the
corresponding class mean.
\end{lemma}

\begin{proof}
If $s_b>r_b$, Eq.~(\ref{eq:last-detector-deviation}) gives
\[
s_b-r_b\le(t_b^\star-r_b)e_{t_b^\star}^{\det}.
\]
The selected times are distinct across detector cycles.  Summing this
inequality proves Eq.~(\ref{eq:localized-prefix-cost}).

If $s_b<q_b$, every relevant detector evaluation after $s_b$ is
deviation-free.  Hence every subsequent no-alarm round satisfies
\begin{equation}
\label{eq:localized-no-alarm-bound}
D_{k,t}^{r_b}<2\gamma_t^{r_b}
\qquad\text{for every valid split }k,
\end{equation}
as in Eq.~(\ref{eq:no-alarm-population-bound}).  To see how a pure interval in
the suffix anchors the full prefix, let $[u,v)\subseteq\mathcal R_b$ be a
pure class-$m$ interval, and write
$p=N_{[r_b,u)}^{\Dstr}$ and $h=N_{[u,v)}^{\Dstr}$.  If $p=0$, the full prefix
is pure.  Otherwise, applying the preceding display at the split $u$ gives,
as in the proof of Lemma~\ref{lem:anchor-formation},
\begin{equation}
\label{eq:localized-split-bound}
\frac{ph}{p+h}
\frac{\left(\bar\mu_{[r_b,u)}^{\Dstr}
-\theta^{(m)}\right)^2}{\sigma^2}
<4(\gamma_v^{r_b})^2.
\end{equation}
The full-prefix mean is a weighted average of the preceding mean and
$\theta^{(m)}$, so
\begin{equation}
\label{eq:localized-anchor-calculation}
\begin{aligned}
\left|\bar\mu_{[r_b,v)}^{\Dstr}-\theta^{(m)}\right|
&=
\frac{p}{p+h}
\left|\bar\mu_{[r_b,u)}^{\Dstr}-\theta^{(m)}\right|
\le
\frac{2\sigma\gamma_v^{r_b}}{\sqrt h}
\sqrt{\frac{p}{p+h}}
\le
\frac{2\sigma\gamma_v^{r_b}}{\sqrt h}.
\end{aligned}
\end{equation}
For $h\ge\ell_{\id}$, the last term is at most $\Delta_{\min}/8$ by
Eq.~(\ref{eq:ell-id-main}).  Thus the full prefix, including observations before $s_b$, is anchored at this class. Lemma~\ref{lem:anchor-formation}, applied to later deviation-free
evaluations, keeps its mean within $\Delta_{\min}/4$ of this class mean and
rules out a second anchor. At $\mathcal A_{q_b}=\cB_b$, use the bound above
if the pure interval first reaches $\ell_{\id}$ decision-stream observations at $q_b$,
and the bound for subsequent prefixes otherwise. This proves
Eq.~(\ref{eq:localized-anchor-mean}). A detector deviation at the closing
evaluation gives $s_b=q_b$, so no local anchor is assigned.

It remains to account for the observations not addressed by the local anchor.
Charge each affected prefix $[r_b,s_b)$ by
Eq.~(\ref{eq:localized-prefix-cost}).  The true changes, detector boundaries, and
suffix boundaries determine $O(S+K)$ stationary intervals.  Every off-anchor
interval in a locally anchored suffix, and every interval in a locally unanchored suffix,
contains $O(\ell_{\id})$ decision-stream observations and hence
$O(\ell_{\id})$ observations; otherwise it would
form, respectively, a competing local anchor or a local anchor.  This proves
Eq.~(\ref{eq:localized-structural-budget}).

For Eq.~(\ref{eq:localized-long-centered-prefix-budget}), first charge all
prediction rounds through $s_b$ by Eq.~(\ref{eq:localized-prefix-cost}).  On the
remaining deviation-free suffix, refine the block at the true changes and
charge the first $O(\ell_{\id})$ rounds of each stationary interval.  Thereafter
the current pure suffix contains at least $\ell_{\id}$ decision-stream
observations.  If the evaluation
at $t>s_b$ does not raise an alarm, it is deviation-free and
Eq.~(\ref{eq:localized-anchor-calculation}) makes $\mathcal A_t$ long and gives
$|\bar\mu_{\mathcal A_t}^{\Dstr}-\mu_t|\le\Delta_{\min}/4$.  If it raises an alarm,
the post-update prediction belongs to the new detector cycle and is instead
charged among that cycle's initial $O(\ell_{\id})$ rounds.  This proves the
claimed bound.

Finally, compatibility depends only on the aggregate active-prefix and
full-block means.  Active-prefix safety and local block safety give the
same-anchor implication exactly as before.  For different anchors,
Eq.~(\ref{eq:localized-anchor-mean}) and the corresponding bound on the active
mean give the same $\Delta_{\min}/2$ separation used
in the rejection argument of
Proposition~\ref{prop:compatibility-summary}.  The same sample-size argument then
excludes every block with a different local anchor.
\end{proof}

\textbf{Step 2 - Returning to the bias and variance bounds}

Lemma~\ref{lem:robust-detector-localization} recovers the properties of
$\mathcal G_{\det}$ used in the preceding proofs, with an additional cost
controlled by $\Xi_{\det}$.  We can therefore apply the same exposure and
capacity arguments with this enlarged cost.

\begin{lemma}[Removing the detector conditioning]
\label{lem:remove-detector-conditioning}
For a universal constant $C$, on every sample path,
\begin{equation}
\label{eq:localized-population-bias}
\begin{aligned}
P_T
&\le
C\Delta_{\max}\bigl((S+K)\ell_{\id}+\Xi_{\det}\bigr)
\times\bigl(B_{\exp}+\log(eT)+1\bigr)
+\Delta_{\max}N_{\rm cmp}.
\end{aligned}
\end{equation}
Moreover,
\begin{equation}
\label{eq:localized-effective-sample-size}
\begin{aligned}
V_T\le{}&C\sqrt{MT}
+C(1+\sqrt M)(S+K)\sqrt{\ell_{\id}}
+C(1+\sqrt M)(\Xi_{\det}+N_{\rm cmp})\\
&+C(1+\sqrt M)\frac{\Delta_{\max}}{\sigma}(S+K)\ell_{\id}
+C(1+\sqrt M)\frac{\Delta_{\max}}{\sigma}\Xi_{\det}.
\end{aligned}
\end{equation}
Consequently,
\begin{equation}
\label{eq:unconditional-population-bias}
\begin{aligned}
\E[P_T]
&\le
C\Delta_{\max}\bigl(B_{\exp}+\log(eT)+1\bigr)
\times\bigl((S+1+\alpha)\ell_{\id}
+\alpha(1+\log T)\bigr)
+\Delta_{\max}\alpha,
\end{aligned}
\end{equation}
and
\begin{equation}
\label{eq:expected-effective-sample-size}
\begin{aligned}
\E[V_T]\le{}&C\sqrt{MT}
+C(1+\sqrt M)(S+1+\alpha)\sqrt{\ell_{\id}}
+C(1+\sqrt M)\alpha(2+\log T)\\
&+C(1+\sqrt M)\frac{\Delta_{\max}}{\sigma}(S+1+\alpha)\ell_{\id}
+C(1+\sqrt M)\frac{\Delta_{\max}}{\sigma}\alpha(1+\log T).
\end{aligned}
\end{equation}
\end{lemma}

\begin{proof}
For the bias, repeat Eq.~(\ref{eq:population-bias-two-parts}) with local anchors.
The current-block harmonic argument and the exchange of the time and block
sums in Eq.~(\ref{eq:exposure-sum-swap}) are unchanged.  Correct-anchor
contamination and locally unanchored blocks are controlled by
Eq.~(\ref{eq:localized-structural-budget}) and the pathwise exposure bound
Eq.~(\ref{eq:H-budget}).
The final compatibility statement of
Lemma~\ref{lem:robust-detector-localization} implies that a wrong-anchor
block can be reused only when the active prefix is short, its population mean is
farther than $\Delta_{\min}/4$ from the current class mean, or the compatibility
statistic deviates.  Using
Eq.~(\ref{eq:localized-long-centered-prefix-budget}) gives
Eq.~(\ref{eq:localized-population-bias}) by the arguments of
Lemmas~\ref{lem:current-block-bias} and \ref{lem:reused-block-bias}.

For the variance, repeat the three-way decomposition of
Appendix~\ref{app:variance-control} with locally anchored long-safe blocks.
The first type has a long-safe active prefix; the second has a
long-unsafe active prefix whose population mean is within $\Delta_{\min}/4$
of the current mean.  Both types require $t>s_{b_t}$.  In the first type,
every live locally anchored long-safe class-$m$ block is compatible; since
$H_b(t)<B_{\exp}$, the reuse rule places it in $\mathcal U_t$.  In both types, a
locally anchored long-safe block with another anchor is long and is
therefore excluded by the same localized compatibility conclusion.  These
are the only detection-dependent facts used in the capacity proof.
Apply the same counter construction to the locally anchored long-safe
repositories; every block entering one of them retains its full stored count and
exposure.  The block-completion transfer, unit capacity growth, and leaky square-root telescope therefore remain unchanged.

If an earlier long interval lies before $s_b$, the block may lose its local
anchor.  It is then omitted from the guaranteed historical repository, while
all prediction rounds in that locally unanchored detector block are assigned
to the local bad set.  This cannot invalidate the lower bound on $W_t$; any
actual reuse only increases $W_t$, while capacity spent from other
repositories is covered by bad-interval leakage.  The only new pieces in the
bad-interval derivations are the affected prefixes, whose complexity
satisfies
\[
\sum_b\sqrt{s_b-r_b}
\le
\sum_b(s_b-r_b)
\le\Xi_{\det}.
\]
Every stationary interval in the remaining suffix of a locally unanchored
block has length $O(\ell_{\id})$.  Thus the long-unsafe calculation and the
capacity, leakage, and bad-interval arguments of
Appendix~\ref{app:variance-control} apply with
Eq.~(\ref{eq:localized-structural-budget}), yielding
Eq.~(\ref{eq:localized-effective-sample-size}).

Finally, take expectations and use Eq.~(\ref{eq:K-bound}),
Eq.~(\ref{eq:local-detection-prefix-length}), and
Eq.~(\ref{eq:compatibility-deviation-count}).  This proves
Eq.~(\ref{eq:unconditional-population-bias}) and
Eq.~(\ref{eq:expected-effective-sample-size}).
\end{proof}

\subsection{Combining the bounds}
\label{app:final-assembly}

We now substitute the population-bias and effective-sample-size bounds into
the regret decomposition. Recall $L_T$ from Section~\ref{ssec:upper-bound}
and $\ell_{\id}$ from Eq.~(\ref{eq:ell-id-main}). Since $B_{\exp}=4$,
{
\begin{equation}
\label{eq:assembly-scale-comparisons}
B_{\exp}+\log(eT)+1\le CL_T,
\end{equation}
where $C$ is a universal constant.
}

\begin{proof}[Proof of Theorem~\ref{thm:ecr-upper-bound}]
Fix $\nu\in\mathcal E$.  Multiplying
Eq.~(\ref{eq:expected-effective-sample-size}) by $\sigma$ and using
Eq.~(\ref{eq:assembly-scale-comparisons}) gives
\begin{align}
\sigma\E_\nu[V_T]
\le{}&C\sigma\sqrt{MT}
\label{eq:assembled-variance}
+C\sigma(1+\sqrt M)
\left((S+1)\sqrt{\ell_{\id}}+\alpha L_T\right)
+C\Delta_{\max}(1+\sqrt M)
\left((S+1)\ell_{\id}+\alpha L_T\right).
\end{align}
Similarly, Eq.~(\ref{eq:unconditional-population-bias}) yields
\begin{equation}
\label{eq:assembled-bias}
\E_\nu[P_T]
\le
C\Delta_{\max}L_T
\left((S+1)\ell_{\id}+\alpha L_T\right).
\end{equation}
Finally, Eq.~(\ref{eq:K-bound}) implies
\[
(\sigma+\Delta_{\max})\E_\nu[K]
\le(\sigma+\Delta_{\max})(S+1+\alpha).
\]
Because $L_T,\ell_{\id}\ge1$, this initialization term is absorbed by the
second and third terms of
Eq.~(\ref{eq:assembled-variance}).  Substituting these
bounds into Lemma~\ref{lem:regret-decomposition} proves
Eq.~(\ref{eq:ecr-upper-bound-main}). Since $\nu$ was arbitrary and the constants
are uniform, this proves Theorem~\ref{thm:ecr-upper-bound}.
\end{proof}

For fixed $\sigma,\Delta_{\min},\Delta_{\max}$, and $\alpha$,
Eq.~(\ref{eq:ecr-upper-bound-main}) implies, uniformly over $\nu\in\mathcal E$,
\[
R_T(\mathrm{ECR};\nu)
=
\widetilde O\!\left(
\sigma\sqrt{MT}+(S+1)(1+\sqrt M)
\right).
\]
Under Eq.~(\ref{eq:sparse-change-regime}), the second term is
$\widetilde O(\sqrt{MT})$, which proves
Corollary~\ref{cor:ecr-recurrent-rate}.
\section{Proof of the lower bound}
\label{app:lower-bound}

We isolate two unavoidable sources of regret through two oracle-aided
problems.  First, even when the recurrent means and class labels are known,
the learner must locate the changes.  Second, even when the change points and
class labels are known, the recurrent means must be estimated. 
Since each oracle only helps the learner, either oracle lower bound also
lower-bounds $R_T^\star(\mathcal E)$; we then take their maximum.  Throughout,
$\Delta_{\min}$ and $\Delta_{\max}$ are fixed parameters of the environment
class, with $\Delta_{\max}>(M-1)\Delta_{\min}$.

\subsection{Cost of unknown change points}
\label{app:change-time-lower-bound}

We reveal the class means, the ordered class labels, and a macroblock containing
each change.  The only remaining uncertainty is the change point location
within each macroblock.  

\begin{proposition}[Lower bound due to unknown change times]
\label{prop:unknown-change-time-lower-bound}
There exist universal constants $c,c_0,C_0>0$ such that, if
\begin{equation}
\label{eq:change-time-global-resolution}
\frac{T}{S}\ge C_0,
\qquad
\frac{\sigma^2}{\Delta_{\min}^2}
\log\left(\frac{T}{S}\right)
\le
c_0\sqrt{\frac{T}{S}},
\end{equation}
then
\begin{equation}
\label{eq:change-time-lower-bound}
R_T^\star(\mathcal E)
\ge
c\frac{\sigma^2}{\Delta_{\min}}\,
S\log\!\left(\frac{T}{S}\right).
\end{equation}
\end{proposition}

\begin{proof}
The argument follows the same broad steps as the change-point lower bound of
\cite{gafni2026the}. For completeness, and because the present setting
requires several adaptations, notably, $\Delta_{\min}$ is fixed while the
constructed class-mean diameter $(M-1)\Delta_{\min}$ depends on $M$, we give
a self-contained proof. We first treat one macroblock and then embed
this local problem into a recurrent environment with $S$ changes.

\noindent\textbf{Step 1 - Cost of a single macroblock.}

We use local time within the macroblock, so local round $1$ denotes its first
prediction round.  For the lower bound, we restrict attention to independent
Gaussian observations with variance $\sigma^2$.
Let $I=\{1,\ldots,H\}$ be a macroblock of length $H$, and let
$\theta^{(0)}\neq\theta^{(1)}$ be two known means with
$\Delta:=|\theta^{(0)}-\theta^{(1)}|$. The learner knows that $I$ contains a single
change from $\theta^{(0)}$ to $\theta^{(1)}$, but not its location, denoted by $\tau$.  For
$\tau\in I$, define
\[
\mu_t^{(\tau)}
=
\begin{cases}
\theta^{(0)}, & t<\tau,\\
\theta^{(1)}, & t\ge\tau.
\end{cases}
\]
For a fixed policy, let $\mathbb P_\tau$ and $\E_\tau$ denote the corresponding
probability law and expectation, including the policy's internal randomness,
when the change occurs at time $\tau$.  The same argument
applies after any prior history whose law is common to all candidate change
locations, since such history contributes no relative entropy \citep{gafni2026the}.

\begin{lemma}[One-macroblock change cost]
\label{lem:one-macroblock-timing-cost}
Suppose
\begin{equation}
\label{eq:change-time-resolution-condition}
H\ge64,
\qquad
\frac{\sigma^2}{\Delta^2}\log H\le\sqrt H.
\end{equation}
Then every possibly randomized online policy $\pi$ satisfies
\begin{equation}
\label{eq:one-macroblock-timing-cost}
\sup_{\tau\in I}\E_{\pi,\tau}\left[
\sum_{t\in I}|\widehat\mu_t^\pi-\mu_t^{(\tau)}|
\right]
\ge
c_{\rm time}\frac{\sigma^2}{\Delta}\log H,
\end{equation}
where $c_{\rm time}>0$ is universal.
\end{lemma}

\begin{proof}
We compare one late change point with several early candidates.  If a policy
anticipates the change too often, it incurs large loss under the late
candidate.  Otherwise, it is stable on some window, and an
information-theoretic bound prevents it from adapting quickly enough when the
change instead occurs at the start of that window.  Fix a policy $\pi$ and
suppress it from the probability notation.

\medskip
\noindent\textbf{Step 1.1 - Candidate change locations.}
Fix $1\le \ell\le(H-1)/2$, to be chosen later, and set
$m:=\lfloor(H-1)/\ell\rfloor$.  We place $m$ disjoint windows of length
$\ell$ in the first $H-1$ rounds.  Their start times are the candidate early
change points, which we compare with a late candidate at time $H$.
Specifically, for $i=1,\ldots,m$, define
\[
W_i
:=
\{n_i,\ldots,n_i+\ell-1\},
\qquad
n_i:=1+(i-1)\ell.
\]
By reflecting the observations and predictions if necessary,
suppose without loss of generality that
$\theta^{(1)}>\theta^{(0)}$.  Define
\[
A_i
:=
\left\{
\#\{t\in W_i:\widehat\mu_t>\frac{\theta^{(0)}+\theta^{(1)}}{2}\ \}\ge\ell/2
\right\}.
\]
Thus the event $A_i$ means that at least half of the predictions in $W_i$ lie
on the post-change side of the midpoint.  Under $\tau=H$, every window
contains only observations with mean $\theta^{(0)}$; under $\tau=n_i$, the
window $W_i$ contains only observations with mean $\theta^{(1)}$.  The late
candidate therefore gives
\begin{equation}
\label{eq:late-window-loss}
\sum_{t\in W_i}|\widehat\mu_t-\theta^{(0)}|
\ge
\frac{\ell\Delta}{4}\1\{A_i\},
\end{equation}
whereas the early candidate gives
\begin{equation}
\label{eq:early-window-loss}
\sum_{t\in W_i}|\widehat\mu_t-\theta^{(1)}|
\ge
\frac{\ell\Delta}{4}\1\{A_i^c\}.
\end{equation}

\textit{Late candidate:} We first show that if the policy makes a window error with non-negligible
probability in every window under the late-change instance, then it
already suffers large regret.  Let
$q_i:=\mathbb P_H(A_i)$ denote this error probability on $W_i$.  If
$q_i>1/(4m)$ for every $i$, then the disjointness of the windows and
Eq.~(\ref{eq:late-window-loss}) give
\[
\E_H\left[
\sum_{t\in I}|\widehat\mu_t-\mu_t^{(H)}|
\right]
\ge
\frac{\ell\Delta}{4}\sum_{i=1}^m q_i
>
\frac{\ell\Delta}{16}.
\]
With the choice \(\ell\asymp 1+(\sigma^2/\Delta^2)\log H\), whose admissibility is verified below, the preceding display already gives the desired regret lower bound. We may therefore restrict attention to the complementary case, in which some \(W_j\) satisfies \(q_j\le 1/(4m)\). Since no change has
yet occurred in $W_j$ under $\mathbb P_H$, a small $q_j$ quantifies the
stability of the policy on this window.

\noindent\textit{Early candidate:}
Now consider a change at $n_j$, and let
$p_j:=\mathbb P_{n_j}(A_j)$.  The event $A_j$ depends only on
the policy randomness and observations available before the last prediction
in $W_j$.  Hence its early- and late-change laws differ through at most
$\ell-1$ Gaussian observations, and data processing gives
\begin{equation}
\label{eq:change-time-binary-kl}
\mathrm{kl}(p_j,q_j)
\le
\frac{(\ell-1)\Delta^2}{2\sigma^2},
\end{equation}
where $\mathrm{kl}$ denotes binary relative entropy.

If $q_j=0$, finiteness of the relative entropy forces $p_j=0$.  Otherwise,
if $p_j\ge1/2$, monotonicity of binary relative entropy in its first argument
above $q_j$ gives
\[
\mathrm{kl}(p_j,q_j)
\ge
\mathrm{kl}(1/2,q_j)
=
\frac12\log\frac1{4q_j(1-q_j)}
\ge
\frac12\log m.
\]
Consequently, whenever
\begin{equation}
\label{eq:change-time-information-condition}
\frac{(\ell-1)\Delta^2}{2\sigma^2}
<
\frac12\log m,
\end{equation}
Eq.~(\ref{eq:change-time-binary-kl}) forces $p_j<1/2$.  Then
Eq.~(\ref{eq:early-window-loss}) yields
\[
\E_{n_j}\left[
\sum_{t\in I}|\widehat\mu_t-\mu_t^{(n_j)}|
\right]
\ge
(1-p_j)\frac{\ell\Delta}{4}
\ge
\frac{\ell\Delta}{8}.
\]
Together with the late-candidate case, this shows that, for every $\ell$ satisfying
Eq.~(\ref{eq:change-time-information-condition}), some change location incurs
at least $\ell\Delta/16$ regret.

\medskip
\noindent\textbf{Step 1.2 - Choice of the window length.}
The preceding lower bound increases with $\ell$, while
Eq.~(\ref{eq:change-time-information-condition}) limits how much information can
be accumulated within one window.  Choose
\[
\ell
:=
1+\left\lfloor
\frac{\sigma^2}{8\Delta^2}\log H
\right\rfloor.
\]
Eq.~(\ref{eq:change-time-resolution-condition}) gives
$\ell\le\sqrt H/4$ and hence $m\ge\sqrt H$.  Moreover,
\[
\frac{(\ell-1)\Delta^2}{2\sigma^2}
\le \frac1{16}\log H
< \frac12\log m,
\]
so Eq.~(\ref{eq:change-time-information-condition}) holds.  Finally,
$\ell\ge(\sigma^2/(8\Delta^2))\log H$, and therefore
$\ell\Delta/16\ge c_{\rm time}(\sigma^2/\Delta)\log H$ for a universal
$c_{\rm time}>0$.  This proves Eq.~(\ref{eq:one-macroblock-timing-cost}).
\end{proof}

\medskip
\noindent\textbf{Step 2 - A hard recurrent environment.}

We now embed the single-macroblock problem into an oracle-aided Gaussian
family with $S$ changes.  The class means, class sequence, and macroblocks are
revealed; only the change location inside each macroblock remains unknown.
Set
\[
H:=\left\lfloor\frac{T-1}{S}\right\rfloor.
\]
For $C_0$ sufficiently large and $c_0$ sufficiently small,
Eq.~(\ref{eq:change-time-global-resolution}) implies
Eq.~(\ref{eq:change-time-resolution-condition}) with
$\Delta=\Delta_{\min}$.

Fix the class means
\[
\theta^{(m)}:=(m-1)\Delta_{\min},
\qquad m=1,\ldots,M.
\]
We construct the sequence of class labels so that every class is visited while every transition has magnitude $\Delta_{\min}$.
To this end, let $z_0, \ldots, z_S$ be the first $S+1$ terms of the reflected sequence
\[
1,2,\ldots,M,M-1,\ldots,2,1,2,\ldots.
\]

Associate one change with each of the $S$ disjoint macroblocks
\[
I_s
:=
\{2+(s-1)H,\ldots,1+sH\},
\qquad s=1,\ldots,S,
\]
and allow its location to range over the rounds in that macroblock.  The
adversarial location will be selected below using
Lemma~\ref{lem:one-macroblock-timing-cost}.  The transition in $I_s$ is
from $\theta^{(z_{s-1})}$ to $\theta^{(z_s)}$.  Set the mean at round $1$ to
$\theta^{(z_0)}$, and keep the mean at $\theta^{(z_S)}$ on any rounds left
after $I_S$.  Let $\mathcal E_{\rm time}$ be the resulting family over all
choices of the $S$ change locations, with independent Gaussian noise of
variance $\sigma^2$.
Every environment in $\mathcal E_{\rm time}$ has exactly $S$ changes and uses all $M$ classes.  Moreover, its class-mean diameter is
$(M-1)\Delta_{\min}<\Delta_{\max}$, so
$\mathcal E_{\rm time}\subseteq\mathcal E$.

Fix any oracle policy $\pi$ and choose the change points
successively.  Once the change points in $I_1,\ldots,I_{s-1}$ have been fixed,
all candidate change points in $I_s$ induce the same prior-history law.  Viewing
$I_s$ in the local rounds of
Lemma~\ref{lem:one-macroblock-timing-cost} and taking
$\Delta=\Delta_{\min}$, the lemma gives a change point in $I_s$
whose expected loss is at least
\[
c_{\rm time}\frac{\sigma^2}{\Delta_{\min}}\log H.
\]
This change location may depend on the policy and its induced law, but not on the
realized observations.  Every candidate completes the transition to class
$z_s$ by the end of $I_s$, so the same argument can be continued in
$I_{s+1}$.  By causality, later choices do not affect the losses already
incurred.

The successive choices produce one environment in $\mathcal E_{\rm time}$
for which summing over the disjoint macroblocks gives
\[
\sup_{\nu\in\mathcal E_{\rm time}}
\E_{\pi,\nu}\left[
\sum_{t=2}^T|\widehat\mu_t^\pi-\mu_t|
\right]
\ge
c_{\rm time}\frac{\sigma^2}{\Delta_{\min}}S\log H.
\]
Since the policy was arbitrary, the same lower bound holds for the oracle
minimax regret.  Finally, $H=\lfloor(T-1)/S\rfloor$ and $H\ge64$ imply
$\log H\asymp\log(T/S)$.  This proves
Eq.~(\ref{eq:change-time-lower-bound}).
\end{proof}

\subsection{Cost of unknown class means}
\label{app:class-mean-lower-bound}

We now reveal all change points and class labels, leaving only the recurrent
means unknown.  We first reduce this oracle problem to one binary estimation
problem for each class, solve the binary problem using Le Cam's method, and
then sum the resulting costs. Define
\[
\bar{\Delta}
:=
\Delta_{\max}-(M-1)\Delta_{\min}>0.
\]

\begin{proposition}[Cost of unknown class means]
\label{prop:unknown-class-mean-lower-bound}
There exist universal constants $c,C>0$ such that, for
\[
T
\ge
C\max\left\{
M,
\frac{\sigma^2M^3}{\bar{\Delta}^2}
\right\},
\]
\begin{equation}
\label{eq:class-mean-lower-bound}
R_T^\star(\mathcal E)
\ge
c\sigma\sqrt{MT}.
\end{equation}
\end{proposition}
\begin{proof}[Proof]
The proof follows the steps below.

\textbf{Step 1 - Reduction to coordinate-wise two-point problems}

Assume for simplicity that $(T-1)/M$ is an
integer; the general case changes only universal constants.  Assign the $S+1$
segment labels cyclically as
\[
1,2,\ldots,M,1,2,\ldots,
\]
truncated after $S+1$ labels, and choose their lengths so that every class
appears on $T_m$ of these rounds, where
\begin{equation}
\label{eq:balanced-class-occupancy}
T_m=\frac{T-1}{M}.
\end{equation}

We next make a further reduction by allowing only two possible values for
the mean of each class.  These values lie at distance $\delta_m$ from a
class-dependent center $a_m$ and are therefore separated by $2\delta_m$.
Formally, set
\[
\begin{aligned}
a_m&:=(m-1)\left(\Delta_{\min}+\frac{\bar{\Delta}}{2M}\right),
\qquad
\delta_m:=
\min\left\{\frac{\sigma}{4\sqrt{T_m}},
\frac{\bar{\Delta}}{4M}\right\}.
\end{aligned}
\]
For a sign vector $v=(v_1,\ldots,v_M)\in\{-1,+1\}^M$, define
\begin{equation}
\label{eq:class-mean-two-point-family}
\theta_v^{(m)}:=a_m+v_m\delta_m,
\qquad m=1,\ldots,M,
\end{equation}
and take independent Gaussian noise with variance $\sigma^2$.  Every member
of this family belongs to $\mathcal E$.  Indeed, adjacent class means satisfy
\[
\theta_v^{(m+1)}-\theta_v^{(m)}
\ge
\Delta_{\min}+\frac{\bar{\Delta}}{2M}-\delta_m-\delta_{m+1}
\ge\Delta_{\min},
\]
whereas the diameter is at most
\begin{align*}
(M-1)\left(\Delta_{\min}+\frac{\bar{\Delta}}{2M}\right)
+\frac{\bar{\Delta}}{2M}
&=(M-1)\Delta_{\min}+\frac{\bar{\Delta}}{2}
<\Delta_{\max}.
\end{align*}

Restrict attention to this finite Gaussian subclass. We reveal the
change points, class labels, and the two candidate means for each class, so
that only the sign vector \(v\) remains unknown to the learner. For any such policy $\pi$, define
\begin{align*}
R(\pi,v)
&:=
\E_v\left[
\sum_{t=2}^T|\widehat\mu_t^\pi-\mu_t|
\right],
\qquad
R_m(\pi,v)
:=
\E_v\left[
\sum_{\substack{2\le t\le T:\\\mu_t=\theta_v^{(m)}}}
|\widehat\mu_t^\pi-\theta_v^{(m)}|
\right].
\end{align*}
Then $R(\pi,v)=\sum_{m=1}^M R_m(\pi,v)$, and the oracle reduction gives
\begin{equation}
\label{eq:class-mean-oracle-reduction}
R_T^\star(\mathcal E)
\ge
\inf_{\pi}
\max_{v\in\{-1,+1\}^M}
R(\pi,v).
\end{equation}

Let $V=(V_1,\ldots,V_M)$ have the uniform product distribution on
$\{-1,+1\}^M$.  Here $\E_V$ averages only over this artificial sign prior,
since $R(\pi,v)$ already includes the observation and policy randomness.  For
every fixed policy,
\begin{equation}
\label{eq:class-mean-max-average-reduction}
\max_v R(\pi,v)
\ge
\E_V[R(\pi,V)]
=
\sum_{m=1}^M\E_V[R_m(\pi,V)].
\end{equation}
We finally reduce this $M$-coordinate problem to $M$ binary problems by
considering one class at a time and conditioning on the remaining signs.
For a fixed class \(m\), let
\(
V_{-m}=(V_1,\ldots,V_{m-1},V_{m+1},\ldots,V_M)
\)
denote the signs of all classes other than \(m\).
The
product structure gives
\begin{equation}
\label{eq:class-mean-conditional-reduction}
\E_V[R_m(\pi,V)]
=
\E_{V_{-m}}
\left[
\E\!\left[R_m(\pi,V)\mid V_{-m}\right]
\right],
\end{equation}
where the inner conditional expectation averages over the two possible
values of $V_m$.  It is therefore a two-point estimation problem for the mean
of class $m$.

\medskip
\noindent\textbf{Step 2 - Solving the two-point problem}

We use the following sequential form of Le Cam's two-point bound.

\begin{lemma}[Sequential two-point estimation cost]
\label{lem:sequential-two-point-estimation}
Fix a revealed class-label sequence, and suppose a given class is active on a
set $\mathcal I\subseteq\{2,\ldots,T\}$ of $n\ge2$ rounds, i.e.,
$|\mathcal I|=n$.  Consider two
Gaussian environments that differ only in whether the mean of this class is
\[
\theta^{(0)}:=a-\delta
\qquad\text{or}\qquad
\theta^{(1)}:=a+\delta,
\]
where $\delta\le\sigma/(4\sqrt n)$.  Then every online policy satisfies
\begin{equation}
\label{eq:sequential-two-point-estimation}
\frac12\sum_{u\in\{0,1\}}
\E_u\left[
\sum_{t\in\mathcal I}
|\widehat\mu_t-\theta^{(u)}|
\right]
\ge
\frac38n\delta.
\end{equation}
\end{lemma}

\begin{proof}
We use the same two candidate means throughout all $n$ appearances of the
class.
Fix $t\in\mathcal I$, and let $\Pbb_t^{(1)}$ and $\Pbb_t^{(0)}$ be the
laws of the learner's full pre-$t$ observation history together with its
internal random seed.  At most \(n\) observations from this class have
appeared before round \(t\), while all other observations have the same law
under the two alternatives. Hence, additivity of KL divergence for independent Gaussian observations gives
\begin{equation}
\label{eq:sequential-two-point-kl}
\KL( \Pbb_t^{(1)}\Vert \Pbb_t^{(0)})
\le
\frac{n(2\delta)^2}{2\sigma^2}
\le
\frac18,
\end{equation}
where the last inequality follows from \(\delta\le \sigma/(4\sqrt n)\).
Pinsker's inequality yields
\begin{equation}
  \label{eq:Pinsker}
\left\| \Pbb_t^{(1)}- \Pbb_t^{(0)}\right\|_{\mathrm{TV}}
\le
\sqrt{\frac12
\KL\!\left(
 \Pbb_t^{(1)}\middle\| \Pbb_t^{(0)}
\right)} \le \frac{1}{4}.
\end{equation}
We next translate this testing difficulty into estimation loss. An estimate
on the wrong side of $a$ is at least $\delta$ away from the true mean, so
\begin{align*}
\E_1|\widehat\mu_t-\theta^{(1)}|
&\ge \delta\,\mathbb P_1(\widehat\mu_t<a),
\qquad
\E_0|\widehat\mu_t-\theta^{(0)}|
\ge \delta\,\mathbb P_0(\widehat\mu_t\ge a).
\end{align*}
Le Cam's testing inequality gives
\[
\mathbb P_1(\widehat\mu_t<a)
+\mathbb P_0(\widehat\mu_t\ge a)
\ge
1-\left\| \Pbb_t^{(1)}- \Pbb_t^{(0)}\right\|_{\mathrm{TV}}
\ge\frac34,
\]
where the last inequality is due to Eq.~(\ref{eq:Pinsker}).
Adding the two inequalities, dividing by two, and applying the preceding
bound gives
\begin{equation}
\label{eq:sequential-le-cam-inequality}
\frac12\Big(
\E_1|\widehat\mu_t-\theta^{(1)}|
+\E_0|\widehat\mu_t-\theta^{(0)}|
\Big)
\ge
\frac{3\delta}{8}.
\end{equation}
Summing Eq.~(\ref{eq:sequential-le-cam-inequality}) over the $n$ rounds in
$\mathcal I$ proves Eq.~(\ref{eq:sequential-two-point-estimation}).  In
particular, choosing $\delta=\sigma/(4\sqrt n)$ gives a cumulative cost of
order $\sigma\sqrt n$.
\end{proof}

Return now to Eq.~(\ref{eq:class-mean-conditional-reduction}).  Conditional on
$V_{-m}$, the two hypotheses differ only in whether the class-$m$ mean is
$a_m-\delta_m$ or $a_m+\delta_m$.  Since
$\delta_m\le\sigma/(4\sqrt{T_m})$, Lemma~\ref{lem:sequential-two-point-estimation}
gives
\[
\E\!\left[R_m(\pi,V)\mid V_{-m}\right]
\ge
\frac38T_m\delta_m.
\]
Averaging over $V_{-m}$ and summing over the classes therefore gives, for
every such policy $\pi$,
\begin{equation}
\label{eq:class-mean-product-prior-bound}
\E_V[R(\pi,V)]
\ge
\frac38\sum_{m=1}^M T_m\delta_m.
\end{equation}

\medskip
\noindent\textbf{Step 3 - The final rate.}

Since $T-1\ge T/2$, Eq.~(\ref{eq:balanced-class-occupancy}) gives the following
bound up to universal constants:
\begin{align}
\sum_{m=1}^M T_m\delta_m
&=
\frac14\min\left\{
\sigma\sqrt{M(T-1)},
\frac{\bar{\Delta}}{M}(T-1)
\right\}
\ge
c\min\left\{
\sigma\sqrt{MT},
\frac{\bar{\Delta}}{M}T
\right\}.
\label{eq:class-mean-summed-cost}
\end{align}
Combining Eq.~(\ref{eq:class-mean-oracle-reduction}),
Eq.~(\ref{eq:class-mean-max-average-reduction}),
Eq.~(\ref{eq:class-mean-product-prior-bound}), and
Eq.~(\ref{eq:class-mean-summed-cost}), and then taking the infimum over
all policies with access to the revealed information, gives
\begin{equation}
\label{eq:class-mean-truncated-lower-bound}
R_T^\star(\mathcal E)
\ge
c\min\left\{
\sigma\sqrt{MT},
\frac{\bar{\Delta}}{M}T
\right\}.
\end{equation}
Under the assumed condition on \(T\), the first term is the
minimum in Eq.~(\ref{eq:class-mean-truncated-lower-bound}), proving
Eq.~(\ref{eq:class-mean-lower-bound}). For fixed
\(M,\sigma,\Delta_{\min}\), and \(\Delta_{\max}\), this condition holds for
all sufficiently large \(T\).

\end{proof}

Combining Propositions~\ref{prop:unknown-change-time-lower-bound} and
\ref{prop:unknown-class-mean-lower-bound} proves
Eq.~(\ref{eq:ecr-lower-bound-main}) whenever
Eq.~(\ref{eq:change-time-global-resolution}) and the condition of
Proposition~\ref{prop:unknown-class-mean-lower-bound} hold.

\section{Necessity of the exposure cap for ECR}
\label{app:uncapped-ecr}

We first create short intervals that alternate between means $0$ and $1$.
The detector misses their internal changes and turns each interval into an
unanchored \emph{source block} with a positive mean. We then create long, mostly-zero
\emph{reuse intervals} whose detected blocks are anchored at class $0$ but unsafe. During
the reuse intervals, the compatibility test accepts all of the source blocks
but rejects the blocks from earlier reuse intervals. The same positive
historical mass is therefore reused repeatedly without a growing historical
denominator.

Let $\mathrm{ECR}_{\infty}$ denote Algorithm~\ref{alg:ECR} with
$B_{\exp}=\infty$. Thus every compatible completed block is reused and
$\mathcal U_t=\cC_t$.

\begin{proposition}[Failure of uncapped reuse]
\label{prop:uncapped-ecr}
Fix $\alpha\in(0,1/4)$. There exist numerical constants
$A,c,C>0$ such that, for every sufficiently large $T$, there is a Gaussian
environment $\nu\in\mathcal E_{T,S,3}(1,A,1)$ with
$\Theta=\{0,1,A\}$ and $c\sqrt T\le S\le C\sqrt T$, for which
\[
R_T(\mathrm{ECR}_{\infty};\nu)
\ge c\frac{T^{5/8}}{(\log T)^{1/4}}.
\]
\end{proposition}

\begin{proof}
For ease of presentation, we use conservative constants and state the intermediate bounds for sufficiently large $T$; tighter constants and estimates can reduce the required horizon.

Group rounds $2v-1$ and $2v$ into cell $v$. A cell labeled
$a\in\{0,1,A\}$ means $\mu_{2v-1}=\mu_{2v}=a$, so it gives the decision and
estimation streams the same population mean.
Let $X_t=\mu_t+\varepsilon_t$, where
$\varepsilon_t\stackrel{\mathrm{i.i.d.}}{\sim}\mathcal N(0,1)$.
We reserve \emph{segment} for a true constant-mean segment, use
\emph{interval} for each constructed zero/one stretch, and use \emph{block}
for rounds between consecutive detector alarms.

Take $A:=1024$ and set
\begin{align*}
&q:=\lceil\log T\rceil
\asymp \log T,
\qquad L:=\lfloor\sqrt T\rfloor
\asymp T^{1/2},\\
&N:=\lfloor L/q\rfloor
\asymp \frac{T^{1/2}}{\log T},
\qquad \eta:=256\sqrt{\frac{\log T}{L}}
\asymp \frac{(\log T)^{1/2}}{T^{1/4}},\\
&J:=\lfloor\eta^{-1/2}\rfloor
\asymp \frac{T^{1/8}}{(\log T)^{1/4}},
\qquad p_j:=j\eta
\asymp \frac{j(\log T)^{1/2}}{T^{1/4}},
\qquad 1\le j\le J.
\end{align*}
Then $Nq\asymp L$ and
$p_J\asymp\sqrt\eta\asymp(\log T)^{1/4}/T^{1/8}$; in particular, for
large $T$, $p_j\le p_J\le\sqrt\eta\le1/4$.

\paragraph{Step 1 - Construction.}
\emph{Initialization.}
Let $m_0:=\lfloor T/4\rfloor\asymp T$ and begin with one class-$A$
segment of $m_0$ cells.

\emph{Source intervals.}
For each $n=1,\ldots,N$, append the $q$-cell interval
$(0,1,0,1,\ldots)$, followed by a class-$A$ segment of $q$ cells.
The logarithmic choice of $q$ keeps the resulting blocks small enough for acceptance during reuse intervals, yet leaves enough observations to detect boundaries with the adjacent class-$A$ segments, where the mean gap is at least $A-1$.

\emph{Reuse intervals.}
For each $j=1,\ldots,J$, append an $L$-cell interval containing
$\lfloor Lp_j\rfloor$ one-cells and zeros elsewhere, followed by a class-$A$
segment of $L$ cells. Spread the one-cells evenly: for example, proportions
$1/8$ and $1/4$ repeat $(0,0,0,0,0,0,0,1)$ and $(0,0,0,1)$, respectively.
Formally, the first $v$ cells contain $\lfloor vp_j\rfloor$ ones, so every
$r$ consecutive cells contain either $\lfloor rp_j\rfloor$ or
$\lceil rp_j\rceil$ ones. The interval mean is therefore
$\lfloor Lp_j\rfloor/L$. Fill all remaining rounds with mean $A$.

Fig.~\ref{fig:uncapped-ecr-construction} depicts the constructed zero/one
intervals, the intervening class-$A$ segments, and the resulting detector
blocks; alarm delays are suppressed.

\begin{figure}[t]
  \centering
  \begin{tikzpicture}[
  x=1cm,
  y=1cm,
  >=Latex,
  every node/.style={font=\scriptsize},
  zero/.style={fill=blue!12,draw=blue!70!black,line width=0.4pt},
  one/.style={fill=orange!15,draw=orange!85!black,line width=0.4pt},
  high/.style={fill=green!11,draw=green!55!black,line width=0.6pt},
  alarm/.style={red!70!black,dashed,line width=0.65pt},
  span/.style={black!55,line width=0.45pt},
  reuse/.style={purple!75!black,line width=0.7pt}
]
  \path[use as bounding box] (0,-1.12) rectangle (7.5,5.05);

  % Color key.
  \node[anchor=west,font=\scriptsize\bfseries] at (0.15,4.86) {cell mean:};
  \filldraw[zero] (1.55,4.74) rectangle +(0.22,0.22);
  \node[anchor=west] at (1.84,4.85) {$0$};
  \filldraw[one] (2.25,4.74) rectangle +(0.22,0.22);
  \node[anchor=west] at (2.54,4.85) {$1$};
  \filldraw[high] (2.95,4.74) rectangle +(0.22,0.22);
  \node[anchor=west] at (3.24,4.85) {$A$};
  \draw[alarm] (4.35,4.74) -- (4.35,4.98);
  \node[anchor=west] at (4.47,4.85) {alarm (delay suppressed)};

  % Initial constant-A run.
  \node[font=\scriptsize\bfseries] at (3.75,4.45) {initialization};
  \filldraw[high] (0.60,3.91) rectangle (6.90,4.27);
  \node at (3.75,4.09) {$A\quad A\quad\cdots\quad A$};
  \node at (3.75,3.73) {$m_0$-cell class-$A$ segment};

  % N source intervals, each followed by a class-A segment.
  \draw[span] (0.60,3.42) -- (0.60,3.52) -- (6.90,3.52) -- (6.90,3.42);
  \node[fill=white,inner sep=1pt,font=\scriptsize\bfseries] at (3.75,3.52)
    {$N$ source intervals and following $A$-segments};

  \begin{scope}[shift={(0.60,2.82)},x=0.175cm]
    \foreach \i in {0,2,4,6} {
      \filldraw[zero] (\i,0) rectangle ++(1,0.38);
    }
    \foreach \i in {1,3,5,7} {
      \filldraw[one] (\i,0) rectangle ++(1,0.38);
    }
  \end{scope}
  \node[fill=white,fill opacity=0.78,text opacity=1,inner sep=0.6pt]
    at (1.30,3.01) {$0,1,0,1,\ldots$};
  \filldraw[high] (2.05,2.82) rectangle (3.45,3.20);
  \node at (2.75,3.01) {$A,\ldots,A$};
  \node at (3.75,3.01) {$\cdots$};
  \begin{scope}[shift={(4.05,2.82)},x=0.175cm]
    \foreach \i in {0,2,4,6} {
      \filldraw[zero] (\i,0) rectangle ++(1,0.38);
    }
    \foreach \i in {1,3,5,7} {
      \filldraw[one] (\i,0) rectangle ++(1,0.38);
    }
  \end{scope}
  \node[fill=white,fill opacity=0.78,text opacity=1,inner sep=0.6pt]
    at (4.75,3.01) {$0,1,0,1,\ldots$};
  \filldraw[high] (5.50,2.82) rectangle (6.90,3.20);
  \node at (6.20,3.01) {$A,\ldots,A$};

  \foreach \x in {0.60,2.025,3.45,4.05,5.475,6.90} {
    \draw[alarm] (\x,2.73) -- (\x,3.29);
  }
  \node at (1.30,2.62) {$q$ cells};
  \node at (2.75,2.62) {$q$ cells};
  \node at (4.75,2.62) {$q$ cells};
  \node at (6.20,2.62) {$q$ cells};
  \node[purple!75!black,font=\tiny] at (1.30,2.38) {unanchored};
  \node[green!45!black] at (2.75,2.38) {$A$-block};
  \node[purple!75!black,font=\tiny] at (4.75,2.38) {unanchored};
  \node[green!45!black] at (6.20,2.38) {$A$-block};

  % Only the source blocks feed the common historical pool.
  \node[draw=purple!75!black,fill=purple!4,rounded corners=1.5pt,
    inner xsep=4pt,inner ysep=2.5pt,font=\scriptsize\bfseries,
    text=purple!75!black] (sources) at (3.025,1.97)
    {same source blocks reused};
  \draw[reuse] (1.30,2.28) to[out=270,in=90] (sources.north west);
  \draw[reuse] (4.75,2.28) to[out=270,in=90] (sources.north east);
  \draw[reuse,-{Latex[length=1.5mm]}]
    (sources.south west) to[out=250,in=90] (0.625,1.19);
  \draw[reuse,-{Latex[length=1.5mm]}]
    (sources.south) to[out=270,in=90] (2.955,1.19);
  \draw[reuse,-{Latex[length=1.5mm]}]
    (sources.south east) to[out=290,in=90] (5.285,1.19);

  \begin{scope}[yshift=-0.65cm]
  % Representative p_1 interval: one orange cell among twelve.
  \begin{scope}[shift={(0.15,1.43)},x=0.07917cm]
    \foreach \i in {0,...,11} {
      \filldraw[zero] (\i,0) rectangle ++(1,0.38);
    }
    \filldraw[one] (10,0) rectangle ++(1,0.38);
  \end{scope}
  \node[fill=white,fill opacity=0.82,text opacity=1,inner sep=0.6pt]
    at (0.625,1.62) {$p_1$};
  \filldraw[high] (1.15,1.43) rectangle (2.10,1.81);
  \node at (1.625,1.62) {$A$};

  \node at (2.29,1.62) {$\cdots$};

  % Representative intermediate p_j interval.
  \begin{scope}[shift={(2.48,1.43)},x=0.07917cm]
    \foreach \i in {0,...,11} {
      \filldraw[zero] (\i,0) rectangle ++(1,0.38);
    }
    \foreach \i in {5,11} {
      \filldraw[one] (\i,0) rectangle ++(1,0.38);
    }
  \end{scope}
  \node[fill=white,fill opacity=0.82,text opacity=1,inner sep=0.6pt]
    at (2.955,1.62) {$p_j$};
  \filldraw[high] (3.48,1.43) rectangle (4.43,1.81);
  \node at (3.955,1.62) {$A$};

  \node at (4.62,1.62) {$\cdots$};

  % Representative p_J interval: the largest one-cell proportion.
  \begin{scope}[shift={(4.81,1.43)},x=0.07917cm]
    \foreach \i in {0,...,11} {
      \filldraw[zero] (\i,0) rectangle ++(1,0.38);
    }
    \foreach \i in {3,7,11} {
      \filldraw[one] (\i,0) rectangle ++(1,0.38);
    }
  \end{scope}
  \node[fill=white,fill opacity=0.82,text opacity=1,inner sep=0.6pt]
    at (5.285,1.62) {$p_J$};
  \filldraw[high] (5.81,1.43) rectangle (6.76,1.81);
  \node at (6.285,1.62) {$A$};

  \foreach \x in {0.15,1.125,2.10,2.48,3.455,4.43,4.81,5.785} {
    \draw[alarm] (\x,1.34) -- (\x,1.90);
  }
  \foreach \x in {0.625,1.625,2.955,3.955,5.285,6.285} {
    \node at (\x,1.22) {$L$ cells};
  }
  \node[purple!75!black,font=\tiny,align=center] at (0.625,0.90)
    {anchor $0$\\[-1pt]unsafe};
  \node[green!45!black] at (1.625,0.90) {$A$-block};
  \node[purple!75!black,font=\tiny,align=center] at (2.955,0.90)
    {anchor $0$\\[-1pt]unsafe};
  \node[green!45!black] at (3.955,0.90) {$A$-block};
  \node[purple!75!black,font=\tiny,align=center] at (5.285,0.90)
    {anchor $0$\\[-1pt]unsafe};
  \node[green!45!black] at (6.285,0.90) {$A$-block};
  \end{scope}

  % J reuse intervals with increasing one-cell proportion p_j.
  \draw[span] (0.15,0.0) -- (0.15,-0.10) -- (7.35,-0.10) -- (7.35,0.0);
  \node[fill=white,inner sep=1pt,font=\scriptsize\bfseries] at (3.75,-0.10)
    {$J$ reuse intervals and following $A$-segments, $p_j=j\eta$};
  \node[text=black!75,font=\scriptsize] at (3.75,-0.43)
    {other completed blocks rejected};

  \draw[->,black!65,line width=0.55pt] (0.55,-0.74) -- (6.95,-0.74)
    node[right] {$v$};
  \node[anchor=west,black!70] at (0.55,-0.98)
    {cell index $v$ runs left-to-right, then continues on the next row};
\end{tikzpicture}
  \caption{\textbf{Environment illustration.}
  Each cell represents two rounds. Each source interval is followed by a
  $q$-cell $A$-segment and produces an unanchored source block. Each reuse
  interval is followed by an $L$-cell $A$-segment and produces an
  unsafe block anchored at $0$; the constant-$A$ segments produce class-$A$
  blocks. During the latter half of each reuse interval, the same source
  blocks are accepted and reused, while all other completed blocks are
  rejected. The illustration is schematic and not to scale; alarm delays
  are suppressed.}
  \label{fig:uncapped-ecr-construction}
\end{figure}
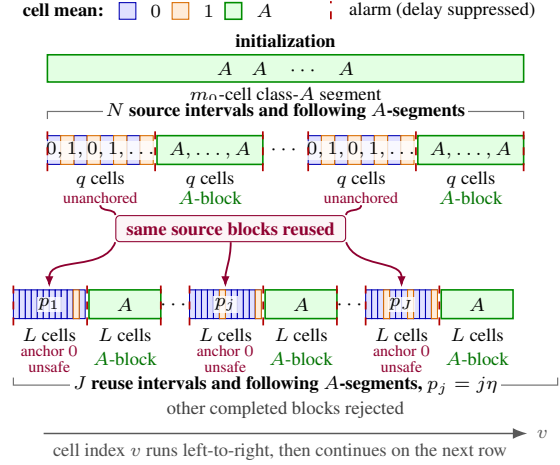

\paragraph{Step 2 - Horizon and number of changes.}
The displayed construction uses
$2m_0+4Nq+4JL=T/2+o(T)$ rounds, so the final class-$A$ segment fills the
remaining horizon. The source intervals contribute
$N(q-1)=\Theta(L)$ changes, while the reuse intervals contribute at
most
\[
2\sum_{j=1}^J\lfloor Lp_j\rfloor
\le L\eta J(J+1)=O(L).
\]
The boundaries with the class-$A$ segments add only $O(N+J)$ changes.
Consequently, $S=\Theta(L)=\Theta(\sqrt T)$.

\paragraph{Step 3 - Source intervals.}
We bound the number of observations and the population mean of each source block. These bounds will establish that the compatibility test accepts the source blocks during the latter half of each subsequent reuse interval.
We restrict attention to the latter half to ensure enough observations for rejecting blocks from earlier reuse intervals, while retaining order \(L\) zero-mean prediction rounds for bias accumulation.

We begin with an upper bound on the detection delay at the boundaries between the source intervals and the adjacent class-$A$ segments.

\textit{Delay upper bound:} Denote the population mean on the $0-1$ side by $0\le\bar \mu_L\le1$ and the $A$ mean by $\bar \mu_R=A$. The difference satisfies $|\bar \mu_L - \bar \mu_R|\ge A-1$.
We work on $\mathcal G_{\det}$ and the event with no internal alarms described below. All sample counts here are per-stream counts, with one observation per cell.
Next, consider the split point $k$ of the statistic at the boundary between the $0-1$ and $A$ intervals, and let $d\le q/4$ be a candidate upper bound on the delay, measured in cells. We argue inductively over the boundaries, starting with the initial long $A$-segment. If the preceding boundary was detected within $d$ cells, then the number of pre-change samples satisfies $N_L\ge q-d\ge d$.
Suppose no alarm has occurred after $d$ incoming cells. The statistic at this split point is
\[
D^2 = \frac{N_L \cdot d}{N_L+d}(\bar \mu_{L}-\bar \mu_{R})^2 \ge \frac{d}{2}(A-1)^2.
\]  
Since the threshold of the change detector test satisfies (for large $T$) $(\gamma_t^r)^2\le12\log T$, and under the good detector event $D^2>4(\gamma_t^r)^2$ guarantees an alarm, a sufficient condition for detection by $d$ cells is
\begin{align*}
&\frac{d}{2}(A-1)^2 > 48 \log T  \nonumber \Longleftrightarrow d > \frac{96 \log T}{(A-1)^2}.
\end{align*}
We therefore choose 
\begin{align}
d:=\left\lceil128\log T/(A-1)^2\right\rceil.  
\end{align}
Since $A=1024$, this choice also satisfies $d\le q/4$ for sufficiently large $T$. An alarm must therefore occur within $d$ cells. The same argument applies with the zero/one and $A$ sides reversed, completing the induction.

\textit{No internal alarms:} For a scan whose samples lie within one source or reuse interval, even spacing gives $D_{k,t}^r\le\sqrt{1/n_1+1/n_2}\le\sqrt2$, where $n_1,n_2$ are its two decision-stream counts. Within a class-$A$ segment, $D_{k,t}^r=0$. For each deterministic valid triple, Gaussian tails therefore bound the alarm probability by
\[
2e^{-(\gamma_t^r-\sqrt2)^2/2}
\le C e^{\sqrt{24\log T}}\frac{\alpha_r}{(t-r)^3}.
\]
There are at most $t-r$ splits for each $r,t$. Summing over all such scans with $r\ge T/3$ shows that the probability of any internal alarm is at most
\[
C e^{\sqrt{24\log T}}\sum_{r\ge T/3}\alpha_r=o(1).
\]
On $\mathcal G_{\det}$, there is no alarm inside the initial class-$A$ segment, so every subsequent restart is after $T/3$. Together with the delay bound and induction above, this gives exactly one alarm within $d$ cells after each boundary between a source or reuse interval and a class-$A$ segment, and no other alarms, with probability at least $1-\alpha-o(1)$.
Alarms occur at even rounds and restart at the preceding odd round, so blocks closed by alarms consist of whole cells. Their sample counts and population means are therefore identical in the two streams.

\textit{block length:} We next use the above result to obtain a lower and upper bound on the number of samples in a source block, denoted here by $N_b=N_b^{\Dstr}=N_b^{\Estr}$. Since the source intervals have length $q$, we may "lose" at most $d$ samples per stream due to the alarm starting the block, and retain at most $d+1$ samples per stream after the end of the interval, including the endpoint allowance. The entry and exit delays need not be equal. We have
$q-d\le N_b\le q+d+1$; therefore
\begin{align}
  \frac{3q}{4} \le N_b \le 2q.
\end{align}
Each source block contains only one-cell class-$0$ and class-$1$ segments and at most $d+1<\ell_{\id}$ class-$A$ cells, so it is unanchored for sufficiently large $T$.

\textit{block bias:} We next obtain a lower bound (for the bias contribution) and an upper bound (for the acceptance guarantee) on the block mean.
For the upper bound, note that we have at most $\lfloor q/2\rfloor$ one-cells and at most $d+1$ $A$-cells; therefore
\[
\bar \mu_b \le \frac{\lfloor q/2\rfloor + A(d+1)}{N_b} \le 1+\frac{A(d+1)}{q-d} \le 2,
\]
where the last inequality holds for sufficiently large $T$ by plugging in $q,d$ and using $A=1024$.

For the lower bound, note that we have $\lfloor q/2\rfloor$ one-cells in the interval. We may lose at most $d$ of these due to the late entry alarm, while any retained $A$-cells contribute nonnegative mass. Therefore, for sufficiently large $T$,
\begin{align*}
& \bar \mu_b \ge \frac{\lfloor q/2\rfloor-d}{N_b} \ge \frac{q/4-1}{N_b} \ge \frac{q/8}{N_b}. 
\end{align*}
We therefore conclude
\begin{align}\label{eq:counter-bias-short}
  \frac{q/8}{N_b} \le \bar \mu_b \le 2.
\end{align}

\textit{Acceptance guarantee:} We now show that the compatibility test accepts these source blocks during the latter half of each later reuse interval.
Consider the active prefix $\mathcal A_t$ during this part of the interval. It starts after the entry alarm and contains no $A$-cells, so $0\le\bar \mu_{\mathcal A_t}\le1$. Therefore $|\bar \mu_{\mathcal A_t}-\bar \mu_b|\le2$. All counts and population means in this comparison use the decision stream. The population compatibility statistic satisfies
\[
Q_b^2(t) = \frac{N_{\mathcal A_t}N_b}{N_{\mathcal A_t}+N_b}(\bar \mu_{\mathcal A_t}-\bar \mu_b)^2 \le 4N_b \le 8q.
\]  
Since $t\ge T/3$, the compatibility threshold satisfies $(\lambda_t^b)^2\ge10\log T$ for sufficiently large $T$. Together with $q\le2\log T$, this gives $Q_b^2(t)\le16\log T<4(\lambda_t^b)^2$. On the good compatibility event $\mathcal G_{\rm cmp}$, we therefore have $\widehat Q_b(t)<Q_b(t)+\lambda_t^b<3\lambda_t^b<4\lambda_t^b$, so block $b$ is accepted.

\paragraph{Step 4 - Blocks from reuse intervals.}
We now need to show two things: rejection of blocks from earlier reuse intervals during later reuse intervals, and accumulation of repeated bias beyond the recurrent benchmark scale.
As in Step 3, we first derive bounds on the number of samples and on the mean differences. All counts and means in the compatibility comparisons below use the decision stream; we suppress its superscript.

\textit{Block length:} Each reuse interval contains $L$ cells. We may lose at most $d$ samples per stream due to the entry alarm and retain at most $d+1$ additional samples per stream due to the exit alarm.
Recall that $d<\log T$ and $L=\lfloor\sqrt T\rfloor$. For sufficiently large $T$, we have $6\log T\le L$, so $d\le L/6$. Thus, for a completed block $b$ of this type,
\begin{align}
  L-d \le N_b \le L+d+1
  & \implies \frac{L}{3} \le N_b \le 2L.
\end{align}
During the latter half of the reuse interval, the active prefix contains at least $L/2-d$ samples and at most $L$ samples, since it starts after the entry alarm and remains inside the interval. Therefore,
\begin{align}
  \frac{L}{3} \le N_{\mathcal A_t} \le L.
\end{align}
The completed-block count bounds hold in both streams, and the active estimation count also satisfies $N_{\mathcal A_t}^{\Estr}\le L$. We will charge the bias on zero-mean rounds in this latter half.

\textit{Bounding the population mean:} We start by bounding the mean of the active prefix in reuse interval $j$. It contains $n=N_{\mathcal A_t}$ consecutive zero/one cells and no $A$-cells. By even spacing, the number of ones is either $\lfloor np_j\rfloor$ or $\lceil np_j\rceil$, and therefore
\begin{align}
  |\bar\mu_{\mathcal A_t}-p_j| \le \frac1n \le \frac3L \le \frac\eta4,
\end{align}
where the last inequality follows from the definition of $\eta$ for sufficiently large $T$.

Next, let $b_i$ be the completed block from an earlier reuse interval $i<j$. Its retained zero/one cells form a consecutive portion of the pattern. By even spacing, their number of ones differs from $p_i$ times their number of cells by at most one. Each of the at most $d+1$ retained $A$-cells contributes $A-p_i\le A$ to the deviation from $p_i$. Therefore,
\begin{align*}
  |\bar\mu_{b_i}-p_i| \le \frac{1+A(d+1)}{N_{b_i}} \le \frac{3[1+A(d+1)]}{L}.
\end{align*}
Since $d+1\le2\log T$ for $\log T\ge2$, we have
\begin{align}
  |\bar\mu_{b_i}-p_i| \le \frac{3(1+2A)\log T}{L} = \frac{6147}{65536}\eta^2 \le \frac\eta4,
\end{align}
where we used $A=1024$ and the definition of $\eta$. The last inequality holds for $\eta\le1$, as ensured by the construction.

These blocks are anchored at class $0$. Indeed, even spacing leaves interior zero segments of length at least $(2p_J)^{-1}\gg\ell_{\id}$; since $Lp_1\to\infty$ and $d=o(L)$, at least one survives the entry delay. Proposition~\ref{prop:anchored-segmentation} therefore gives the anchor. The mean bound also gives $\bar\mu_{b_i}\ge p_i-\eta/4\ge3\eta/4\gg N_{b_i}^{-1/2}$, so each such block is unsafe. The same argument applies to every completed block from a reuse interval, and the two streams have identical population means and counts.

\textit{Mutual rejection:} We next lower-bound the compatibility statistic between the active prefix in reuse interval $j$ and the earlier completed block $b_i$. First,
\begin{align}
  |\bar\mu_{\mathcal A_t}-\bar\mu_{b_i}| \ge |p_j-p_i|-\frac\eta4-\frac\eta4 \ge \frac\eta2,
\end{align}
where the last inequality follows from $p_j-p_i=(j-i)\eta\ge\eta$.
The sample-count lower bounds give $N_{\mathcal A_t}N_{b_i}/(N_{\mathcal A_t}+N_{b_i})\ge L/6$. Together with the mean separation, this yields
\begin{align}
  Q_{b_i}^2(t) \ge \frac L6\left(\frac\eta2\right)^2 = \frac{256^2}{24}\log T > 500\log T.
\end{align}
For sufficiently large $T$, the threshold satisfies $(\lambda_t^{b_i})^2\le20\log T$, so
\[
  Q_{b_i}^2(t)>25(\lambda_t^{b_i})^2,
\]
and hence $Q_{b_i}(t)>5\lambda_t^{b_i}$.
On the good compatibility event $\mathcal G_{\rm cmp}$,
\[
  \widehat Q_{b_i}(t)>Q_{b_i}(t)-\lambda_t^{b_i}>4\lambda_t^{b_i},
\]
so every block from an earlier reuse interval is rejected during the latter half of reuse interval $j$.

\textit{Rejection of the $A$-blocks:} Let $b_A$ be a completed block from an $A$-segment of $m\ge q$ cells. This covers the $q$-cell and $L$-cell $A$-segments, as well as the initial $A$-segment. The block loses at most $d$ $A$-cells at entry and retains at most $d+1$ zero/one cells at exit. Therefore,
\[
N_{b_A}\ge m-d\ge q/2,
\]
and, since the added cells have nonnegative means,
\[
\bar\mu_{b_A}\ge A\frac{m-d}{m+d+1}\ge\frac A2,
\]
where the last inequality follows from $m\ge q\ge4d$ and $d\ge1$.
The active prefix in the current reuse interval contains only zeros and ones, so $0\le\bar\mu_{\mathcal A_t}\le1$. Hence,
\[
|\bar\mu_{\mathcal A_t}-\bar\mu_{b_A}|\ge A/2-1.
\]
During the latter half of this interval, $N_{\mathcal A_t}\ge L/3\ge q/2$. Thus both sample counts are at least $q/2$, giving
\[
Q_{b_A}^2(t)\ge\frac q4(A/2-1)^2.
\]
With $A=1024$, $q\ge\log T$, and $(\lambda_t^{b_A})^2\le20\log T$ for sufficiently large $T$, we obtain
\begin{align}
Q_{b_A}^2(t)>500\log T
\ge25(\lambda_t^{b_A})^2.
\end{align}
On the good compatibility event $\mathcal G_{\rm cmp}$, this implies $\widehat Q_{b_A}(t)>Q_{b_A}(t)-\lambda_t^{b_A}>4\lambda_t^{b_A}$, so every completed $A$-block is rejected during the latter half of the interval.

\paragraph{Step 5 - Repeated bias accumulation.}
Finally, we lower-bound the regret caused by repeated bias accumulation. As before, we consider the latter half of each reuse interval. We work on $\mathcal G_0$, the intersection of $\mathcal G_{\det}$, the event with no internal alarms, and $\mathcal G_{\rm cmp}$. All counts and population means below refer to the estimation stream, and sums over $b$ run over the $N$ source blocks.
We first upper-bound the number of samples used by the estimator, since they may dilute the bias.

\textit{Number of samples:} Recall that the test accepts all source blocks and rejects all other completed blocks. Each source block contains at most $2q$ samples, and the active prefix in the latter half of reuse interval $j$ contains at most $L$. Since $Nq\le L$,
\begin{align}
  W_t = N_{\mathcal A_t} + \sum_b N_b \le L+2Nq \le 3L \le 4L.
\end{align} 

\textit{Bias mass:} From Eq.~(\ref{eq:counter-bias-short}), we have $N_b\bar\mu_b\ge q/8$ for each source block. Also,
\[
Nq \ge L-q\ge L/2
\]
for sufficiently large $T$. Therefore,
\begin{align}
  \sum_b N_b\bar\mu_b \ge \frac{Nq}{8} \ge \frac{L}{16}.
\end{align}
Let $\widetilde\mu_t$ denote the population estimate obtained by replacing the observations in the prediction average with their means, while keeping the selected samples fixed. Consider a zero-mean prediction round in the latter half of a reuse interval. Since all population means in the construction are nonnegative, the active prefix cannot cancel the source blocks' positive contribution. Thus,
\begin{align*}
|\widetilde\mu_t-\mu_t|
&=\widetilde\mu_t \ge\frac{\sum_bN_b\bar\mu_b}{W_t}\ge\frac{L/16}{4L}
=\frac1{64}.
\end{align*}

\textit{Repeated accumulation:} Since $p_j\le1/4$, even spacing ensures that the latter half of reuse interval $j$ contains at least
\[ 3L/8-2\ge L/8 \]
zero-cells for sufficiently large $T$. Indeed, this half contains $\lfloor L/2\rfloor$ cells and at most $p_j\lfloor L/2\rfloor+1$ ones. Its number of zero-cells is therefore at least $\frac34\lfloor L/2\rfloor-1\ge3L/8-2$; the subtraction of $2$ allows for rounding.

Each zero-cell contributes one even prediction round with population error at least $1/64$. The same argument applies to all $J$ reuse intervals, giving
\begin{align}
\sum_{t\in\mathcal T_+}|\widetilde\mu_t-\mu_t|
\ge\frac{JL}{512} \asymp \frac{T^{5/8}}{(\log T)^{1/4}},
\end{align}
on $\mathcal G_0$, where $\mathcal T_+=\{t\in\{2,\ldots,T\}:W_t>0\}$. Since $M=3$, this exceeds the recurrent benchmark scale $\sqrt{MT}$ by a factor of order $T^{1/8}/(\log T)^{1/4}$.

Conditioning on the decision samples fixes the blocks, reuse decisions, and weights, while the estimation noises remain centered. Thus the conditional mean of $\widehat\mu_t$ is $\widetilde\mu_t$ whenever $W_t>0$, and Jensen's inequality gives
\[
\mathbb E\!\left[
|\widehat\mu_t-\mu_t|
\,\middle|\,\text{decision samples}
\right]
\ge|\widetilde\mu_t-\mu_t|.
\]
Since $\mathcal G_0$ depends only on the decision samples and has probability at least $1-2\alpha-o(1)$, the population lower bound on this event yields
\[
R_T(\mathrm{ECR}_\infty;\nu)
\ge \mathbb P(\mathcal G_0)\frac{JL}{512}
\gtrsim JL.
\]
This proves the proposition.

\end{proof}

\paragraph{Effect of the exposure cap.}
The exposure cap changes neither the detector blocks nor the compatibility
decisions. On $\mathcal G_0$, Eq.~(\ref{eq:H-budget}) gives
$H_b(T+1)\le B_{\exp}+1$. Moreover, the per-block bounds above imply
$\sum_b N_b^{\Estr}\bar\mu_b^{\Estr}\le4Nq\le4L$, where the sum runs over the source blocks. Their contribution on zero-mean rounds is therefore at most
\[
\sum_b N_b^{\Estr}\bar\mu_b^{\Estr}H_b(T+1)
\le4L(B_{\exp}+1).
\]
For $B_{\exp}=4$, Theorem~\ref{thm:ecr-upper-bound} gives
\[
R_T(\mathrm{ECR};\nu)
=O\!\left(\sqrt T(\log T)^2\right).
\]
Thus only uncapped reuse accumulates the additional factor $J$.

\section{Retrospective limits}
\label{app:retrospective}
\begin{lemma}[Retrospective localization uncertainty] \label{lem:retrospective-localization} 
Fix two known means $\theta^{(0)}\neq\theta^{(1)}$ and let $\Delta:=|\theta^{(1)}-\theta^{(0)}|$. 
Consider the Gaussian subclass 
  \[ 
    X_t\sim \begin{cases} \mathcal N(\theta^{(0)},\sigma^2), & t<\eta,\\ \mathcal N(\theta^{(1)},\sigma^2), & t\ge\eta, 
    \end{cases} 
  \] 
independently across $t$, where the change point $\eta$ is unknown. 
Let $\widehat\tau$ be any possibly randomized estimator based on the entire sample $X_{1:T}$. 
Then, for every integer $h \ge 1$ and every $\tau$ such that $\tau+2h \le T$, 
  \begin{equation} \label{eq:retrospective-localization-lb} 
\max_{\eta\in\{\tau,\tau+2h\}} \Pbb_\eta\!\left(|\widehat\tau-\eta|\ge h\right) \ge \frac14\exp\!\left(-\frac{h\Delta^2}{\sigma^2}\right). 
  \end{equation} 
Consequently, for such $h$, if for some $\delta\in(0,1/4)$, 
  \[ 
    \sup_\eta \Pbb_\eta\!\left(|\widehat\tau-\eta|\ge h\right) \le\delta, 
  \] 
then necessarily 
\begin{equation} \label{eq:retrospective-localization-radius} 
h \ge \frac{\sigma^2}{\Delta^2} \log\frac{1}{4\delta}. 
\end{equation} 
\end{lemma} 

\begin{proof}
The proof constructs two environments that are different enough: their
change points are separated by \(2h\) samples, so that confusing them causes
a localization error of at least \(h\), while remaining similar enough in KL
divergence that no estimator can reliably distinguish between them.
Let
\[
A:=\{\widehat\tau\ge\tau+h\}.
\]
The event \(A\) corresponds to placing the estimated change point on or to the
right of the midpoint between the two possible change locations. Under a
change at \(\tau\), the event \(A\) implies
\[
|\widehat\tau-\tau|\ge h,
\]
whereas under a change at \(\tau+2h\), the event \(A^c\) implies
\[
|\widehat\tau-(\tau+2h)|\ge h.
\]
Therefore, the sum of the localization-error probabilities under the two
environments satisfies
\[
\begin{aligned}
\Pbb_\tau\!\left(|\widehat\tau-\tau|\ge h\right)
+\Pbb_{\tau+2h}\!\left(
|\widehat\tau-(\tau+2h)|\ge h
\right)
&\ge
\Pbb_\tau(A)+\Pbb_{\tau+2h}(A^c).
\end{aligned}
\]

By the Bretagnolle--Huber inequality,
\[
\Pbb_\tau(A)+\Pbb_{\tau+2h}(A^c)
\ge
\frac{1}{2}
\exp\!\left(
-\mathrm{KL}\!\left(
\Pbb_\tau\,\middle\|\,\Pbb_{\tau+2h}
\right)
\right).
\]

It remains to compute the KL divergence between the two observation laws.
They differ only on the \(2h\) indices
\(t=\tau,\ldots,\tau+2h-1\). Hence,
\begin{align*}
\mathrm{KL}\!\left(
\Pbb_\tau\,\middle\|\,\Pbb_{\tau+2h}
\right)
&=
2h\,
\mathrm{KL}\!\left(
\mathcal N(\theta^{(1)},\sigma^2)
\,\middle\|\,
\mathcal N(\theta^{(0)},\sigma^2)
\right)
=
2h\frac{\Delta^2}{2\sigma^2}
=
\frac{h\Delta^2}{\sigma^2}.
\end{align*}
Substituting this expression into the Bretagnolle--Huber inequality gives
\[
\Pbb_\tau(A)+\Pbb_{\tau+2h}(A^c)
\ge
\frac{1}{2}
\exp\!\left(-\frac{h\Delta^2}{\sigma^2}\right).
\]

Consequently, at least one of the two localization-error probabilities is at
least one half of the right-hand side. Therefore,
\[
\max_{\eta\in\{\tau,\tau+2h\}}
\Pbb_\eta\!\left(|\widehat\tau-\eta|\ge h\right)
\ge
\frac{1}{4}
\exp\!\left(-\frac{h\Delta^2}{\sigma^2}\right),
\]
which proves Eq.~(\ref{eq:retrospective-localization-lb}).
If both localization-error probabilities are at most \(\delta\), then
\[
\delta
\ge
\frac{1}{4}
\exp\!\left(-\frac{h\Delta^2}{\sigma^2}\right),
\]
which rearranges to
\[
h
\ge
\frac{\sigma^2}{\Delta^2}
\log\frac{1}{4\delta}.
\]
This proves Eq.~(\ref{eq:retrospective-localization-radius}).
\end{proof}

\begin{remark} 
The lower bound gives the estimator the two mean values adjacent to the change and the entire observation sequence \(X_{1:T}\), leaving only the change location unknown. It therefore also applies to any online detector followed by arbitrary retrospective post-processing. If the history is split at \(\widehat\tau\) when the true change point is \(\eta\), then exactly \(|\widehat\tau-\eta|\) observations are assigned to the wrong neighboring block. For a jump of size \(\Delta=\Delta_{\min}\) with \(\sigma^2/\Delta_{\min}^2\ge1\), take \(h=\lceil\sigma^2/\Delta_{\min}^2\rceil\) and assume \(T\ge2h+2\).
Lemma~\ref{lem:retrospective-localization} then shows that, for any hard split, some change location yields at least \(h\) misassigned observations with probability at least \(e^{-2}/4\). This two-point argument does not establish an additional logarithmic
dependence on \(T\); obtaining such a dependence is beyond the scope of this paper.
\end{remark}
\section{Extension to \texorpdfstring{$L_p$}{Lp}}
\label{app:extension-p}
For a fixed $p>0$, consider replacing the absolute loss by the cumulative
$p$-power loss
\begin{equation}
\label{eq:regret-p}
\mathbb E\!\left[\sum_{t=2}^T
|\widehat\mu_t-\mu_t|^p\right].
\end{equation}
The corresponding per-environment error and minimax criterion are defined
analogously. This section gives only a high-level discussion;
we do not claim an $L_p$ regret guarantee for ECR.

\paragraph{The value of recurrence.}
We first isolate the stochastic estimation term under Gaussian noise. If
$\widehat\mu_n$ is the sample mean of $n$ independent observations from
$\mathcal N(\mu,\sigma^2)$, then
\[
\widehat\mu_n-\mu\sim\mathcal N(0,\sigma^2/n),
\qquad
\mathbb E|\widehat\mu_n-\mu|^p
=c_p\sigma^p n^{-p/2},
\]
where $c_p:=\mathbb E|Z|^p$ for $Z\sim\mathcal N(0,1)$. Consequently,
\[
\sum_{n=1}^N\mathbb E|\widehat\mu_n-\mu|^p
\asymp
\sigma^p
\begin{cases}
N^{1-p/2}, & 0<p<2,\\
\log(eN), & p=2,\\
1, & p>2.
\end{cases}
\]
For the sub-Gaussian model of this paper, the corresponding right-hand sides
remain upper bounds up to a constant depending on $p$.

Consider equal-length segments and balanced total occupancy across the $M$
classes, ignoring rounding and first-prediction initialization costs. A
no-reuse benchmark that relearns each of the $S+1$ segments and a
recurrent benchmark that pools all observations from the same recurrent class
have the respective stochastic-loss scales
\begingroup
\setlength{\arraycolsep}{2.5pt}
\[
\begin{array}{c|c|c}
 & \text{No reuse} & \text{Recurrent pooling} \\ \hline
0<p<2
& \sigma^p(S+1)^{p/2}T^{1-p/2}
& \sigma^pM^{p/2}T^{1-p/2} \\[2pt]
p=2
& \sigma^2(S+1)\log\!\left(\dfrac{eT}{S+1}\right)
& \sigma^2M\log\!\left(\dfrac{eT}{M}\right) \\[4pt]
p>2
& \sigma^p(S+1)
& \sigma^pM.
\end{array}
\]
\endgroup
Thus, whenever $M<S+1$, the recurrent benchmark improves upon the no-reuse benchmark for every fixed
$p>0$. For $p<2$,
the improvement factor is of order
$((S+1)/M)^{p/2}$; at $p=1$ this recovers the improvement from
$\sigma\sqrt{(S+1)T}$ to $\sigma\sqrt{MT}$.

\paragraph{The cost of unknown change times.}
The testing argument underlying the change-time lower bound gives the
following heuristic for $p$-power loss. A jump of size $\Delta$ requires a
localization window of order
$\sigma^2\Delta^{-2}\log(T/S)$, while each inaccurate prediction in that
window incurs loss of order $\Delta^p$. Repeating this argument over $S$
changes suggests the unavoidable scale
\[
\sigma^2\Delta_{\min}^{p-2}S\log\!\left(\frac{T}{S}\right),
\]
under resolution conditions analogous to
Eq.~(\ref{eq:change-time-global-resolution}). For $p=2$, this agrees with the single-change lower bound of \cite{gafni2026the}. For $p>2$,
larger admissible jumps may produce a stronger worst-case dependence than the
display based on $\Delta_{\min}$.

\paragraph{Implications for the minimax rate.}
For $1\le p<2$, adapting the two oracle lower-bound arguments suggests the
schematic minimax benchmark
\[
\max\!\left\{
\sigma^pM^{p/2}T^{1-p/2},
\ \sigma^2\Delta_{\min}^{p-2}S\log\!\left(\frac{T}{S}\right)
\right\}.
\]
A corresponding algorithmic target would contain the sum of these two terms,
up to logarithmic factors and additional parameter-dependent contamination
terms.

The recurrent estimation term can dominate the change-time term only if
\[
S\log\!\left(\frac{T}{S}\right)
\lesssim
\left(\frac{\Delta_{\min}}{\sigma}\right)^{2-p}
M^{p/2}T^{1-p/2}.
\]
For fixed $p,M,\sigma$, and $\Delta_{\min}$, this reduces, up to logarithmic
factors, to
\[
S=\widetilde O\!\left(T^{1-p/2}\right).
\]
For example, the resulting powers are $1/2$ for $p=1$, $1/4$ for
$p=3/2$, and $0.1$ for $p=1.8$. Hence the admissible growth of $S$ becomes
more restrictive as $p$ approaches $2$.

The point $p=2$ is the summability transition: stationary estimation loss is
polynomial in the segment or class occupancy for $p<2$, logarithmic for
$p=2$, and bounded for $p>2$. At $p=2$, class estimation and change
localization are both logarithmic in the horizon, but they scale with $M$ and
$S$, respectively, and therefore need not be of the same overall order.
For $0<p<1$, small residual biases are amplified relative to $L_1$, so the
linear exposure accounting does not directly apply.

\section{Extension to vector-valued observations}
\label{app:vector-extension}

We outline the modifications to the proof of
Theorem~\ref{thm:ecr-upper-bound} for vector-valued observations.
Dimension enters through concentration and the Euclidean estimation error
of a sample mean. The compatibility, exposure, and recurrent-capacity
arguments retain their scalar structure.

\paragraph{Model and regret.}
Let
\[
X_t=\mu_t+\varepsilon_t,
\qquad
\mu_t=\theta^{(c_t)},
\qquad
\theta^{(m)}\in\mathbb{R}^d,
\]
where $c_t\in[M]$ is piecewise constant on the same unknown segmentation as
in Section~\ref{sec:problem-formulation}. Assume that the noises are independent, centered, and satisfy
\[
\mathbb{E}\left[\exp\{\lambda u^\top\varepsilon_t\}\right]
\leq
\exp\left(\frac{\sigma^2\lambda^2}{2}\right)
\]
for every $\lambda\in\mathbb{R}$ and every unit vector
$u\in\mathbb{R}^d$. The class parameters satisfy
\begin{align*}
&\min_{m\neq m'}
\|\theta^{(m)}-\theta^{(m')}\|_2
\geq \Delta_{\min}
\qquad
\max_{m,m'}
\|\theta^{(m)}-\theta^{(m')}\|_2
\leq \Delta_{\max}.
\end{align*}
We assume that these geometric conditions define a nonempty parameter class.
The regret is
\[
R_T^{(d)}
:=
\mathbb{E}\left[
\sum_{t=2}^T
\|\widehat\mu_t-\mu_t\|_2
\right].
\]
The decision/estimation split, block definitions, exposure counters
$H_b(t)$, reuse sets $\mathcal U_t$, and effective sample sizes $W_t$ are
unchanged.

\paragraph{Vector detector.}
For a valid split $r<k<t$, define
\begin{align*}
&\widehat D^r_{k,t}
:=
\frac{1}{\sigma}
\sqrt{
\frac{N^D_{[r,k)}N^D_{[k,t)}}{N^D_{[r,t)}}
}
\left\|
\overline X^D_{[r,k)}-
\overline X^D_{[k,t)}
\right\|_2,
\qquad
C_t^r:=\max_{r<k<t}\widehat D^r_{k,t},
\end{align*}
with the usual convention when no valid split exists.
The alarm rule remains
\[
N_G^r:=\inf\left\{t>r:C_t^r\geq\widetilde{\gamma}_t^r\right\},
\]
where $\widetilde{\gamma}_t^r$ is specified below. Its population
counterpart $D^r_{k,t}$ is obtained by replacing the sample means by the
corresponding population means.

Let $A^r_{k,t}$ denote the normalized
centered two-sample difference. Every one-dimensional projection of
$A^r_{k,t}$ is $1$-sub-Gaussian. Let $\mathcal N$ be a $1/2$-net of the unit
sphere with $|\mathcal N|\leq 5^d$. For every vector $z$, some
$u\in\mathcal N$ satisfies $u^\top z\geq \|z\|_2/2$. Consequently, for every
$x\geq0$,
\[
\mathbb P\left(
\|A^r_{k,t}\|_2
\geq
2\sqrt{2\left(d\log 5+x\right)}
\right)
\leq e^{-x}.
\]
The reverse triangle inequality gives
\[
\left|
\widehat D^r_{k,t}-D^r_{k,t}
\right|
\leq
\|A^r_{k,t}\|_2.
\]
Thus, with $\alpha_r=6\alpha/(\pi^2r^2)$, one valid anytime threshold is
\[
\widetilde{\gamma}_t^r
:=
2\sqrt{2\left(
 d\log 5
 +3\log(t-r)
 +\log\frac{\pi^2}{6\alpha_r}
\right)}.
\]
Indeed, for a fixed triple $(r,k,t)$ the preceding tail is at most
$6\alpha_r/[\pi^2(t-r)^3]$. A union bound over at most $t-r$ candidate split
points and then over all $t>r$ gives
\[
\mathbb P\left(
\exists\, r<k<t\leq T:
\left|\widehat D^r_{k,t}-D^r_{k,t}\right|
\geq \widetilde{\gamma}_t^r
\right)
\leq \alpha_r.
\]
The deterministic-restart summation used in the scalar proof therefore still
gives $\mathbb E[N_{\mathrm{FA}}]\leq\alpha$ and
$\mathbb E[K]\leq S+1+\alpha$. Each detector candidate and each
completed-block compatibility comparison now costs $O(d)$ operations, so
vectorization multiplies the scalar arithmetic cost by a factor of order
$d$.

\paragraph{Vector compatibility statistic.}
For an active prefix $\mathcal A_t$ and a completed block $b$, define
\[
\widehat Q_b(t)
:=
\frac{1}{\sigma}
\sqrt{
\frac{N^D_{\mathcal A_t}N^D_b}{N^D_{\mathcal A_t}+N^D_b}
}
\left\|
\overline X^D_{\mathcal A_t}-\overline X^D_b
\right\|_2,
\]
and let $Q_b(t)$ be its population counterpart. With
$\alpha_b=6\alpha/[\pi^2(b+1)^2]$, define
\[
\widetilde{\lambda}_{b,t}
:=
\sqrt{8\left(
 d\log 5
 +6\log t
 +\log\frac{\pi^2}{6\alpha_b}
\right)},
\]
and use the compatibility rule
\[
\mathcal C_t^{(d)}
:=
\left\{
 b<b_t:
 \widehat Q_b(t)<4\widetilde{\lambda}_{b,t}
\right\}.
\]
The same net argument gives, for every fixed pair of deterministic intervals,
\[
\mathbb P\left(
\left|\widehat Q_b(t)-Q_b(t)\right|
\geq\widetilde{\lambda}_{b,t}
\right)
\leq
\frac{6\alpha_b}{\pi^2t^6}.
\]
The deterministic union over interval pairs and block indices used in the
proof of Lemma~\ref{lem:compat-concentration} then yields a summable
compatibility-deviation probability.
Outside such a deviation event,
\[
\begin{aligned}
Q_b(t)\leq2\widetilde{\lambda}_{b,t}
&\quad\Longrightarrow\quad
b\in\mathcal C_t^{(d)},
\qquad
b\in\mathcal C_t^{(d)}
\quad\Longrightarrow\quad
Q_b(t)<5\widetilde{\lambda}_{b,t}.
\end{aligned}
\]
The numerical constants are not optimized.

\paragraph{Identification scale and anchor segmentation.}
For all indices not exceeding $T$, the two confidence levels above satisfy
\[
\begin{aligned}
\left(\widetilde{\gamma}_t^r\right)^2
&\leq
C_\gamma\left(d+\log\frac{T}{\alpha}\right),
\qquad
\left(\widetilde{\lambda}_{b,t}\right)^2
\leq
C_\lambda\left(d+\log\frac{T}{\alpha}\right),
\end{aligned}
\]
for universal constants $C_\gamma,C_\lambda$. It is therefore sufficient to
replace the scalar identification scale by
\[
\ell_{\mathrm{id}}^{(d)}
:=
\left\lceil
C_{\mathrm{id}}
\frac{\sigma^2}{\Delta_{\min}^2}
\left(
 d+\log\frac{T}{\alpha}
\right)
\right\rceil,
\]
where $C_{\mathrm{id}}$ is chosen sufficiently large. 

The anchor arguments then carry over with Euclidean norms. To see the key
step, suppose that a no-alarm prefix beginning at $r$ contains a pure
class-$m$ interval $[u,v)$ with $h=N^D_{[u,v)}$ observations, and let
$p=N^D_{[r,u)}$. At the split $k=u$,
\[
\left(D^r_{u,v}\right)^2
=
\frac{1}{\sigma^2}
\frac{ph}{p+h}
\left\|
\overline\mu^D_{[r,u)}-\theta^{(m)}
\right\|_2^2.
\]
On the detector good event, no alarm implies $D^r_{u,v}<2\widetilde{\gamma}_v^r$.
Using the convex-combination identity for the full-prefix mean gives
\[
\left\|
\overline\mu^D_{[r,v)}-\theta^{(m)}
\right\|_2
=
\frac{p}{p+h}
\left\|
\overline\mu^D_{[r,u)}-\theta^{(m)}
\right\|_2
\leq
\frac{2\sigma\widetilde{\gamma}_v^r}{\sqrt h}.
\]
Hence $h\geq\ell_{\mathrm{id}}^{(d)}$ makes the full-prefix mean lie within
$\Delta_{\min}/8$ of $\theta^{(m)}$. The preservation argument is identical:
once a prefix of size at least $\ell_{\mathrm{id}}^{(d)}$ is within
$\Delta_{\min}/8$ of its anchor, a later no-alarm prefix remains within
$\Delta_{\min}/4$, by applying the same CUSUM bound to the split at the
anchor-formation time and using the Euclidean triangle inequality.

If a later pure interval belongs to a different class $m'$, then the distance
between its mean and the anchored prefix is at least
$3\Delta_{\min}/4$. Once both sides of the split contain
$\ell_{\mathrm{id}}^{(d)}$ decision-stream observations,
\[
\left(D^r_{k,t}\right)^2
\geq
\frac{9}{16\sigma^2}
\frac{N^D_{[r,k)}N^D_{[k,t)}}{N^D_{[r,t)}}
\Delta_{\min}^2
\geq
\frac{9}{32}
\frac{\ell_{\mathrm{id}}^{(d)}\Delta_{\min}^2}{\sigma^2},
\]
which exceeds the no-alarm level for a sufficiently large
$C_{\mathrm{id}}$. Thus the anchor-formation, preservation, and separation
conclusions of Lemma~\ref{lem:anchor-formation} remain valid after replacing
$\ell_{\mathrm{id}}$ by $\ell_{\mathrm{id}}^{(d)}$. By the same argument as
in Lemma~\ref{lemma:global-contamination-control}, the aggregate count of
off-anchor observations in anchored blocks plus observations in unanchored
blocks remains
\[
O\left((S+K)\ell_{\mathrm{id}}^{(d)}\right).
\]
The localization argument used to remove the detector good event also carries
over, because it uses only the uniform detector bound, convex combinations,
and the triangle inequality.

\paragraph{Vector safety and compatibility guarantees.}
The natural vector analogue of safety is
\[
\left\|
\overline\mu^D_b-\theta^{(\mathfrak a(b))}
\right\|_2
\leq
\sigma\sqrt{\frac{d}{N^D_b}},
\]
with the analogous definition for the active prefix. This choice reduces to
the scalar definition when $d=1$ and matches the Euclidean estimation scale
of a $d$-dimensional sample mean.

For acceptance, let $n=N^D_{\mathcal A_t}$ and $q=N^D_b$. If both the active prefix
and block $b$ are safe for class $m$, then
\[
Q_b^2(t)
\leq
\frac{nq}{n+q}
\left(
\sqrt{\frac{d}{n}}+\sqrt{\frac{d}{q}}
\right)^2
=
d\frac{(\sqrt n+\sqrt q)^2}{n+q}
\leq 2d.
\]
Since $\left(\widetilde{\lambda}_{b,t}\right)^2$ is at least a constant multiple of $d$, the
first compatibility implication above guarantees acceptance.

For rejection, suppose that the active prefix has anchor $m$ and block $b$
has a different anchor $m'$. Anchor separation gives
\[
\left\|
\overline\mu^D_{\mathcal A_t}-\overline\mu^D_b
\right\|_2
\geq
\frac{\Delta_{\min}}{2}.
\]
If both objects are long, then
\[
Q_b^2(t)
\geq
\frac{N^D_{\mathcal A_t}N^D_b}{N^D_{\mathcal A_t}+N^D_b}
\frac{\Delta_{\min}^2}{4\sigma^2}
\geq
\frac{\ell_{\mathrm{id}}^{(d)}\Delta_{\min}^2}{8\sigma^2}.
\]
For a sufficiently large $C_{\mathrm{id}}$, this is larger than
$25\left(\widetilde{\lambda}_{b,t}\right)^2$, contradicting compatibility. Thus the acceptance
and rejection guarantees continue to hold with the vector safety definition.

\paragraph{Population bias and exposure.}
For a round with $W_t>0$, define the population counterpart of the ECR
estimator by
\begin{align*}
&\widetilde\mu_t
:=
\frac{
\sum_{u\in \mathcal A_t\cap\mathcal S^E}\mu_u
+
\sum_{b\in\mathcal U_t}N^E_b\overline\mu^E_b
}{W_t},
\qquad
P_T^{(d)}
:=
\sum_{t:W_t>0}
\|\widetilde\mu_t-\mu_t\|_2.
\end{align*}
The scalar decomposition into current-block and reused-block bias remains
valid by the Euclidean triangle inequality. The exposure calculation is also
unchanged. For an anchored block $b$ used while its anchor is the active
class, write
\[
\beta_b:=\overline\mu^E_b-\theta^{(\mathfrak a(b))}.
\]
Its cumulative contribution is bounded by
\[
\sum_{t:b\in\mathcal U_t}
\frac{N^E_b}{W_t}\|\beta_b\|_2
\leq
N^E_b\|\beta_b\|_2 H_b(T+1).
\]
As in the scalar proof, $H_b(T+1)\leq B_{\exp}+1$. Moreover, if
$c_b^E$ denotes the number of off-anchor estimation observations in block
$b$, then
\[
N^E_b\|\beta_b\|_2
=
\left\|
\sum_{u\in\cB_b\cap\mathcal S^E}
\bigl(\mu_u-\theta^{(\mathfrak a(b))}\bigr)
\right\|_2
\leq
\Delta_{\max}c_b^E.
\]
Thus the repeated bias of a block is controlled by its fixed contamination
mass times its exposure, exactly as in the scalar argument.

The current-block harmonic argument, the treatment of unanchored blocks, and
the exclusion of wrong-anchor blocks after the active prefix becomes long and
centered all use only the diameter bound and the triangle inequality. Thus
the arguments of Lemmas~\ref{lem:current-block-bias} and
\ref{lem:reused-block-bias} carry over. Applying the localization steps of
Lemmas~\ref{lem:robust-detector-localization} and
\ref{lem:remove-detector-conditioning} then gives
\begin{align*}
\mathbb E[P_T^{(d)}]
&\lesssim
\Delta_{\max}
\left[
 (S+1)\ell_{\mathrm{id}}^{(d)}+\alpha L_T
\right]
\left(B_{\exp}+L_T+1\right)
+\Delta_{\max}\alpha,
\qquad
L_T:=1+\log\frac{T}{\alpha}.
\end{align*}

\paragraph{Stochastic estimation under Euclidean loss.}
Condition on the full decision stream, as in
Lemma~\ref{lem:regret-decomposition}. This freezes all
selected estimation indices, while the selected estimation noises remain
independent and centered. If $Z_t:=\widehat\mu_t-\widetilde\mu_t$, then every
projection of $Z_t$ is sub-Gaussian with variance proxy
$\sigma^2/W_t$. For an orthonormal basis $e_1,\ldots,e_d$,
\[
\mathbb E\left[
\|Z_t\|_2^2
\,
\middle|
\mathcal G_T^D
\right]
=
\sum_{j=1}^d
\mathbb E\left[
(e_j^\top Z_t)^2
\,
\middle|
\mathcal G_T^D
\right]
\leq
\frac{d\sigma^2}{W_t}.
\]
Jensen's inequality therefore gives
\[
\mathbb E\left[
\|Z_t\|_2
\,
\middle|
\mathcal G_T^D
\right]
\leq
\sigma\sqrt{\frac{d}{W_t}}.
\]
Consequently, with
\[
V_T:=\sum_{t:W_t>0}\frac{1}{\sqrt{W_t}},
\]
the vector regret decomposition becomes
\[
R_T^{(d)}
\leq
\sigma\sqrt d\,\mathbb E[V_T]
+
\mathbb E[P_T^{(d)}]
+
(\sigma\sqrt d+\Delta_{\max})\mathbb E[K].
\]
The last term accounts for the same fallback rounds as in
Lemma~\ref{lem:regret-decomposition}, using
$\mathbb E\|\varepsilon_t\|_2\leq\sigma\sqrt d$.

\paragraph{Recurrent capacity and the role of $M$.}
The capacity arguments in Eq.~(\ref{eq:guaranteed-class-mass}) and
Lemmas~\ref{lem:hard-capacity-dynamics}--\ref{lem:classwise-leaky-telescope}
are unchanged because they depend only on sample counts, exposure usage, and
compatibility guarantees. Let
$\mathcal B^{\mathrm{LS}}_m$ denote the anchored class-$m$ blocks satisfying
the vector safety condition. At a round at which the active prefix is also
long and safe for class $m$, every live block in
$\mathcal B^{\mathrm{LS}}_m$ is compatible. Define
\begin{align*}
&\Phi_m(t)
:=
\sum_{\substack{b\in\mathcal B^{\mathrm{LS}}_m\\ b<b_t}}
N_b^E\bigl(B_{\exp}-H_b(t)\bigr)_+,
\qquad
Y_m(t):=B_{\exp}A_m(t)+\Phi_m(t),
\qquad
x_m(t):=\frac{Y_m(t)}{B_{\exp}-2},
\end{align*}
where $A_m(t)$ is the active class-$m$ estimation mass. On such a round,
\begin{align*}
W_t
&\geq
A_m(t)+
\sum_{\substack{b\in\mathcal B^{\mathrm{LS}}_m,\ b<b_t\\
H_b(t)<B_{\exp}}}N_b^E
\geq
\frac{Y_m(t)}{B_{\exp}}
=
\frac{B_{\exp}-2}{B_{\exp}}x_m(t).
\end{align*}
When a class-$m$ estimation observation arrives, it adds $B_{\exp}$ units to
the active capacity, while at most two prediction updates occur before the
next estimation observation and consume at most two units of total capacity.
Thus $x_m$ increases by at least one. When an anchored long-safe block is
completed, its active capacity is transferred to the repository at exposure
zero, so $Y_m$ is preserved exactly. Therefore different occurrences of the
same class advance the same counter instead of restarting it.

Ignoring leakage during bad intervals, the integral comparison gives
\[
\sum_{t\in\mathcal T_{\mathrm{LS},m}}
\frac{1}{\sqrt{W_t}}
\lesssim
\sqrt{x_m^{\mathrm{final}}}.
\]
The terminal capacities satisfy
\[
\sum_{m=1}^M x_m^{\mathrm{final}}\lesssim T,
\]
so Cauchy--Schwarz yields
\[
\sum_{m=1}^M
\sqrt{x_m^{\mathrm{final}}}
\leq
\sqrt{M\sum_{m=1}^M x_m^{\mathrm{final}}}
\lesssim
\sqrt{MT}.
\]
Dimension does not modify this telescope.

The bad-interval and leakage bounds require one modification due to the
vector safety radius. If an active prefix with $n$ observations is unsafe for
its anchor, then
\[
\left\|
\sum_{u\in \mathcal A_t\cap\mathcal S^D}
\bigl(\mu_u-\theta^{(m)}\bigr)
\right\|_2
>
\sigma\sqrt{dn}.
\]
If $c$ of these observations are off-anchor, the left-hand side is at most
$\Delta_{\max}c$, and therefore
\[
\sqrt n
<
\frac{\Delta_{\max}}{\sigma\sqrt d}\,c.
\]
The same inequality holds for an unsafe completed block.
Combining Lemmas~\ref{lem:long-unsafe-variance-cost} and
\ref{lem:bad-interval-complexity} with this estimate gives the bound
\begin{align*}
&\mathbb E[V_T]
\lesssim
\sqrt{MT}
+
(1+\sqrt M)
\left(
 (S+1)\sqrt{\ell_{\mathrm{id}}^{(d)}}
 +\alpha L_T
 +
 \frac{\Delta_{\max}}{\sigma\sqrt d}
 \left[
  (S+1)\ell_{\mathrm{id}}^{(d)}+\alpha L_T
 \right]
\right).
\end{align*}
In the contamination terms, the factor $1/(\sigma\sqrt d)$ cancels when
the preceding bound is multiplied by $\sigma\sqrt d$ in the regret decomposition.

\paragraph{Resulting regret bound.}
The preceding modifications give the
following vector analogue, up to universal numerical constants:
\[
\begin{aligned}
R_T^{(d)}(\mathrm{ECR};\nu)
\lesssim{}&
\sigma\sqrt{dMT}
+
\sigma\sqrt d\,(1+\sqrt M)
{}\times
\left[
 (S+1)\sqrt{\ell_{\mathrm{id}}^{(d)}}+\alpha L_T
\right]
\\
&+
\Delta_{\max}(L_T+\sqrt M)
{}\times
\left[
 (S+1)\ell_{\mathrm{id}}^{(d)}+\alpha L_T
\right].
\end{aligned}
\]
A formal theorem would require restating and checking the constants in
Lemmas~\ref{lem:local-detector-concentration}--\ref{lem:remove-detector-conditioning}.

The leading statistical term is therefore
\[
O\left(\sigma\sqrt{dMT}\right).
\]
For comparison, relearning independently on every segment gives the no-reuse benchmark's
estimation scale
\[
\sigma\sqrt d
\sum_{j=0}^S\sqrt{|\mathcal I_j|}
\leq
\sigma\sqrt{d(S+1)T},
\]
whereas pooling all occurrences of each recurrent class gives
\[
\sigma\sqrt d
\sum_{m=1}^M\sqrt{N_m}
\leq
\sigma\sqrt{dMT}.
\]
Thus recurrence yields the same improvement factor as in the scalar case;
dimension multiplies both benchmark scales through $\sqrt d$.

A sufficient condition for the recurrent term to dominate the two displayed
change-dependent terms is that
\begin{align*}
S+1
&\lesssim
\frac{\sqrt{MT}}
     {(1+\sqrt M)\sqrt{\ell_{\mathrm{id}}^{(d)}}},
\qquad
S+1
\lesssim
\frac{\sigma\sqrt{dMT}}
     {\Delta_{\max}(L_T+\sqrt M)\ell_{\mathrm{id}}^{(d)}},
\end{align*}
up to the lower-order $\alpha$ terms. For fixed $d$, $M$, $\sigma$,
$\Delta_{\min}$, and $\Delta_{\max}$, this reduces, up to logarithmic
factors, to the same condition $S=\widetilde O(\sqrt T)$ as in the scalar
paper, and then
\[
R_T^{(d)}(\mathrm{ECR};\nu)
=
\widetilde O\left(\sigma\sqrt{dMT}\right).
\]
Finally, if $\Delta_{\max}\ge(M-1)\Delta_{\min}$, the scalar
change-time construction embeds in one coordinate of $\mathbb R^d$ and yields
the same lower bound under the conditions of
Proposition~\ref{prop:unknown-change-time-lower-bound}. Establishing a matching
$\Omega(\sigma\sqrt{dMT})$ recurrent-estimation lower bound would require a
separate $d$-dimensional packing argument and sufficient geometric slack
under the separation and diameter constraints. 
\end{document}